\pdfoutput=1
\documentclass[11pt]{article}

\usepackage{preamble_arxiv}

\usepackage{titling}
\predate{}
\postdate{}

\newcommand{\papertitle}{Finite-Sample Wasserstein Bounds and Concentration Inequalities for Nonlinear Stochastic Approximation}

 \title{\papertitle}
\author{Seo Taek Kong \and R. Srikant}

\date{}

\makeatletter\begin{document}

\maketitle
\let\thefootnote\relax\footnotetext{University of Illinois Urbana-Champaign. Emails: \texttt{\{skong10, rsrikant\}@illinois.edu}}

\begin{abstract}
Stochastic approximation methods such as stochastic gradient descent and temporal-difference learning return either the last iterate or the Polyak--Ruppert average.
We establish finite-sample Wasserstein-$p$ Gaussian approximation bounds for both outputs under a common framework combining moment bounds for the normalized iterates with quantitative Gaussian approximation of the appropriately weighted driving noise.
These bounds yield convergence of admissible moments and high probability bounds with Gaussian leading terms: heavier dependence on the confidence level is confined to a correction that vanishes with the sample size.
We verify the framework's conditions for nonlinear stochastic approximation with independent or Markovian observations and linear stochastic approximation with Markovian observations, with applications to stochastic gradient descent and temporal-difference learning.
 \end{abstract}

\tableofcontents

\section{Introduction}\label{sec:introduction}
Stochastic approximation (SA) methods, including stochastic gradient descent (SGD) and temporal-difference (TD) learning, seek a root $\vc x^*$ of an unknown mean vector field $\bar F$ through updates
\begin{equation}\label{eq:sa_boilerplate}
    \vc x_{k + 1}
    =
    \vc x_k - \gamma_k\left(F(\vc x_k, \xi_k) + \vc\zeta_k\right),
    \qquad k \geq 1,
\end{equation}
where $\{\gamma_k\}$ is a user-specified step-size sequence.
The framework allows the observation sequence $\{\xi_k\}$ to be Markovian and accommodates an independent i.i.d. additive-noise sequence $\{\vc\zeta_k\}$.

Two standard outputs are the last iterate $\vc x_n$ and the Polyak--Ruppert (PR) average
\begin{equation}\label{eq:pr_average}
    \bar{\vc x}_n
    =
    \frac{1}{n-\bar n_0}\sum_{k = \bar n_0+1}^{n}\vc x_k
    ,
    \qquad \bar n_0 = \lfloor n/2\rfloor
    .
\end{equation}
Writing $\pi$ for the stationary law of $\xi$, let $\bar F(\vc x)=\mathbb E_\pi[F(\vc x,\xi)]$.
Suppose the solution $\vc x^*$ to the fixed point equation $\bar F(\vc x^*) = 0$ exists and is unique.
Under appropriate conditions, the normalized errors
\begin{equation}\label{eq:normalized_iterate}
    \vc y_n
    =
    \gamma_n^{-1/2}(\vc x_n - \vc x^*),
    \qquad
    \bar{\vc y}_n
    =
    \sqrt{n-\bar n_0}\,(\bar{\vc x}_n - \vc x^*),
\end{equation}
converge to respective Gaussian limits by the central limit theorem (CLT) \citep{kushner_yin,polyakJuditsky,borkar25}.
This asymptotic theory provides benchmarks for limiting moments and high probability bounds.
We develop a non-asymptotic CLT under general Markovian noise that extends these results to finite-sample guarantees.
Our results reconcile several gaps between finite-time bounds and asymptotic theory.

\subsection{Last Iterate}
Let $\mat\Sigma$ denote the limiting covariance of $\vc y_n$, and let $\vc z$ be a standard Gaussian vector in $\mathbb R^d$.
Our analysis uses the weighted norm $\lVert\vc v\rVert = (\vc v\tp\mat Q\vc v)^{1/2}$.
For $\delta \in (0, 1)$, define
\begin{equation*}
    \sigma_{\mat Q}(\mat\Sigma)
    \coloneqq
    \lVert\mat Q^{1/2}\mat\Sigma^{1/2}\rVert_2,
    \qquad
    \mathcal G_{\mat\Sigma}(\delta)
    \coloneqq
    \mathbb E\lVert\mat\Sigma^{1/2}\vc z\rVert
    +\sigma_{\mat Q}(\mat\Sigma)\sqrt{2\log(1/\delta)}.
\end{equation*}
The Gaussian concentration inequality for Lipschitz continuous functions \citep[Theorem 5.6]{boucheron12} implies
\begin{equation}\label{eq:Gaussian_tail}
    \lVert\mat\Sigma^{1/2}\vc z\rVert
    \leq
    \mathcal G_{\mat\Sigma}(\delta),
    \qquad \text{w.p. } \geq 1 - \delta.
\end{equation}
Integrating this tail bound and applying Minkowski's inequality gives, for $p \geq 2$, the moment bound
\begin{equation}\label{eq:Gaussian_p_bound}
    \lVert\mat\Sigma^{1/2}\vc z\rVert_{L_p}
    \leq
    \mathbb E\lVert\mat\Sigma^{1/2}\vc z\rVert
    +\sigma_{\mat Q}(\mat\Sigma)\sqrt p.
\end{equation}
While the CLT $\vc y_n \to \mathcal N(\vc 0, \mat \Sigma)$ suggests $\| \vc y_n\| \leq \mathcal G_{\mat \Sigma} (\delta)$ with probability $\geq 1 - \delta$ and the moment bounds $\| \vc y_n \|_{L_p} \leq \mathbb{E}\| \mat \Sigma^{1/2} \vc z \| + \sigma_{\mat Q} (\mat \Sigma) \sqrt{p}$ as $n \to \infty$, the CLT alone provides neither a finite-sample tail bound nor convergence of moments.

Existing finite-time bounds complement asymptotic CLTs, with neither being stronger than the other.
\citet{chen_contractive_concentration} obtain high probability bounds of the form
\begin{equation}\label{eq:weibullian_tail}
    \lVert\vc x_n - \vc x^*\rVert
    \leq C\sqrt{\gamma_n}\left(1 + \log\left(\frac{1}{\delta}\right)\right)^\beta,
    \qquad \text{w.p. } \geq 1 - \delta,
\end{equation}
for some $\beta > 1/2$, where $C$ is independent of $\delta$.
The $\sqrt{\gamma_n}$ dependence matches the scale predicted by the CLT and therefore gives the correct convergence rate for $\|\vc x_n - \vc x^*\|$.
Normalizing by $\sqrt{\gamma_n}$ yields $\| \vc y_n \| \leq C (1 + \log \delta^{-1})^\beta$.
While this bound holds for uniformly in $n$, the bound does not recover Gaussian tails of the form \eqref{eq:Gaussian_tail} suggested by the CLT, due to the dependence on $\log^{\beta} (1/\delta)$.
An intriguing result is that the exponent $\beta > 1/2$ can be unavoidable in general even for \emph{linear} stochastic approximation with \emph{independent} noise \citep{chen_contractive_concentration}.

Finite-time bounds of the form \eqref{eq:weibullian_tail} are also stronger than CLTs in how it provides $L_p$ bounds.
Integrating \eqref{eq:weibullian_tail} and normalizing by $\sqrt{\gamma_n}$ gives, for $p \geq 2$,
\begin{equation}\label{eq:yn_Lp}
    \lVert\vc y_n \rVert_{L_p}
    \leq C(1 + \beta)^\beta p^\beta.
\end{equation}
While this implies $\vc x_n \to \vc x^*$ in $L_p$, it does not recover the Gaussian moments of the form \eqref{eq:Gaussian_p_bound}.
The Weibullian tail $p^\beta \gg \sqrt{p}$ is heavier than Gaussian tails, and the Gaussian tails suggested by CLTs are not recovered as $n \to \infty$.

\textbf{Contributions:}
We obtain a bound on the Wasserstein-$p$ distance between $\mathrm{Law}(\vc y_n)$ and its Gaussian limit under a \emph{moment threshold} condition of the form $\sup_{n \geq N_q} \| \vc y_n\|_{L_q} \leq B_q$.
Here, the threshold $N_q$ is allowed to grow with $q$ and the admissable order $p$ is determined by $q$ and other problem parameters.
For general settings, such moment threshold conditions hold for $q = 2$ \citep{borkar25}, which is sufficient for our Wasserstein-1 bound.
For more specialized settings, such moment threshold conditions can be verified for larger values of $q$.

Our Wasserstein-1 bound can be used to establish that, for every $\delta \in (0, 1)$, there exists a sequence $\varepsilon_n \to 0$ such that for all sufficiently large $n$,
\begin{equation*}
    \lVert \vc y_n \rVert \leq \mathcal{G}_{\mat \Sigma}(\delta - \varepsilon_n) + \varepsilon_n,
    \qquad \whp.
\end{equation*}
Consequently, for all $\delta \in (0, 1)$ and all $\varepsilon > 0$,
\begin{equation*}
    \limsup_{n \to \infty} \mathbb{P}\left(\| \vc y_n \| \geq \mathcal{G}_{\mat \Sigma}(\delta) + \varepsilon \right) \leq \delta .
\end{equation*}
Therefore, our finite-time bound recovers the Gaussian tail \eqref{eq:Gaussian_tail} asymptotically.

For specialized settings where the moment threshold condition holds with $B_q = \mathcal O(q^\beta)$ for a range of values $q \geq 4$, we show that there exists a sequence $\varepsilon_n' \to 0$ independent of $\delta$ such that
\begin{equation*}
    \| \vc y_n \| \leq \mathcal{G}_{\mat \Sigma}\left(\frac{\delta}{2}\right) + \varepsilon_n' \polylog \left(\frac{2}{\delta} \right) ,
    \qquad \whp
    .
\end{equation*}
Unlike the bound in \eqref{eq:yn_Lp}, the heavy tail $\varepsilon_n' \polylog (2/\delta)$ vanishes as $n \to \infty$, thereby recovering the $\sqrt{\log (2/\delta)}$ factor in the Gaussian tail \eqref{eq:Gaussian_tail} asymptotically even when $\beta > 1/2$, because the term carrying the heavier dependence vanishes with $n$.
The mild $\delta / 2$ factor arises from a union bound.

Finally, our results imply rates of $L_p$ convergence to those of the Gaussian limit.
The CLT does not imply $\| \vc y_n \|_{L_p} \to \| \mat \Sigma^{1/2} \vc z \|_{L_p}$, while finite-time bounds of the form \eqref{eq:yn_Lp} do not imply convergence $\| \vc y_n \|_{L_p} \to \| \mat \Sigma^{1/2} \vc z \|_{L_p}$.
Because a quantitative CLT in the Wasserstein-$p$ distance yields a rate of convergence in $L_p$, we obtain that for each $p$ covered by the assumptions and all sufficiently large $n$,
\begin{equation*}
    \left\lvert\lVert\vc y_n\rVert_{L_p} - \lVert\mat\Sigma^{1/2}\vc z\rVert_{L_p}\right\rvert
    \leq \varepsilon_n
    ,
\end{equation*}
where $\varepsilon_n \to 0$ as $n \to \infty$.
The range of $p$ for which the above bound holds is contingent on established moment bounds, which hold for various applications considered in \Cref{sec:applications}.

\subsection{Polyak--Ruppert Average}
The PR average has a Gaussian limit whose covariance is asymptotically optimal \citep{polyakJuditsky}.
Its non-asymptotic theory includes sharp moment and high probability bounds \citep{gadat2022optimal,durmus2024finite}, as well as quantitative CLTs for specialized settings \citep{srikant2024CLT,samsonov2024gaussian,sheshukova2025gaussianapproximationmultiplierbootstrap,butyrin2025improved,samsonov2025statistical,srikant2024CLT}.

\citet{khodadadian2025general} provide a general procedure for deriving high probability bounds for the PR average from high-probability bounds on the unaveraged iterates, with applications to contractive stochastic approximation.
Their bounds recover asymptotically sub-Gaussian bounds of the form $\| \bar{\vc y}_n \| \leq \mathcal{O}(\log^{1/2} (1/\delta)) + o(1)$, with a correction term $o(1)$ that decays to zero as $n \to \infty$ with $\delta$ fixed.
We establish a non-asymptotic CLT in the Wasserstein-$p$ distance under the same framework used to establish the last iterate's CLT.
Our non-asymptotic CLT yields a high probability bound $\mathcal G_{\bar{\mat \Sigma}}(\delta - o(1)) + o(1)$, whose leading term involves the limiting covariance rather than a sub-Gaussian variance proxy.
Moreover, our result implies convergence of moments and quantifies the convergence rates.

\subsection{Applications and High Probability Bounds}
The last-iterate and PR-average analyses both rely on the moment bounds for normalized iterates and quantitative Gaussian approximations for the driving noise stated in \Cref{ass:normalized,ass:clt}.
We verify these conditions for nonlinear stochastic approximation with Markovian observations, nonlinear stochastic approximation with independent observations, and linear stochastic approximation with Markovian observations.
The resulting high probability bounds are summarized in \Cref{tab:applications}.
In the table, $\mathcal G_{\mat\Sigma}$ and $\mathcal G_{\bar{\mat\Sigma}}$ are the Gaussian tails associated with the limiting covariances of the last iterate and the PR average, respectively.
We use $C$ and $\bar C$ to denote problem-dependent constants independent of $n$ and $\delta$.
\begin{table}[!ht]
    \centering
    \small
    \renewcommand{\arraystretch}{1.3}
    \begin{tabular}{@{}
        >{\raggedright\arraybackslash}m{0.13\textwidth}
        >{\centering\arraybackslash}m{0.405\textwidth}
        >{\centering\arraybackslash}m{0.405\textwidth}@{}}
        \toprule
        Setting & Last iterate $\lVert\vc y_n\rVert$ & PR average $\lVert\bar{\vc y}_n\rVert$ \\
        \midrule
        NM
        (\Cref{cor:nm})
                & $
            \mathcal G_{\mat\Sigma}\left(\delta-
                o(1) \right)
            +\mathcal O\left(\frac{1}{n^{1/4}} \right)$

                & $
            \mathcal G_{\bar{\mat\Sigma}}\left(\delta-
                o(1)\right)
            +\mathcal O\left(\frac{1}{n^{1/8}} \right)$
        \\
        ID
        (\Cref{cor:id})
                & $
            \mathcal G_{\mat\Sigma}\left(\delta-o(1)\right)
            +o(1)$
                & $
            \mathcal G_{\bar{\mat\Sigma}}\left(\delta-o(1)\right)
            +o(1)$                                                                                                                                       \\
        LM
        (\Cref{cor:lm})
                & $\begin{gathered}
            \left(1+o(1)\right)
            \mathcal G_{\mat\Sigma}\left(\frac{\delta}{2}\right)
            +\mathcal O\left(\frac{\log\left(\frac{2}{\delta}\right)}{n^{1/6}}\right)
            \end{gathered}$
                & $\begin{gathered}
            \left(1+o(1)\right)
            \mathcal G_{\bar{\mat\Sigma}}\left(\frac{\delta}{2}\right)+
            \mathcal O\left(\frac{\log\left(\frac{2}{\delta}\right)}{n^{1/6}}  \right)
            \end{gathered}$ \\
        \bottomrule
    \end{tabular}
    \caption{
        Finite-sample high probability bounds for the three applications considered in \Cref{sec:applications}.
        NM, ID, and LM denote nonlinear Markov, nonlinear independent, and linear Markov stochastic approximation, respectively in \Cref{sec:nonlinear_markov,sec:nonlinear_iid,sec:linear_markov}.
        The $o(1)$ describes a sequence that decays to zero as $n \to \infty$ with all problem parameters fixed, and $\mathcal O$ hides problem parameters independent of $n$ and $\delta$.
        The exact rates are specified in \Cref{cor:nm,cor:id,cor:lm}.
    }
    \label{tab:applications}
\end{table}

\subsection{Related Work}
Classical asymptotic theory establishes Gaussian limits for stochastic approximation \citep{kushner_yin,pelletier1998weak}.
More recently, \citet{borkar25} establish central limit theorems and covariance convergence for both the last iterate and the Polyak--Ruppert average under Markovian observations.

Finite-sample error and concentration bounds for the last iterate have been obtained for linear stochastic approximation and contractive stochastic approximation \citep{srikant19,chen_contractive_concentration}.
For the PR average, existing results include moment and high probability bounds \citep{gadat2022optimal,durmus2024finite}, while \citet{khodadadian2025general} transfer high-probability bounds for the unaveraged iterates to their average.
These results control estimation error rather than distance to the Gaussian limit.

Finite-sample Gaussian approximations for PR averages have been established for temporal-difference learning with Markovian observations in Wasserstein distance \citep{srikant2024CLT}, for linear stochastic approximation and nonlinear stochastic gradient descent with independent observations in convex distance \citep{samsonov2024gaussian,sheshukova2025gaussianapproximationmultiplierbootstrap,butyrin2025improved}, and for linear stochastic approximation with Markovian observations through Berry--Esseen bounds for linear projections \citep{samsonov2025statistical}.
Our framework covers a broader range of models, providing multivariate Wasserstein bounds for both the last iterate and the PR average under nonlinear stochastic approximation and Markovian observations.

We establish a Wasserstein-$p$ framework for nonlinear stochastic approximation driven by $V$-ergodic Markov chains.
Our main results establish quantitative bounds on the CLTs for $\vc y_n$ and $\bar{\vc y}_n$, contingent on moment bounds and a rate in the CLT for a weighted sum of martingale difference sequences.
Two contemporaneous works \citep{haque2026accurately,zhang2026gaussianapproximationmultivariatemartingale} are closely related.
An earlier version of this manuscript appeared contemporaneously with \citep{haque2026accurately}, who establish Wasserstein-1 bounds under additive i.i.d. noise and multiplicative martingale-difference noise.
Both the earlier version of our manuscript and the current version establish a Wasserstein-$p$ framework under general Markovian observations, where the main difference is the form of the CLT assumed and the rates established.
The main results are established contingent on Wasserstein-$p$ CLTs for sums of martingale difference sequences, which were not available for general martingales at the time of the earlier version of this manuscript for $p > 1$.
Subsequent work \citep{zhang2026gaussianapproximationmultivariatemartingale,fang2026martingalecentrallimittheorems} addressed this open problem, establishing martingale CLTs for $p > 1$.
We then revised this manuscript to use these results, and subsequently sharpened the convergence rates.
During these revisions, a second version of \citep{zhang2026gaussianapproximationmultivariatemartingale} established stochastic approximation bounds with matching rates in some overlapping settings when the underlying Markov chain is uniformly ergodic.
Our framework permits $V$-ergodicity without requiring uniformly ergodicity, as in \Cref{sec:linear_markov}, and establishes Wasserstein-$p$ bounds and their dependence on $p$, where the $p$ dependence is needed to establish concentration inequalities.
Moreover, we verify the assumptions used to establish our main results in \Cref{sec:applications}.

\textbf{Organization:}
\Cref{sec:problems} introduces the problem formulation, motivating examples, and assumptions.
\Cref{sec:wasserstein_bounds} establishes the general Wasserstein bounds, and \Cref{sec:applications} verifies their assumptions for MCMC-driven stochastic approximation, SGD, and TD learning.
 \section{Formulation}\label{sec:problems}
We consider the following standard polynomially decaying step-size schedule.
\begin{assumption}\label{ass:step_size}
    The step size is chosen as $\gamma_k = \gamma_1 k^{-a}$ with $\gamma_1 > 0$, and $a$ is chosen as:
    \begin{assumptionsub}[label=(\roman*),ref=\theassumption(\roman*)]
        \item\label{ass:step_size_last} (Last iterate) For the last iterate, $a = 1$.
        \item\label{ass:step_size_pr} (Averaged iterate) For the PR average, $a \in (0, 1)$.
    \end{assumptionsub}
\end{assumption}

\paragraph{Notation.}
We use $\vee$ to denote the binary max operator $p \vee 2 = \max\{p,2\}$.
Vectors are denoted in bold italic, as $\vc v$.
We use the unbolded notation $F$ and $\bar F$ for the vector field and mean field.
The Euclidean norm of a vector $\vc v$ is denoted by $\lVert\vc v\rVert_2$.

For a scalar random variable $Y$, $\lVert Y\rVert_{L_r} = (\mathbb E\lvert Y\rvert^r)^{1/r}$.
For a random vector $\vc v$, $\lVert\vc v\rVert_{L_r} = (\mathbb E\lVert\vc v\rVert^r)^{1/r}$, where the norm without a subscript is the weighted norm \eqref{eq:weighted_norm} of the matrix $\mat Q$ fixed below.
The notation $\vc z$ is reserved for a standard Gaussian vector in $\mathbb R^d$.
A sequence is named by dropping the index, so that $\vc m$ denotes $\{\vc m_k\}$ and $V(\xi)$ denotes $\{V(\xi_k)\}$.
For a sequence $\vc W$ adapted to $\{\mathcal H_k\}$, we use the notation $\condsd(\vc W_k \mid \mathcal H_{k - 1}) \coloneqq \sqrt{\mathbb E\left[\lVert\vc W_k\rVert^2 \mid \mathcal H_{k - 1}\right]}$.
The notation $\condsd(\vc W \mid \mathcal H)$ is used to denote a sequence formed by $\vc W$ and $\mathcal H$.

Matrices are denoted in upright bold, as $\mat B$, and $\| \mat B \|_2$ denotes the operator norm.
We use $\lVert\mat B\rVert_{HS} = \left(\operatorname{tr}(\mat B\tp\mat B)\right)^{1/2}$ for the Hilbert--Schmidt norm, $\mat B\tp$ for the transpose of $\mat B$, and $\mat B^{1/2}$ for the positive square root of a positive definite matrix.
The norms $\lVert \mat B \rVert_{L_r, 2}, \lVert \mat B \rVert_{L_r, HS}$ denote the $L_r$ norm of the operator and Hilbert norms, respectively.
All matrices in this paper are real.

\subsection{Motivating Examples}
\label{sec:motivating_examples}
To motivate our assumptions and results, we consider three canonical applications of \eqref{eq:sa_boilerplate}.
The examples cover nonlinear and linear mean fields with independent and Markovian observations, following the organization of \Cref{sec:applications}.

\textbf{Nonlinear Stochastic Approximation.}
Many learning problems seek to minimize an expected loss $\mathcal L(\vc x) = \mathbb E_{\pi}[\ell(\vc x, \xi)]$.
Stochastic gradient descent algorithms aim to find the minimizer $\vc x^* = \operatorname*{arg\,min}_{\vc x \in \mathbb R^d} \mathcal L(\vc x)$ using \eqref{eq:sa_boilerplate} with $F(\vc x, \xi) = \nabla\ell(\vc x, \xi)$ and $\bar F(\vc x) = \nabla\mathcal L(\vc x)$.
\begin{itemize}
    \item \emph{MCMC-driven SGD for wireless scheduling.}
          Consider a wireless network in which the vertices of a finite graph $G = (\mathsf V, \mathsf E)$ represent communication links and each edge joins two links that cannot transmit simultaneously.
          A feasible set of active links is therefore an independent set $\xi \in \mathcal I(G)$.
          The user configures a randomized scheduler represented by a transition matrix $P$ on $\mathcal I(G)$.
          For a parameter $\eta > 0$, define the target distribution
          \begin{equation*}
              \pi(\xi)
              = \frac{\eta^{\lvert\xi\rvert}}{\sum_{\xi' \in \mathcal I(G)}\eta^{\lvert\xi'\rvert}},
              \qquad \xi \in \mathcal I(G).
          \end{equation*}
          A standard choice of $P$ selects a link uniformly, activates it with probability $\eta/(1 + \eta)$ and deactivates it with probability $1/(1 + \eta)$ when none of its neighbors is active, and otherwise leaves it inactive \citep{vigoda2001note}.
          The transition matrix $P$ is chosen from $G$ and $\eta$ before the algorithm begins and does not depend on the iterate $\vc x_k$.
          Since an active link transmits one packet per time slot, its stationary activation probability is its long-run service rate and, for a continuously backlogged link, its throughput; see, e.g., \citep{ni2011q}.
          As an illustrative estimation task, consider using logistic regression to learn how observable link features relate to service probability, thereby pooling information across links from the correlated schedule observations.
          Given a feature vector $\vc a_v \in \mathbb R^d$ for each link $v \in \mathsf V$ and a ridge parameter $\kappa > 0$, take
          \begin{equation*}
              \ell(\vc x, \xi)
              = \frac{1}{\lvert\mathsf V\rvert}\sum_{v \in \mathsf V}
              \left[
              \log\left(1 + e^{\vc a_v\tp\vc x}\right)
              - \mathbf{1}_{\{v \in \xi\}}\vc a_v\tp\vc x
              \right]
              + \frac{\kappa}{2}\lVert\vc x\rVert_2^2.
          \end{equation*}
          The corresponding stochastic gradient is
          \begin{equation*}
              F(\vc x, \xi)
              = \nabla\ell(\vc x, \xi)
              = \kappa\vc x
              + \frac{1}{\lvert\mathsf V\rvert}\sum_{v \in \mathsf V}
              \left[
              \frac{1}{1 + e^{-\vc a_v\tp\vc x}}
              - \mathbf{1}_{\{v \in \xi\}}
              \right]\vc a_v.
          \end{equation*}
          Thus, $\vc x^*$ is the regularized logistic projection of the scheduler's link-level service rates onto the chosen feature model.
          Running \eqref{eq:sa_boilerplate} estimates this projection from Markovian schedules on a combinatorial state space without changing the transition matrix $P$.
          \Cref{sec:nonlinear_markov} considers problems motivated by this application.

    \item \emph{Independent SGD.}
        In optimization, the observations $\{\xi_k\}$ are drawn independently from $\pi$.
        \Cref{sec:nonlinear_iid} includes the standard setting of smooth, strongly convex objectives.
\end{itemize}

\textbf{Linear Stochastic Approximation.}
Some applications involve a linear operator
\begin{equation}\label{eq:lsa}
    F(\vc x, \xi) = \mat A(\xi)\vc x - \vc b(\xi),
    \qquad \pi(\mat A) = \bar{\mat A}, \quad \bar{\mat A}\vc x^* = \pi(\vc b).
\end{equation}
\begin{itemize}
    \item \emph{TD(0) learning.}
        Policy evaluation in reinforcement learning seeks to evaluate a fixed policy from a Markov reward process with discount factor $\beta \in (0, 1)$.
        Given feature vectors $\vc\phi(S) \in \mathbb R^d$, the temporal-difference update direction corresponds to $\xi_k = (S_k, R_k, S_{k + 1})$ and
        \begin{equation*}
            F(\vc x, \xi) = \vc\phi(S)\left[\left(\vc\phi(S) - \beta\vc\phi(S')\right)\tp\vc x - R\right].
        \end{equation*}
        This matches \eqref{eq:lsa} with $\mat A(\xi) = \vc\phi(S)\left(\vc\phi(S) - \beta\vc\phi(S')\right)\tp$ and $\vc b(\xi) = \vc\phi(S)R$, where $\vc x^*$ solves $\mathbb E_{\pi}[F(\vc x^*, \xi)] = \vc 0$.
        \Cref{sec:linear_markov} treats this setting under boundedness conditions on the linear operators and their Poisson solutions.

\end{itemize}

The condition on $\nabla\bar F(\vc x^*)$ used to establish convergence of $\vc y$ depends on the step-size exponent, as imposed in \Cref{ass:stability}.
While $\gamma_1$ must be tuned to be sufficiently large when $a = 1$ to achieve the $1/\sqrt{n}$ rate with the last iterate, one may use Polyak--Ruppert averaging to achieve the $1/\sqrt{n}$ rate with $a < 1$, thereby bypassing the stringent requirement $\gamma_1$ be large relative to the spectrum of $\nabla\bar F(\vc x^*)$.

\subsection{Stability and System Properties}\label{sec:abscissa}
For a measure $\pi$ on state space $\mathcal S$, let $\bar F(\vc x) = \int F(\vc x, \xi)\pi(d\xi)$.
We assume that a fixed point solution $\vc x^*$ exists such that $\bar F(\vc x^*) = 0$.
\begin{assumption}\label{ass:regularity}
    For every $\xi$, the map $\vc x \mapsto F(\vc x, \xi)$ is continuously differentiable, and there is a finite $L_F$ with
    \begin{equation*}
        \lVert F(\vc x, \xi) - F(\vc x', \xi)\rVert
        \leq L_F\lVert\vc x - \vc x'\rVert,
        \qquad \vc x, \vc x' \in \mathbb R^d, \ \xi \in \mathcal S.
    \end{equation*}
    The mean field $\bar F$ is continuously differentiable, and we denote its Jacobian by $\nabla\bar F(\vc x)$.
    The Jacobian is assumed to satisfy
    \begin{equation*}
        \lVert\nabla\bar F(\vc x) - \nabla\bar F(\vc x')\rVert
        \leq
        L_{\nabla}\lVert\vc x - \vc x'\rVert .
    \end{equation*}
\end{assumption}
Under \Cref{ass:regularity}, the recursion for $\vc x$ can be linearized around $\vc x^*$ to obtain the recursion
\begin{equation}\label{eq:x_linearization}
    \vc x_{k+1} - \vc x^* = (\identity - \gamma_k \nabla \bar F (\vc x^*)) (\vc x_k - \vc x^*)
    - \gamma_k \left(F(\vc x_k, \xi_k) - \nabla \bar F(\vc x^*) (\vc x_k - \vc x^*) + \vc \zeta_k  \right).
\end{equation}
Normalizing by $\sqrt{\gamma_{k+1}}$, this gives rise to the recursion
\begin{equation}\label{eq:finite_step_drift}
    \vc y_{k + 1}
    =
    \bigl(\identity - \gamma_k\drift_k\bigr)\vc y_k
    +\text{remainder},
    \qquad
    \drift_k \coloneqq \frac{1}{\gamma_k}\left(1 - \sqrt{\frac{\gamma_k}{\gamma_{k + 1}}}\right)\identity
    +\sqrt{\frac{\gamma_k}{\gamma_{k + 1}}}\nabla\bar F(\vc x^*).
\end{equation}
Under \Cref{ass:step_size}, $\drift_k$ converges to $\drift = \nabla\bar F(\vc x^*) - \bar\gamma\identity$, where $\bar\gamma \coloneqq \lim_{k \to \infty}\frac{1}{\gamma_k}\bigl(\sqrt{\gamma_k/\gamma_{k + 1}} - 1\bigr)$.

\begin{assumption}\label{ass:stability}
    Every eigenvalue of $\drift$ has a positive real part.
\end{assumption}
When the step size in \Cref{ass:step_size} is set with $a = 1$, $\gamma_1$ must be sufficiently large to also satisfy \Cref{ass:stability}.
Using a small value of $\gamma_1$ when $a = 1$ can lead to degraded rates, whereas setting $\gamma_1$ to be large can increase the asymptotic covariance.

For a matrix $\mat Q \succ 0$, define the inner product $\langle\vc u, \vc v\rangle_{\mat Q} = \vc u\tp\mat Q\vc v$ and the induced norm
\begin{equation}\label{eq:weighted_norm}
    \lVert\vc v\rVert_{\mat Q} = \langle\vc v, \vc v\rangle_{\mat Q}^{1/2},
    \qquad
    \lVert\mat B\rVert_{\mat Q} = \sup_{\vc v \ne 0}
    \frac{\lVert\mat B\vc v\rVert_{\mat Q}}{\lVert\vc v\rVert_{\mat Q}}.
\end{equation}
By Lyapunov stability, there exists $\mat Q \succ 0$ such that
\begin{equation}\label{eq:lyapunov}
    \drift^\top \mat Q + \mat Q \drift = \identity .
\end{equation}
Throughout the manuscript, we fix such a matrix $\mat Q$, define $\lambda \coloneqq 1/(2\lVert\mat Q\rVert_2)$, and write $\lVert\cdot\rVert = \lVert\cdot\rVert_{\mat Q}$ for both the vector norm and its induced matrix norm.
The Euclidean vector norm and its induced operator norm are both denoted by $\| \cdot \|_2$.
For every $t \geq 0$ and $\vc v \in \mathbb R^d$,
\begin{equation*}
    \frac{d}{dt}\lVert e^{-t\drift}\vc v\rVert^2
    =
    -(e^{-t\drift}\vc v)\tp(\drift\tp\mat Q + \mat Q\drift)e^{-t\drift}\vc v
    \leq - 2\lambda\lVert e^{-t\drift}\vc v\rVert^2.
\end{equation*}
Gr\"onwall's inequality and the discrete-time contraction \citep[Lemma~17]{kaledin2020finite} show that, for every $t \geq 0$ and $\gamma \in (0, \lambda/\lVert\drift\rVert^2]$,
\begin{equation}\label{eq:contraction}
    \lVert e^{-t\drift}\rVert
    \leq
    e^{-\lambda t},
    \qquad
    \lVert\identity - \gamma\drift\rVert
    \leq
    1 - \frac{1}{2}\lambda\gamma.
\end{equation}

\subsection{Noise Model}\label{sec:poisson}
Let $\mathcal H_k = \sigma(\vc x_1, \xi_0, \vc\zeta_1, \xi_1, \ldots, \vc\zeta_k, \xi_k)$ denote the history at time $k$.
For a sequence $\vc W$ adapted to $\{\mathcal H_k\}$ and $r \geq 1$, define
\begin{equation}\label{eq:rosenthal_quantity}
    \seqnorm{\vc W}_{L_r} \coloneqq \sup_{k \geq 1}\lVert\vc W_k\rVert_{L_r}.
\end{equation}
Our analysis applies to algorithms that are driven by additive exogenous noise.
\begin{assumption}\label{ass:exogenous}
    The sequence $\{\vc \zeta_k\}_{k \geq 1}$ is an i.i.d., centered sequence that is independent of $(\vc x_1, \{\xi_k\}_{k \geq 0})$.
    For some $q_\zeta \geq 2$, the sequence satisfies
    \begin{equation*}
        \seqnorm{\vc\zeta}_{L_{q_\zeta}} < \infty.
    \end{equation*}
\end{assumption}
Although our analysis can accommodate martingale difference additive noise, we assume that \(\vc\zeta\) is i.i.d. and independent of \(\xi\) so that the augmented process \(\{(\xi_k,\vc\zeta_k)\}\) is also Markovian and inherits its ergodic properties.

State-dependent (multiplicative) noise frequently emerges in modern machine learning algorithms.
For empirical risk minimization with SGD, the noise can form an i.i.d. sequence.
More generally, the noise is generated by an underlying Markov chain.
Our analysis applies to this general setting.
Denote by $P$ the transition kernel of the Markov chain $\xi$ with respect to $\{\mathcal H_k\}$.
\begin{assumption}\label{ass:markov}
    There exists a function $V\colon\mathcal S \to [1, \infty)$ such that the following conditions hold.
    \begin{assumptionsub}
        \item\label{ass:noise_markov}
        The Markov chain $\{\xi_k\}$ is $\psi$-irreducible and aperiodic, with transition kernel $P$ and stationary distribution $\pi$.
        There exist a petite set $\mathcal C$, constants $\rho_V \in (0, 1)$, $b_V >0$, and $q_V \geq 2$ such that
        \begin{equation*}
            PV(\xi) \leq \rho_VV(\xi) + b_V\mathbf{1}_{\mathcal C}(\xi),
        \end{equation*}
        and $\sup_{k \geq 0}\lVert V(\xi_k)\rVert_{L_{q_V}} < \infty$.

        \item\label{ass:poisson_moments}
        There exists a constant $K_F$ such that $\lVert F(\vc x^*, \xi)\rVert \leq K_FV(\xi)$ for all $\xi \in \mathcal S$.

    \end{assumptionsub}
\end{assumption}
Under \Cref{ass:noise_markov}, the chain is $V$-uniformly ergodic \citep[Theorem 16.0.1]{meyn2012markov}: there exist $C_V < \infty$ and $\varrho \in (0, 1)$ such that every $\varphi$ satisfying $\sup_{\zeta' \in \mathcal S}\lVert\varphi(\zeta')\rVert/V(\zeta') < \infty$ obeys
\begin{equation}\label{eq:V_ergodic}
    \lVert P^j\varphi(\xi) - \pi(\varphi)\rVert
    \leq
    C_V\varrho^jV(\xi)\sup_{\zeta' \in \mathcal S}\frac{\lVert\varphi(\zeta')\rVert}{V(\zeta')},
    \qquad j \geq 0, \ \xi \in \mathcal S.
\end{equation}
\Cref{ass:poisson_moments} ensures that \eqref{eq:V_ergodic} applies to $F(\vc x^*, \cdot)$.
If $F(\vc x^*, \cdot)$ is bounded, \Cref{ass:poisson_moments} holds for every $V \geq 1$ and therefore places no additional restriction on the function $V$ satisfying \Cref{ass:noise_markov}.
For TD learning with bounded features, $\lVert F(\vc x^*, \xi)\rVert$ is bounded by a constant multiple of $1 + \lvert R\rvert$.
Consequently, \Cref{ass:poisson_moments} holds whenever $1 + \lvert R\rvert$ is bounded by a constant multiple of $V(\xi)$.

For $\vc x \in \mathbb R^d$, consider the centered operator $\widehat F(\vc x, \xi) = F(\vc x, \xi) - \bar F(\vc x)$.
Under Assumptions~\ref{ass:regularity} and \ref{ass:markov}, it holds by \citep[Theorem 17.4.2]{meyn2012markov} that there exists a solution $\Phi(\vc x, \cdot)$, for every $\vc x$, to the Poisson equation
\begin{equation}\label{eq:poisson}
    \widehat F(\vc x, \cdot) = \Phi(\vc x, \cdot) - P\Phi(\vc x, \cdot),
    \qquad \pi\left(\Phi(\vc x, \cdot)\right) = 0.
\end{equation}
Under Assumptions~\ref{ass:regularity}, \ref{ass:poisson_moments}, and $V \geq 1$, $\widehat F$ satisfies
\begin{equation*}
    \lVert\widehat F(\vc x, \xi) \rVert
    \leq
    \left(K_F + 2L_F\lVert\vc x - \vc x^*\rVert\right)V(\xi).
\end{equation*}
Define $L_\Phi \coloneqq \tfrac{2C_VL_F}{1 - \varrho}$.
Under Assumptions~\ref{ass:regularity} and \ref{ass:markov}, we show in \Cref{lem:poisson_regularity} that $\tfrac{\Phi(\cdot, \xi)}{V(\xi)}$ and $\tfrac{(P \Phi)(\cdot, \xi)}{(PV)(\xi)}$ are $L_\Phi$--Lipschitz continuous in $\vc x$ for all $\xi$.
Using $\Phi$, we define a martingale difference sequence
\begin{equation}\label{eq:Mk}
    \vc m_k = - \left(\Phi(\vc x^*, \xi_k) - (P \Phi)(\vc x^*, \xi_{k-1}) \right) - \vc\zeta_k
    .
\end{equation}
Under Assumptions~\ref{ass:exogenous} and \ref{ass:markov}, the sequence $\vc m$ has finite moments through order $q_m \coloneqq q_V\wedge q_\zeta$.

\subsection{Uniform Moment Bounds}\label{sec:moment_assumption}
Recall the normalized-error sequence defined in \eqref{eq:normalized_iterate}.
Our main results require a uniform moment bound for the normalized iterates.
Because existing moment bounds \citep{gerencser1992rate,borkar2000ode,srikant19,chen_contractive_concentration} do not follow from Assumptions~\ref{ass:step_size}--\ref{ass:markov} alone, we make the following assumption.
\begin{assumption}\label{ass:normalized}
    For some $q_y \in [2, \infty]$, there is a finite threshold index $k_0(q_y)$ such that
    \begin{equation*}
        B_{q_y}
        \coloneqq
        \sup_{k \geq k_0(q_y)}\lVert\vc y_k\rVert_{L_{q_y}}
        < \infty.
    \end{equation*}
    The threshold and the bound may depend on $q_y$, $\gamma_1$, and other model parameters.
\end{assumption}
\Cref{ass:normalized} is verified for various applications in \Cref{sec:applications}.

 \section{Main Results}\label{sec:wasserstein_bounds}
In this section, we establish finite-sample Wasserstein-$p$ bounds for the last iterate and the PR average with respective Gaussian references motivated by the CLT.
To this end, we use the martingale difference sequence $\vc m$ to linearize the recursion for $\vc y$.

Our finite sample bounds hold for all $n \geq n_0$, where $n_0$ is defined as
\begin{equation*}
    n_0
    \coloneqq
    \min\{k\in\mathbb N:k\geq\max\{2,k_0(q_y)+1\},\;2\gamma_k\lVert\drift\rVert^2\leq\lambda\}
\end{equation*}
so that $\| \identity - \gamma_k \drift \| \leq 1 - \lambda \gamma_k / 2$ and $\| \vc y_k \|_{L_{q_y}} \leq B_{q_y}$ for all $k \geq n_0$.
The bounds stated in this section use the $\mathcal O$ notation assuming problem constants and $p$ are fixed for simplicity of exposition.
The full statements are deferred to \Cref{app:wasserstein_bounds}.

\subsection{Martingale Approximation and Error Decomposition}\label{sec:linearization}
A key step of our analysis is to write the recursion for $\vc y$ as follows.
Define
\begin{equation}\label{eq:Rk}
\begin{alignedat}{2}
    \vc R_k^{\mathrm{drift}}
    &\coloneqq -\gamma_k(\drift_k-\drift)\vc y_k,
    \qquad&
    \vc R_k^{\mathrm{nl}}
    &\coloneqq -\frac{\gamma_k}{\sqrt{\gamma_{k+1}}}
    \bigl(\bar F(\vc x_k)
    -\nabla\bar F(\vc x^*)(\vc x_k-\vc x^*)\bigr),
    \\
    \vc R_k^{\mathrm{step}}
    &\coloneqq \left(\frac{\gamma_k}{\sqrt{\gamma_{k+1}}}
    -\sqrt{\gamma_k}\right)\vc m_k,
    \qquad&
    \vc R_k^{\mathrm{P}}
    &\coloneqq -\frac{\gamma_k}{\sqrt{\gamma_{k+1}}}
    \bigl(\widehat F(\vc x_k,\xi_k)+\vc\zeta_k+\vc m_k\bigr).
\end{alignedat}
\end{equation}

\begin{lemma}[Error Decomposition]\label{lem:linearization}
    Under \Cref{ass:regularity}, the normalized-error sequence $\vc y$ defined in \eqref{eq:normalized_iterate} satisfies the recursion
    \begin{equation}\label{eq:yk}
        \vc y_{k + 1} = \vc y_k - \gamma_k\drift\vc y_k
        +\sqrt{\gamma_k}\vc m_k + \vc R_k,
    \end{equation}
    where $\vc R_k=\vc R_k^{\mathrm{drift}}+\vc R_k^{\mathrm{nl}}+\vc R_k^{\mathrm{step}}+\vc R_k^{\mathrm P}$.
    In particular, under \Cref{ass:step_size},
    \begin{equation}\label{eq:yns}
    \begin{aligned}
        \vc y_{n} &= \sum_{k=n_0}^{n-1} \mat W_{k, n} \vc m_k + \prod_{k=n_0}^{n-1} (\identity - \gamma_k \drift) \vc y_{n_0} + \sum_{k=n_0}^{n-1} \mat W_{k,n}\frac{\vc R_k}{\sqrt{\gamma_k}}
        ,\quad n > n_0
        ,\\
        \bar{\vc y}_n &= \sum_{k=n_0}^{n-1}\gamma_k\bar{\mat W}_{k+1,n}\vc m_k + \bar{\mat W}_{n_0,n}(\vc x_{n_0}-\vc x^*) + \sum_{k=n_0}^{n-1}\sqrt{\gamma_{k+1}}\bar{\mat W}_{k+1,n}(\vc R_k^{\mathrm{nl}}+\vc R_k^{\mathrm P})
        ,\quad \bar n_0 > n_0
        ,
    \end{aligned}
    \end{equation}
    where $\mat W_{k,n}$ and $\bar{\mat W}_{k,n}$ for all $n_0 \leq k \leq n$ are defined as
    \begin{equation*}
        \mat W_{k, n} \coloneqq \sqrt{\gamma_k} \prod_{j=k+1}^{n-1} (\identity - \gamma_j \drift)
        , \quad \text{and} \quad
        \bar{\mat W}_{k,n}
        \coloneqq
        \frac{1}{\sqrt{n-\bar n_0}}
        \sum_{j=\max\{k,\bar n_0+1\}}^{n}
        \prod_{i=k}^{j-1}(\identity-\gamma_i\drift)
        .
    \end{equation*}
\end{lemma}
\Cref{lem:linearization} decomposes the remainder $\vc R_k$ into the sum $\vc R_k^{\mathrm{drift}}+\vc R_k^{\mathrm{nl}}+\vc R_k^{\mathrm{step}}+\vc R_k^{\mathrm P}$, which allows us to exploit the structure of each component separately and control their weighted sums in $L_p$ at the rates needed for our finite-sample Wasserstein bounds.

To obtain finite-sample Wasserstein bounds for $\vc y_n$ and $\bar{\vc y}_n$, we require quantitative $L_p$ bounds on the weighted remainder sums in \eqref{eq:yns}, i.e., the terms containing $\vc R_k$.
Their bounds are obtained by bounding individual terms in \eqref{eq:Rk} as follows.
\begin{enumerate}
    \item
    $\vc R_k^{\mathrm{drift}}$ accounts for the difference between $\drift_k$ and $\drift$.
    Using the fact that $\drift_k \to \drift$, in \Cref{sec:drift_mismatch}, we obtain bounds on $\gamma_k\lVert\drift_k-\drift\rVert$ and combine them with \Cref{ass:normalized} to bound the weighted sums of $\vc R_k^{\mathrm{drift}}$.

    \item When Assumptions~\ref{ass:step_size}, \ref{ass:regularity}, and \ref{ass:normalized} hold, then Taylor's expansion gives, for $1\leq p\leq q_y/2$ and $k\geq k_0(q_y)$,
    \begin{equation}\label{eq:barF_linearization}
        \left\lVert\bar F(\vc x_k)-\nabla\bar F(\vc x^*)(\vc x_k-\vc x^*)\right\rVert_{L_p}
        \leq \frac{L_\nabla}{2}\gamma_k\lVert\vc y_k\rVert_{L_{2p}}^2
        \leq \frac{L_\nabla}{2}\gamma_kB_{2p}^2.
    \end{equation}
    Thus, \eqref{eq:barF_linearization} leads to bounds on the weighted sums of $\vc R_k^{\mathrm{nl}}$ as shown in \Cref{sec:linearized_remainder}.

    \item The $L_p$ norms of weighted sums involving $\vc R_k^{\mathrm{step}}$ are controlled using the Burkholder-Rosenthal inequality (\Cref{lem:BR_weighted}) as shown in \Cref{sec:stepsize_normalization}.

    \item Bounds on the $L_p$ norms of weighted sums involving $\vc R_k^{\mathrm P}$ are obtained (\Cref{lem:martingale_approximation_full}) using regularity properties of the solution to the Poisson equation, the Burkholder--Rosenthal inequality, and summation by parts.
\end{enumerate}
Combining these bounds gives the following result.
\begin{lemma}\label{lem:remainder_bounds}
    Suppose \Crefrange{ass:regularity}{ass:normalized} hold with $1/q_y+1/q_V\leq1/2$.
    The following hold for all $1 \leq p \leq \min\{q_m, q_V /2, q_y / 2\}$.
    \begin{lemmasub}
        \item\label{lem:remainder_last}
        If \Cref{ass:step_size_last} holds with $\lambda\gamma_1>2$, then for all $n > n_0$,
        \begin{equation*}
            \left\lVert\sum_{k=n_0}^{n-1}\mat W_{k,n} \frac{\vc R_k}{\sqrt{\gamma_k}}\right\rVert_{L_p}
            = \mathcal{O}\left(\sqrt{\gamma_n} \right)
            .
        \end{equation*}

        \item\label{lem:remainder_pr}
        If \Cref{ass:step_size_pr} holds, then for all $n$ with $\bar n_0\geq n_0$,
        \begin{equation*}
            \left\lVert
                \sum_{k=n_0}^{n-1}\sqrt{\gamma_{k+1}}\bar{\mat W}_{k+1,n}(\vc R_k^{\mathrm{nl}}+\vc R_k^{\mathrm P})
            \right\rVert_{L_p}
            = \mathcal{O} \left(
                \sqrt{n} \gamma_n
            \right)
            .
        \end{equation*}
    \end{lemmasub}
\end{lemma}

\subsection{Gaussian Limits and Approximation Error}
With the error bound in \Cref{lem:remainder_bounds}, it is natural to compare the sequences $\vc y_n$ and $\bar{\vc y}_n$ in \eqref{eq:yns} to appropriate Gaussian limits.
Under Assumptions~\ref{ass:exogenous} and \ref{ass:noise_markov}, the asymptotic covariance of $\vc m$ in \eqref{eq:Mk} is given by
\begin{equation}\label{eq:noise_covariance}
    \mat \Gamma
    =\mathbb E_\pi[\vc m_1\vc m_1^\top]
    .
\end{equation}
Ignoring the remainder $\vc R_k$ in the recursion \eqref{eq:yk}, the sequence $\vc y$ resembles a discretization of the Ornstein-Uhlenbeck (OU) process
\begin{equation}\label{eq:OU}
    \vc Y_t = \int_0^t \mat W(t - \tau) \mat \Gamma^{1/2} d \vc B_\tau + e^{-\drift t} \vc Y_0  ,
    \quad \mat W(t - \tau) = e^{-\drift (t - \tau)}
\end{equation}
whose asymptotic covariance, under \Cref{ass:stability}, is $\mat \Sigma \coloneqq \int_0^\infty e^{-\drift t} \mat \Gamma e^{-\drift^\top t} dt$.
Unlike the OU process, $\vc y$ is driven by martingale differences that are generally neither Gaussian nor independent.

The normalized error of the PR average $\bar{\vc y}_n$ can be expressed as $\bar{\vc y}_n = \tfrac{1}{\sqrt{\lceil n / 2 \rceil}} \sum_{k=\bar n_0 + 1}^{n} \sqrt{\gamma_k}\vc y_k$.
This averaging along a discrete trajectory motivates a comparison with the sample path average $\bar{\vc Y}_t \coloneqq \tfrac{1}{\sqrt{t/2}} \int_{t/2}^t \vc Y_s ds$.
From the stochastic differential equation
\[
    d\vc Y_s
    =-\drift\vc Y_s\,ds+\mat\Gamma^{1/2}\,d\vc B_s
    \Leftrightarrow
    \vc Y_s\,ds
    =\drift^{-1}\mat\Gamma^{1/2}\,d\vc B_s
    -\drift^{-1}\,d\vc Y_s,
\]
we obtain that $\bar{\vc Y}_t$ can be expressed as
\begin{equation}\label{eq:pr_ou}
    \bar{\vc Y}_t
    =
    \frac{1}{\sqrt{t/2}}
    \int_{t/2}^t
    \drift^{-1}\mat\Gamma^{1/2}\,d\vc B_s
    +
    \frac{\drift^{-1}(\vc Y_{t/2}-\vc Y_t)}
        {\sqrt{t/2}}.
\end{equation}
The limiting covariance is then recovered as $\bar{\mat \Sigma} \coloneqq \drift^{-1} \mat \Gamma \drift^{-\top}$ when $\|\vc Y_0\|_{L_2} < \infty$.

It is then natural to compare $\vc y_n$ and $\bar{\vc y}_n$ to their respective continuous-time analogues $\vc Y_t$ and $\bar{\vc Y}_t$.
To this end, we define
\begin{equation}\label{eq:sn}
    \vc s_n = \sum_{k=n_0}^{n-1} \mat W_{k, n} \vc m_k
    \quad \text{and} \quad
    \bar{\vc s}_n = \sum_{k=n_0}^{n-1} \gamma_k \bar{\mat W}_{k+1, n} \vc m_k
\end{equation}
with respective covariances $\mat \Sigma_n$ and $\bar{\mat \Sigma}_n$.
The following makes the connection between the discrete-time weights and the continuous-time processes \eqref{eq:OU} and \eqref{eq:pr_ou} precise.
\begin{proposition}\label{prop:dtct}
    Under Assumptions~\ref{ass:step_size_last} and \ref{ass:stability} with $\lambda \gamma_1 > 2$,
    \begin{equation*}
        \left\|
            \frac{\mat W_{k,n}}{\sqrt{\gamma_k}}
            -\exp\left(-\drift\sum_{j=k+1}^{n-1}\gamma_j\right)
        \right\|
        =
        \mathcal O\left(
            \frac{1}{n}
        \right),
        \qquad \forall k \in [n_0, n)
        .
    \end{equation*}
    Under Assumptions~\ref{ass:step_size_pr} and \ref{ass:stability}, the following holds for all $\bar n_0\leq k<n$:
    \begin{equation*}
    \left\|
        \gamma_k\bar{\mat W}_{k+1,n}
        -\frac{\drift^{-1}}{\sqrt{n-\bar n_0}}
    \right\|
    =
    \mathcal O\left(\frac{1}{\sqrt n}
    \left(
        \frac1{n\gamma_n}
        +e^{-\lambda(n-k)\gamma_n/2}
    \right)\right)
    .
    \end{equation*}
\end{proposition}

Recall the definition of the Wasserstein-$p$ distance
\begin{equation}\label{eq:wasserstein_def}
    \Wass_p(\mu, \nu)
    = \inf \left(
    \mathbb E\lVert\vc u - \vc v\rVert^p
    \right)^{1/p},
    \qquad
    \mu = \mathrm{Law}(\vc u), \nu = \mathrm{Law}(\vc v)
    ,
\end{equation}
where the infimum is over all couplings $(\vc u, \vc v)$ with marginal laws $\mu$ and $\nu$.
\begin{lemma}\label{lem:covariance_convergence}
    In addition to Assumptions~\ref{ass:step_size}, \ref{ass:stability}, \ref{ass:exogenous} and \ref{ass:noise_markov}, suppose
    \begin{equation}\label{eq:varphi_covariance}
        \varphi(\xi)
            =
        P\bigl(\Phi(\vc x^*,\cdot)\Phi(\vc x^*,\cdot)^\top\bigr)(\xi)
        -(P\Phi)(\vc x^*,\xi)(P\Phi)(\vc x^*,\xi)^\top
        .
    \end{equation}
    satisfies $\sup_{\xi \in \mathcal S} \| \varphi(\xi) \| / V(\xi) < \infty$. Then, for all $p \geq 1$, the following hold:
    \begin{enumerate}
        \item
        If \Cref{ass:step_size_last} holds with $\lambda \gamma_1 > 2$ and $\mat \Sigma \succ 0$, then
        \begin{equation}
            \| \mat \Sigma_n - \mat \Sigma \| = \mathcal{O}\left(\frac{1}{n}\right)
            ,\quad
            \Wass_p\left(
                \normal(\vc 0,\mat\Sigma_n),
                \normal(\vc 0,\mat\Sigma)
                \right)
            =
            \mathcal{O}\left(\frac{1}{n-n_0}\right)
            .
        \end{equation}

        \item
        If \Cref{ass:step_size_pr} holds with $\bar{\mat \Sigma} \succ 0$, then
        \begin{equation*}
            \| \bar{\mat \Sigma}_n - \bar{\mat \Sigma} \| = \mathcal{O}(1/n \gamma_n)
            ,
            \quad
            \Wass_p\left(\normal(\vc0,\bar{\mat\Sigma}_n),\normal(\vc0,\bar{\mat\Sigma})\right)
            =
            \mathcal O \left(\frac{1}{n \gamma_n } \right) .
        \end{equation*}
    \end{enumerate}
\end{lemma}
\Cref{lem:covariance_convergence} establishes the convergence $\mat \Sigma_n \to \mat \Sigma$ and $\bar{\mat \Sigma}_n \to \bar{\mat \Sigma}$, and convergence of the corresponding Gaussian laws.
Under the assumptions of \Cref{lem:covariance_convergence}, the martingale CLT \citep{hallp1980martingalelimittheoryanditsapplication} implies $\vc s_n \to \mathcal N(\vc 0, \mat \Sigma)$ and $\bar{\vc s}_n \to \mathcal N (\vc 0, \bar{\mat \Sigma})$.
Compared to asymptotic CLTs, existing quantitative bounds assume stronger conditions on the underlying martingale difference sequences, or the best known convergence rates are not known to be optimal or depend on the assumed structure.
Instead of enumerating these distinct conditions, we operate under the following abstraction.
\begin{assumption}\label{ass:clt}
    There exists $q_{\mathrm{CLT}}$ such that, for all $p \in [1, q_{\mathrm{CLT}}]$, the following hold:
    \begin{assumptionsub}[label=(\roman*),ref=\theassumption(\roman*)]
        \item\label{ass:clt_last}
        The asymptotic covariance $\mat \Sigma $ is positive definite, and there exists $\rho_p$ such that for all $n > n_0$ and some $r > 0$,
        \begin{equation*}
            \Wass_p \left(\Law\left(\vc s_n\right),
            \normal(\vc 0, \mat \Sigma)  \right) \leq \frac{\rho_p}{(n - n_0)^r}.
        \end{equation*}
        \item\label{ass:clt_pr}
        The asymptotic covariance $\bar{\mat \Sigma}$ is positive definite, there exists $\bar{\rho}_p$ such that for all $n \geq n_0$ and some $\bar{r} > 0$,
        \begin{equation*}
            \Wass_p \left(\Law\left(\bar{\vc s}_n\right),
            \normal(\vc 0, \bar{\mat \Sigma})  \right) \leq \frac{\bar{\rho}_p}{(n-\bar n_0)^{\bar{r}}}.
        \end{equation*}
    \end{assumptionsub}
\end{assumption}
\Cref{ass:clt} is motivated by the recent quantitative CLTs \citep{bonis2020steinsmethodnormalapproximation,bonis2024improved,srikant2024CLT,zhang2026gaussianapproximationmultivariatemartingale,fang2026martingalecentrallimittheorems}.
When rates along the CLTs are known, we show that $\vc y_n \to \mathcal N(\vc 0, \mat \Sigma)$ and $\bar{\vc y}_n \to \mathcal N(\vc 0, \bar{\mat \Sigma})$ under the discrete-time conditions in Assumptions~\ref{ass:step_size}--\ref{ass:normalized}.

\subsection{Finite-Sample Wasserstein-\texorpdfstring{$p$}{p} Bounds}\label{sec:finite_sample_wasserstein_bounds}
Under \Crefrange{ass:step_size}{ass:clt} with $1/q_y + 1/q_V \leq 1/2$, our Wasserstein-$p$ bounds hold for all $p$ such that $1 \leq p \leq p_\star$, where
\begin{equation*}
    p_\star
    \coloneqq
    \min\{q_{\mathrm{CLT}}, q_m, q_V/2, q_y/2\}
    .
\end{equation*}
Applying \Cref{lem:remainder_last} and assuming the CLT in \Cref{ass:clt_last}, we obtain a bound on the distance between the law of $\vc y_n$ and the Gaussian reference $\mathcal{N}(\vc 0, \mat \Sigma)$.
\begin{theorem}\label{thm:last_iterate}
    Suppose the hypotheses of \Cref{lem:remainder_last} hold and \Cref{ass:clt} holds.
    Then, for all $n>n_0$ and $p\in[1,p_\star]$,
    \begin{equation}\label{eq:last_iterate}
        \Wass_p\left(\Law(\vc y_n),\normal(0,\mat\Sigma)\right)
        =
        \mathcal O \left(\frac{1}{(n-n_0)^r}
            + \sqrt{\gamma_n}
        \right)
        .
    \end{equation}
\end{theorem}
When \Cref{ass:clt} holds with exponent $r = 1/2$, the dominant rate in the bound of \Cref{thm:last_iterate} is $\sqrt{\gamma_n} = \mathcal{O}(1/\sqrt{n})$.
In \Cref{thm:last_iterate_full}, the bound's scaling with respect to $p$ are expressed through the bounds assumed in Assumptions~\ref{ass:exogenous}--\ref{ass:normalized}.
A tighter dependence on $p$ translates to sharper finite-time Gaussian approximations for higher orders and higher confidence for concentration inequalities.

When $\mathrm{Cov}(\bar{\vc s}_n) \to \bar{\mat \Sigma}$ and the CLT for $\bar{\vc y}_n$ holds, the asymptotic covariance of $\bar{\vc y}_n$ is $\bar{\mat\Sigma}\coloneqq\drift^{-1}\mat\Gamma\drift^{-\top}$.
The following establishes the distance from $\mathrm{Law}(\bar{\vc y}_n)$ to its limiting Gaussian reference.
\begin{theorem}\label{thm:pr_average}
    Suppose the hypotheses of \Cref{lem:remainder_pr} and \Cref{ass:clt} hold.
    Then, for all $n$ with $\bar n_0\geq n_0$ and $p\in[1,p_\star]$,
    \begin{equation*}
        \Wass_p\left(\Law(\bar{\vc y}_n),\normal(\vc0,\bar{\mat\Sigma})\right)
        =\mathcal O\left(
        \frac{1}{n^{\bar r}}
        +\sqrt{n} \gamma_n
        \right)
        .
    \end{equation*}
\end{theorem}
When $\bar r \geq 1/4$, choosing $a = 3/4$ gives the rate $\mathcal O(n^{-1/4})$.

For linear stochastic approximation with independent observations, \Cref{ass:clt_pr} can be verified to hold with $\bar r = \min\{1/2, 1-a\}$ the bound in \Cref{lem:remainder_pr} can be improved to $\mathcal O(n^{- \bar r} + \sqrt{\gamma_n})$, and the bound of \Cref{thm:pr_average} converges to zero for every $a \in (0, 1)$.
Choosing $a = 2/3$ then improves the correction rate from $n^{-1/4}$ to $n^{-1/3}$.

\subsection{From Wasserstein Bounds to Confidence Inequalities}\label{sec:confidence_transfer}
We use our finite-sample Wasserstein-$p$ bounds to obtain two high probability bounds through a coupling argument and Markov's inequality.
In both bounds, a Gaussian tail is asymptotically dominant, while finite-time corrections are determined by the bounds in \Cref{thm:last_iterate,thm:pr_average}.
To obtain high probability bounds, we use our full bounds in \Cref{thm:last_iterate_full,thm:pr_average_full} so that the dependence on $p$ can be used.

Let $\mathcal E_{n, p}$ denote the resulting upper bound in \Cref{thm:last_iterate_full}, so that $\Wass_p\!\left(\Law(\vc y_n), \Law(\mat\Sigma^{1/2}\vc z)\right) \leq \mathcal E_{n, p}$.
For any $t>\mathbb E\|\mat \Sigma^{1/2}\vc z\|$, any coupling $(\vc y_n, \mat\Sigma^{1/2}\vc z)$, and any $\vartheta > 0$,
\begin{equation*}
    \mathbb P\left(\lVert\vc y_n\rVert \geq t\right)
    \leq
    \left(
    \mathbb P\left(\lVert\mat\Sigma^{1/2}\vc z\rVert \geq t - \vartheta\right)
    +\mathbb P\left(\lVert\vc y_n - \mat\Sigma^{1/2}\vc z\rVert \geq \vartheta\right)
\right).
\end{equation*}
For any $\delta \in (0, 1)$, $p \in [1, p_\star]$, and $n$ for which this bound holds, consider $\vartheta > \mathcal E_{n, p} \delta^{-1/p}$.
Recall that a Gaussian vector $\vc z$ satisfies $\| \mat{\Sigma}^{1/2} \vc z\| \leq \mathcal{G}_{\mat \Sigma} (\delta - \mathcal E_{n, p}^p / \vartheta^p)$ with probability $\geq 1 - \delta + \mathcal E_{n, p}^p / \vartheta^p$.
We replace the two terms above by this Gaussian tail and the Wasserstein-$p$ bound (which holds for the optimal coupling) to obtain
\begin{equation}\label{eq:coupling_transfer}
    \mathbb P\left(
    \lVert \vc y_n\rVert
    \ge
    \mathcal G_{\mat\Sigma} \left(
    \delta-\frac{\mathcal E_{n,p}^p}{\vartheta^p}
    \right)+\vartheta
    \right)
        \le
    \delta
    .
\end{equation}
We obtain two types of concentration bounds using \eqref{eq:coupling_transfer}.

The first is closely related to an argument in \citep{liu2026wasserstein}.
By choosing $\vartheta = \mathcal E_{n, p} (2/\delta)^{1/p}$, the resulting bound in \eqref{eq:coupling_transfer} can be rearranged to obtain
\begin{equation*}
    \| \vc y_n \| \leq \mathcal{G}_{\mat \Sigma} \left(\frac{\delta}{2} \right)
    + \mathcal E_{n, p}\left(\frac{2}{\delta}\right)^{1/p}
    ,
    \quad \text{w.p. } \geq 1 - \delta .
\end{equation*}
For any $\delta \in (0, 1)$, set $p_\delta \coloneqq 2 \vee \log(2/\delta)$.
If Assumptions~\ref{ass:step_size}--\ref{ass:clt} hold such that $p_\delta$ is admissible, then we have by $p = p_\delta$ the bound
\begin{equation*}
    \| \vc y_n \| \leq \mathcal{G}_{\mat \Sigma} \left(\frac{\delta}{2} \right)
    + e \mathcal E_{n, p_\delta}
    ,
    \quad \text{w.p. } \geq 1 - \delta .
\end{equation*}
As a result, the correction $\mathcal E_{n, p_\delta}$ converges to zero at a rate determined by \Cref{thm:last_iterate}.
However, this choice requires Wasserstein bounds of order $p_\delta \asymp \log (1/\delta)$, and the Gaussian term remains evaluated at the more conservative level $\delta/2$ as $n \to \infty$.

The bound in \eqref{eq:coupling_transfer} can be used to obtain asymptotically tighter bounds at any admissible order $p$ independent of $\delta$, including $p=1$.
Suppose $n$ is sufficiently large that $\mathcal E_{n,p}^{p/(p+1)} < \delta$.
Choosing $\vartheta=\mathcal E_{n,p}^{p/(p+1)}$ in \eqref{eq:coupling_transfer} gives
\begin{equation}\label{eq:concentration_fixedp}
    \lVert\vc y_n\rVert
    \leq
    \mathcal G_{\mat\Sigma}\left(
    \delta-\mathcal E_{n,p}^{p/(p+1)}
    \right)
    +\mathcal E_{n,p}^{p/(p+1)},
    \qquad \text{w.p. }\geq1-\delta.
\end{equation}
The cost of using a fixed $p$ independent of $\delta$ is the correction $\mathcal E_{n,p}^{p/(p+1)}$, which decays more slowly than $\mathcal E_{n,p}$ when $\mathcal E_{n,p}<1$; its exponent approaches one as $p \to \infty$.
However, this bound is asymptotically tighter and recovers Gaussian tails in the following sense: for all $\delta \in (0, 1)$ and all $\varepsilon > 0$,
\begin{equation}\label{eq:confidence_convergence}
    \limsup_{n \to \infty} \mathbb{P}\left( \| \vc y_n \| \geq \mathcal{G}_{\mat \Sigma}(\delta) + \varepsilon \right)
    \leq  \delta
    .
\end{equation}

The same two arguments can be repeated for the Polyak--Ruppert average.
Let $\bar{\mathcal E}_{n, p}$ denote the upper bound in \Cref{thm:pr_average_full}.
When the step size exponent is chosen as $a \in (1/2, 1)$, this bound decays to zero as $n \to \infty$.
If $p_\delta \leq p_\star$ and $\bar{\mathcal E}_{n,p}^{p/(p+1)} < \delta$, repeating the above arguments gives
\begin{equation*}
    \| \bar{\vc y}_n \| \leq \min\left\{
        \mathcal{G}_{\bar{\mat \Sigma}} \left(\frac{\delta}{2} \right) + e\bar{\mathcal E}_{n, p_\delta} ,
        \mathcal{G}_{\bar{\mat \Sigma}} \left(\delta - \bar{\mathcal E}_{n, p}^{p / (p + 1)}\right) + \bar{\mathcal E}_{n, p}^{p/(p+1)}
    \right\}
\end{equation*}
which holds with probability $\geq 1 - \delta$.
Consequently, we obtain a convergence in the sense \eqref{eq:confidence_convergence} with $\vc y_n$ and $\mat \Sigma$ replaced by $\bar{\vc y}_n$ and $\bar{\mat \Sigma}$.

\subsection{Sharp Asymptotic Moment Bounds}\label{sec:moment}
\Cref{thm:last_iterate,thm:pr_average} assume an upper bound on the moments; see \Cref{ass:normalized}.
Note that the moment bound in \Cref{ass:normalized} is a crude bound, and is not indicative of the limiting moments.
As we will see in \Cref{sec:applications}, it is easy to establish that \Cref{ass:normalized} holds for many applications.
One of the key points of this paper is that, starting from such a crude bound, Wasserstein convergence gives the following two-sided bounds in terms of the Gaussian moments:
\begin{equation}\label{eq:wasserstein_moment_transfer}
    \begin{aligned}
        \left\lvert
        \lVert\vc y_n\rVert_{L_p} - \lVert\mat\Sigma^{1/2}\vc z\rVert_{L_p}
        \right\rvert
        \leq \mathcal E_{n,p}
        \quad
        \text{and}
        \quad
        \left\lvert
        \lVert\bar{\vc y}_n\rVert_{L_p} - \lVert\bar{\mat\Sigma}^{1/2}\vc z\rVert_{L_p}
        \right\rvert
        \leq \bar{\mathcal E}_{n,p},
    \end{aligned}
\end{equation}
for all $p \in [1,p_\star]$.
Taking the limit in each inequality whenever its right-hand side converges to zero, as discussed following the respective theorem statements, gives
\begin{equation*}
    \lVert\vc y_n\rVert_{L_p}
    \longrightarrow
    \lVert\mat\Sigma^{1/2}\vc z\rVert_{L_p}
    \qquad
    \text{and}
    \qquad
    \lVert\bar{\vc y}_n\rVert_{L_p}
    \longrightarrow
    \lVert\bar{\mat\Sigma}^{1/2}\vc z\rVert_{L_p}
    .
\end{equation*}

 \section{Applications}\label{sec:applications}
This section specializes \Cref{thm:last_iterate,thm:pr_average} to three applications: stochastic gradient descent (SGD) driven by a Markov chain, SGD with independent observations, and temporal difference (TD) learning where the operator is linear.
For each setting, we show that \Cref{ass:normalized,ass:clt} hold and convert the resulting bounds of \Cref{thm:last_iterate,thm:pr_average} to concentration inequalities following \Cref{sec:confidence_transfer}.
The quantities $q_y, r, \rho_p, \bar \rho_p$ in \Cref{ass:normalized,ass:clt} for which we verify \Cref{ass:clt} differ across applications, leading to application-specific concentration inequalities.

The last iterate and the PR average are understood to be generated from separate runs using their respective step size schedules.
For simplicity of exposition, we assume $\vc \zeta_k = \vc 0$.
Further, we present the bounds with problem-dependent constants that may depend on the dimension $d$.
The corollaries in this section use the explicit dependence on $p$ in the bounds of \Cref{thm:last_iterate_full,thm:pr_average_full}.

\subsection{Nonlinear Stochastic Approximation with Markovian Noise (MCMC)}\label{sec:nonlinear_markov}
We consider nonlinear stochastic approximation driven by a finite-state Markov chain, as in the wireless scheduling example of \Cref{sec:motivating_examples}.
\begin{settingassump}{}{NM}\label{ass:nm}
    The Markov chain $\{\xi_k\}$ is irreducible and aperiodic on a finite state space $\mathcal S$, and $\mat \Gamma \succ 0$.
    The algorithm is initialized on a bounded domain.
    In addition to \Cref{ass:stability}, the mean ODE $\dot{\vc x}=-\bar F(\vc x)$ is globally asymptotically stable with unique equilibrium $\vc x^*$, and $\bar F$ satisfies \Cref{ass:ode_infinity}.
\end{settingassump}
For every finite graph and $\eta>0$, the scheduler in \Cref{sec:motivating_examples} is an irreducible and aperiodic Markov chain on $\mathcal I(G)$ whose transition matrix does not depend on the iterate.
The ridge-logistic mean field is the gradient of a strongly convex objective, and its scaled limit is $\bar F_\infty(\vc x)=\kappa(\vc x-\vc x^*)$.
Thus, the example satisfies \Cref{ass:nm} under the stated step-size and regularity conditions.
Let $\varphi(t, \vc x)$ denote the solution to $\dot{\vc u}(t) = -\bar F(\vc u(t))$ initialized at $\vc u(0) = \vc x$.
Under \Cref{ass:regularity,ass:nm}, this flow is globally exponentially stable \citep{borkar2000ode,lauand2025markovian}: for some positive $b_e$ and $\rho_e$,
\begin{equation}\label{eq:global_exponential_flow}
    \lVert\varphi(t, \vc x) - \vc x^*\rVert
    \leq b_e e^{-\rho_e t}\lVert\vc x - \vc x^*\rVert,
    \qquad
    \forall t \geq 0.
\end{equation}
In this section, we assume $a \in (1/2, 1]$.
When $a = 1$, we further assume $\rho_e > \bar \gamma$.
Under Assumptions~\ref{ass:step_size}, \ref{ass:regularity}, and \ref{ass:nm}, \Cref{ass:normalized} with $q_y =2$ follows from \citep[Theorem~4]{borkar25}.
Next, we verify \Cref{ass:clt} using \citep[Theorem~1]{zhang2026gaussianapproximationmultivariatemartingale}.
\begin{proposition}\label{prop:nm_clt}
    Suppose Assumptions~\ref{ass:regularity} and \ref{ass:nm} hold.
    If \Cref{ass:step_size_last} holds with $\lambda \gamma_1 > 2$, then \Cref{ass:clt_last} holds with $q_{\mathrm{CLT}} = 1$ and $r = 1/2$.
    If \Cref{ass:step_size_pr} holds, then Assumption~\ref{ass:clt_pr} holds with $q_{\mathrm{CLT}} = 1$ and $\bar r = \min\{1/2, 1-a\}$.
\end{proposition}

For the PR average, consider $a = 3/4$.
Since the assumptions of \Cref{thm:last_iterate,thm:pr_average} hold with $p_\star = 1$, $r = \frac{1}{2}$, and $\bar r = \min \{1/2, 1-a\}$, there exist $C_{\mathrm{NM}}$ and $\bar C_{\mathrm{NM}}$ such that for all sufficiently large $n$,
\begin{equation}\label{eq:nm_w1_bound}
    \Wass_1\left(\Law(\vc y_n),\normal(0,\mat\Sigma)\right)
    \leq
    \frac{C_{\mathrm{NM}}}{\sqrt{n}}
    \quad \text{and} \quad
    \Wass_1\left(\Law(\bar{\vc y}_n),\normal(0,\bar{\mat\Sigma})\right)
    \leq
    \frac{\bar C_{\mathrm{NM}}}{n^{1/4}}
    .
\end{equation}
Using the method discussed in \Cref{sec:confidence_transfer}, we obtain the following high probability bounds.
\begin{corollary}\label{cor:nm}
    Suppose Assumptions~\ref{ass:regularity} and \ref{ass:nm} hold.
    \begin{enumerate}
        \item \emph{Last iterate.} Under \Cref{ass:step_size_last} with $\gamma_1 > \max\{\tfrac{1}{2\rho_e}, \tfrac{2}{\lambda}\}$, the following bound holds for every sufficiently large $n$ and any $\delta \in (0,1)$ satisfying $\delta n^{1/4} > \sqrt{C_{\mathrm{NM}}}$:
        \begin{equation*}
            \lVert\vc x_n - \vc x^*\rVert
            \leq \sqrt{\gamma_n}
            \left[
                \mathcal G_{\mat\Sigma}\left(\delta-
                \frac{\sqrt{C_{\mathrm{NM}}}}{n^{1/4}}\right)
                +
                \frac{\sqrt{C_{\mathrm{NM}}}}{n^{1/4}}
            \right],
            \qquad
            \text{w.p. } \geq 1 - \delta
            .
        \end{equation*}
        \item
        \emph{PR average.}
        Under \Cref{ass:step_size_pr} with $a = 3/4$, the following bound holds for every sufficiently large $n$, and any $\delta \in (0,1)$ satisfying $\delta n^{1/8} > \sqrt{\bar C_{\mathrm{NM}}}$:
        \begin{equation*}
            \lVert\bar{\vc x}_n - \vc x^*\rVert
            \leq \frac{1}{\sqrt{n-\bar n_0}}
            \left[
                \mathcal G_{\bar{\mat\Sigma}}
                \left(\delta-\frac{\sqrt{\bar C_{\mathrm{NM}}}}{n^{1/8}}\right)
                +
                \frac{\sqrt{\bar C_{\mathrm{NM}}}}{n^{1/8}}
            \right]
            \qquad \text{w.p. } \geq
            1 - \delta
            .
        \end{equation*}
    \end{enumerate}
\end{corollary}

 \subsection{Strongly Convex and Smooth Minimization}\label{sec:nonlinear_iid}

We consider minimization of strongly convex and smooth objectives using stochastic gradients.
\begin{settingassump}{}{ID}\label{ass:id}
    The sequence $\{\xi_k\}_{k \geq 0}$ is i.i.d. with law $\pi$, independent of $\vc x_1$, and $\mat\Gamma$ is positive definite.
    There exist $\mu > 0$ and $q_{\mathrm{ID}} \in [4, \infty)$ and $\sigma_{q_{\mathrm{ID}}} < \infty$ such that, for every $\vc x \in \mathbb R^d$,
    \begin{equation}\label{eq:independent_noise_moment_growth}
        \langle\vc x - \vc x^*, \bar F(\vc x)\rangle_{\mat Q}
        \geq \mu\lVert\vc x - \vc x^*\rVert^2
        \quad \text{and} \quad
        \mathbb E\left[
            \left\lVert F(\vc x, \xi_1) - \bar F(\vc x)\right\rVert^{q_{\mathrm{ID}}}
        \right]
        \leq
        \sigma_{q_{\mathrm{ID}}}^{q_{\mathrm{ID}}}
        \left(1 + \lVert\vc x - \vc x^*\rVert^{q_{\mathrm{ID}}}\right).
    \end{equation}
    The initial condition satisfies $\lVert\vc x_1 - \vc x^*\rVert_{L_{q_{\mathrm{ID}}}} < \infty$.
\end{settingassump}
Since the noise is i.i.d., \Cref{ass:markov} holds with $V(\xi) = 1 + \lVert F(\vc x^*, \xi)\rVert$ and $q_V=q_{\mathrm{ID}}$.
\begin{proposition}\label{prop:nonlinear_moment_stability}
    Under Assumptions~\ref{ass:step_size}, \ref{ass:regularity}, and \ref{ass:id} with $\mu > \bar \gamma$, there exists $N_{q_{\mathrm{ID}}}^{\mathrm{ID}}$ such that
    \begin{equation*}
        \sup_{k \geq N_{q_{\mathrm{ID}}}^{\mathrm{ID}}}\lVert\vc y_k\rVert_{L_{q_{\mathrm{ID}}}} \leq
        2\sigma_{q_{\mathrm{ID}}}
        \sqrt{\frac{q_{\mathrm{ID}} - 1}{\mu - \bar{\gamma}}}
        .
    \end{equation*}
\end{proposition}
\citet[Proposition~1]{gadat2022optimal} prove a similar moment bound, but we present our own result and proof for two reasons.
First, their result does not apply to $a = 1$, which is the exponent required for the fastest convergence rate of the last iterate.
Second, their result does not quantify the dependence on the order $q_{\mathrm{ID}}$.

We next use \citep[Theorems~5 and 6]{bonis2020steinsmethodnormalapproximation} to show that \Cref{ass:clt} holds under \Cref{ass:id}.
Define
\begin{equation}\label{eq:Psi_iid}
    \Psi_p^{\mathrm{iid}}
    =
    1+
    \|\mat\Gamma^{-1/2}\vc m_1\|_{L_4,2}^{3}
    +
    \|\mat\Gamma^{-1/2}\vc m_1\|_{L_{p+2},2}^{1+2/p}.
\end{equation}
\begin{proposition}\label{prop:clt_iid}
    Suppose Assumptions~\ref{ass:stability} and \ref{ass:id} hold.
    There exist constants $C_{\mathrm{ID}}^{\mathrm{CLT}}, \bar C_{\mathrm{ID}}^{\mathrm{CLT}}$ independent of $n$ and $p$ such that, for every $p\in[2,q_{\mathrm{ID}}-2]$, the following statements hold:
    \begin{propositionsub}
        \item\label{prop:clt_iid_last} Under \Cref{ass:step_size_last} with $\lambda \gamma_1 > 2$, the following holds for all $n > n_0$:
        \begin{equation}\label{eq:clt_iid_last}
            \Wass_p\left(\Law(\vc s_n),\normal(0,\mat\Sigma)\right)
            \leq
            C_{\mathrm{ID}}^{\mathrm{CLT}}\frac{p^{3/2}\Psi_p^{\mathrm{iid}}\sqrt d}{\sqrt{n-n_0}}
            .
        \end{equation}
        \item\label{prop:clt_iid_pr} Under \Cref{ass:step_size_pr}, the following holds for all $n \geq n_0$:
        \begin{equation*}
            \Wass_p\left(\Law(\bar{\vc s}_n),\normal(0,\bar{\mat \Sigma})\right)
            \leq
            \bar{C}_{\mathrm{ID}}^{\mathrm{CLT}}\left(\frac{p^{3/2}\Psi_p^{\mathrm{iid}}}{\sqrt{n}} + \frac{\| \vc z \|_{L_p, 2}}{n \gamma_n} \right)
            .
        \end{equation*}
    \end{propositionsub}
\end{proposition}
For the PR average, suppose $a = 3/4$.
Under the assumptions of \Cref{prop:nonlinear_moment_stability,prop:clt_iid}, \Cref{thm:last_iterate,thm:pr_average} hold for all $p \in [2, q_{\mathrm{ID}}/2]$.
Therefore, for some constants $C_{\mathrm{ID}} \geq 1$ and $\bar C_{\mathrm{ID}}>0$, the following holds for all sufficiently large $n$:
\begin{equation}\label{eq:Wp_id}
\begin{aligned}
    \Wass_p\left(
        \Law(\vc y_n),
        \normal(0,\mat\Sigma)
    \right)
    &\leq
    C_{\mathrm{ID}}\frac{p^{3/2}\Psi_p^{\mathrm{iid}}\sqrt d}{n^{1/2}},
    \\
    \Wass_p\left(
        \Law(\bar{\vc y}_n),
        \normal(0,\bar{\mat \Sigma})
    \right)
    &\leq
    \bar C_{\mathrm{ID}}\left(
    \frac{p^{3/2}\Psi_p^{\mathrm{iid}}}{n^{1/2}}
    + \frac{p}{n^{1/4}}
    \right).
\end{aligned}
\end{equation}
\begin{corollary}\label{cor:id}
    Suppose Assumptions~\ref{ass:regularity} and \ref{ass:id} hold.
    For every $\delta \in (0, 1)$ and all sufficiently large $n$, the following hold with probability $\geq 1 - \delta$:
    \begin{enumerate}
        \item
        Under \Cref{ass:step_size_last} with $\gamma_1>\max\{1/(2\mu),2/\lambda\}$,
        \begin{equation*}
            \| \vc y_n \| \leq \mathcal{G}_{\mat \Sigma}\left(
                \delta -
                C_{\mathrm{ID}}q_{\mathrm{ID}}^{3/2}
                \left(\Psi_{q_{\mathrm{ID}}/2}^{\mathrm{iid}}\sqrt{\frac{d}{n}}\right)^{\tfrac{q_{\mathrm{ID}}}{(q_{\mathrm{ID}}+2)}}
            \right)
                +
            C_{\mathrm{ID}}q_{\mathrm{ID}}^{3/2}
            \left(\Psi_{q_{\mathrm{ID}}/2}^{\mathrm{iid}}\sqrt{\frac{d}{n}}\right)^{\tfrac{q_{\mathrm{ID}}}{(q_{\mathrm{ID}}+2)}}
            .
        \end{equation*}
        \item
        Under \Cref{ass:step_size_pr} with $a=3/4$,
        \begin{equation*}
            \| \bar{\vc y}_n \| \leq \mathcal{G}_{\bar{\mat \Sigma}}\left(\delta -
                \left(\frac{q_{\mathrm{ID}} \bar C_{\mathrm{ID}}}{n^{1/4}}\right)^{\tfrac{q_{\mathrm{ID}}}{q_{\mathrm{ID}}+2}}
            \right)
            +
            \left(\frac{q_{\mathrm{ID}} \bar C_{\mathrm{ID}}}{n^{1/4}}\right)^{\tfrac{q_{\mathrm{ID}}}{q_{\mathrm{ID}}+2}}
            .
        \end{equation*}
    \end{enumerate}
\end{corollary}
The threshold on $n$ ensures that \eqref{eq:Wp_id} holds and that the bounds are well-defined.

 \subsection{Linear Stochastic Approximation with Markovian Noise}\label{sec:linear_markov}
Here we consider linear stochastic approximation with Markovian noise, where $F$ takes the form in \eqref{eq:lsa}.
Let $\mathcal A$ denote the centered solution of the Poisson equation $    \mat A-\bar{\mat A} = \mathcal A-P\mathcal A$.
Motivated by applications to TD learning with linear function approximation \citep{tsitsiklis1996analysis,SBbook2018}, we consider the following assumption.
\begin{settingassump}{}{LM}\label{ass:lm}
    In addition to Assumptions \ref{ass:stability} and \ref{ass:noise_markov}, $\mat\Gamma \succ 0$, and the functions $\mat A$, $\vc b$, $\mathcal A$, and $\Phi(\vc x^*,\cdot)$ are bounded uniformly in $\xi$.
    The initial iterate satisfies $\lVert\vc x_1\rVert_{L_p}<\infty$ for every $p\geq1$.
\end{settingassump}
By linearity, the Poisson equation \eqref{eq:poisson} is solved by $\Phi$ which satisfies
\begin{equation}\label{eq:linear_poisson_solution}
    \Phi(\vc x^* + \vc e, \xi)
    = \mathcal A(\xi)\vc e + \Phi(\vc x^*,\xi)
    ,
    \quad \forall \vc e \in \mathbb{R}^d, \xi \in \mathcal{S}.
\end{equation}
Under \Cref{ass:lm}, $\vc m$ and \(F(\vc x^*,\cdot)\) are bounded, and there exist constants $K_0, K_1$ such that
\begin{equation*}
    \lVert(P\Phi)(\vc x, \xi)\rVert
    \leq K_0 + K_1\lVert\vc x - \vc x^*\rVert,
    \qquad \forall \vc x \in \mathbb{R}^d,
    \xi \in \mathcal{S}
    .
\end{equation*}
Let \(V\) be the function in \Cref{ass:noise_markov}.
Since \(V\geq1\), it follows that \(\lVert F(\vc x^*,\xi)\rVert\leq K_FV(\xi)\) which implies \Cref{ass:poisson_moments}.
The following shows \Cref{ass:normalized} holds for every finite $q_y \geq 2$.
\begin{proposition}\label{prop:linear_markov_moment}
    Under Assumptions~\ref{ass:step_size} and \ref{ass:lm}, \Cref{ass:normalized} holds for every finite $q_y \geq 2$.
\end{proposition}
Finite-time \(L_p\) bounds for linear stochastic approximation with Markovian observations have been established in \citep{srikant19,durmus2021stability}.
However, we establish \Cref{prop:linear_markov_moment} because \citet{srikant19} considers constant step sizes and \citet{durmus2021stability} require stronger assumptions on the Markov chain compared to \Cref{ass:noise_markov}.

\citet{fang2026martingalecentrallimittheorems} establish a Wasserstein-$p$ CLT for Markov chains when $\vc m$ has finite $3p$ moments.
We use their results to verify \Cref{ass:clt}.
\begin{proposition}\label{prop:lm_clt}
    Suppose \Cref{ass:lm} holds.
    There exist constants $C_{\mathrm{LM}}^{\mathrm{CLT}},\bar C_{\mathrm{LM}}^{\mathrm{CLT}}$ independent of $n$ and $p$ such that, for every finite $p\geq1$, the following statements hold:
    \begin{propositionsub}
        \item\label{prop:lm_clt_last} Under \Cref{ass:step_size_last} with $\lambda\gamma_1>2$, for all $n>n_0$,
        \begin{equation}\label{eq:lm_clt_last}
            \Wass_p\left(\Law(\vc s_n),\normal(\vc0,\mat\Sigma)\right)
            \leq C_{\mathrm{LM}}^{\mathrm{CLT}}\frac{d^{1/6}\sqrt{p(d+p)}}{(n-n_0)^{1/6}}.
        \end{equation}
        \item\label{prop:lm_clt_pr} Under \Cref{ass:step_size_pr}, for all $n\geq n_0$,
        \begin{equation}\label{eq:lm_clt_pr}
            \Wass_p\left(\Law(\bar{\vc s}_n),\normal(\vc0,\bar{\mat \Sigma})\right)
            \leq \bar C_{\mathrm{LM}}^{\mathrm{CLT}}
            \left(\frac{p}{n^{1/6}} + \frac{\| \vc z \|_{L_p}}{n \gamma_n}\right).
        \end{equation}
    \end{propositionsub}
\end{proposition}
Under the assumptions of \Cref{prop:linear_markov_moment,prop:lm_clt}, \Cref{thm:last_iterate,thm:pr_average} hold with $p_\star = \infty$: for all $p \geq 1$ and all sufficiently large $n$,
\begin{equation}\label{eq:Wp_lm}
\begin{aligned}
    \Wass_p\left(
        \Law(\vc y_n),
        \normal(0,\mat\Sigma)
    \right)
    &\leq
    C_{\mathrm{LM}}\left(
    \frac{p}{n^{1/6}}+\frac{p^{3/2}}{n^{1/2}}
    \right),
    \\
    \Wass_p\left(
        \Law(\bar{\vc y}_n),
        \normal(0,\bar{\mat \Sigma})
    \right)
    &\leq
    \bar C_{\mathrm{LM}}\left(
    \frac{p}{n^{1/6}}
    +
    p^{3/2}\sqrt{\gamma_n} + \frac{\sqrt p}{n\gamma_n}
    + \sqrt{p} \sqrt{n} \gamma_n
    \right).
\end{aligned}
\end{equation}
In the following, we restrict $\delta \leq 2/e^2$ to simplify presentation so that $\log (2/\delta) \geq 2$.
\begin{corollary}\label{cor:lm}
    Suppose \Cref{ass:lm} holds.
    For every $\delta \in (0, 2/e^2]$ and all sufficiently large $n$, the following hold with probability $\geq 1 - \delta$:
    \begin{enumerate}[label=(\alph*),ref=\thecorollary(\alph*)]
        \item\label[corollary]{cor:lm_last}
        Under \Cref{ass:step_size_last} with $\gamma_1>2/\lambda$,
        \begin{equation*}
            \| \vc y_n \| \leq
            \left(1+\frac{C_{\mathrm{LM}}}{\sigma_{\mat Q}(\mat\Sigma)n^{1/6}}\right)
            \mathcal{G}_{\mat \Sigma}\left(\frac{\delta}{2}\right)
            +\frac{C_{\mathrm{LM}}}{n^{1/6}}\log\left(\frac{2}{\delta}\right)
            .
        \end{equation*}
        \item\label[corollary]{cor:lm_pr}
        Under \Cref{ass:step_size_pr} with $a \in [2/3, 5/6)$,
        \begin{equation*}
            \| \bar{\vc y}_n \| \leq
            \left(1+\frac{\bar C_{\mathrm{LM}}}{\sigma_{\mat Q}(\bar{\mat\Sigma})n^{1/6}}\right)
            \mathcal{G}_{\bar{\mat \Sigma}}\left(\frac{\delta}{2}\right)
            +\frac{\bar C_{\mathrm{LM}}}{n^{1/6}}\log\left(\frac{2}{\delta}\right)
            .
        \end{equation*}
    \end{enumerate}
\end{corollary}

If we further assume that the sequence $\{\xi_k\}$ is i.i.d., then the bound in \Cref{thm:pr_average_full} holds with $\bar c_{\mathrm{inc}}$ and $L_\nabla$ replaced by zero.
The rate $\bar r$ can be improved from $\min \{1/6, 1-a\}$ to $\min \{1/2, 1-a\}$ following the proof of \Cref{prop:clt_iid_pr}, and the convergence rate decays to zero for all $a \in (0, 1)$.
The exponent minimizing the bound becomes $a = 2/3$, and the corresponding dominant correction rate improves to $n^{-1/3}$.
  \section{Conclusion}
We developed a unified finite-sample Gaussian approximation framework for the last iterate and the PR average of stochastic approximation.
By combining moment stability of the normalized iterates with quantitative Gaussian approximation of the accumulated noise, we obtained Wasserstein-$p$ bounds for both outputs under nonlinear and linear dynamics with independent or Markovian observations.
Our results bridge asymptotic theory and existing finite-time bounds by showing that the high probability bounds consist of a Gaussian leading term and a vanishing correction that captures the heavier dependence on the confidence level.
Furthermore, the Wasserstein bounds establish convergence of admissible moments to the corresponding Gaussian moments, which follows from neither weak convergence alone nor one-sided finite-time moment bounds.
Applications to stochastic gradient descent and temporal-difference learning yield finite-sample high probability bounds with Gaussian leading terms and corrections that vanish with the sample size.
  
\appendix
\section{Supplement to \texorpdfstring{\protect\Cref{sec:wasserstein_bounds}}{Section \ref*{sec:wasserstein_bounds}}}\label{app:wasserstein_bounds}
In this section, we state and prove the full versions of our main results in \Cref{sec:wasserstein_bounds}.
To prove \Cref{lem:remainder_bounds} and provide explicit dependence on $p$ and other problem parameters, we first state \Cref{lem:martingale_approximation_full} which provides a bound on the $L_p$ norm of the weighted sum involving $\vc R_k^{\mathrm P}$.
The proofs of \Cref{lem:martingale_approximation_full,lem:remainder_bounds,lem:linearization} are deferred to \Cref{app:lemma_proofs}.

Let $\BRvar$ and $\BRinc$ be universal constants defined in \eqref{eq:BR}.
For $p \geq 1$, define
\begin{align*}
    \martingale_p
    &\coloneqq
    4L_\Phi\bigl(\BRvar\sqrt{p \vee 2} + \BRinc (p \vee 2)\bigr)
        B_{q_y}\seqnorm{V(\xi)}_{L_{q_V}}
        ,
    \\
    \update_p(k)
    &\coloneqq
    \sqrt{2\gamma_k}L_FB_{q_y}\seqnorm{V(\xi)}_{L_{q_V}}
    +K_F\seqnorm{V(\xi)}_{L_{q_V}}^2
    +\seqnorm{\vc\zeta}_{L_p}\seqnorm{V(\xi)}_{L_p},
    \\
    \telescope_p(k)
    &\coloneqq
    \seqnorm{\Phi(\vc x^*, \xi)}_{L_p}
    +\sqrt{2 \gamma_k} L_\Phi B_{q_y}
    \seqnorm{V(\xi)}_{L_{q_V}}.
\end{align*}
\begin{lemma}\label{lem:martingale_approximation_full}
    Suppose Assumptions~\ref{ass:step_size}, \ref{ass:regularity}, \ref{ass:exogenous}, \ref{ass:markov}, and \ref{ass:normalized} hold.
    Consider an interval $[k_1, k_2)$ with $k_1 \geq \max\{2, k_0(q_y) + 1\}$.
    Let $p$ be any $1 \leq p \leq \min \{q_m, q_V / 2\}$ such that $\frac1{p \vee 2} \geq \frac1{q_y} + \frac1{q_V}$.
    Then, for any deterministic sequence of $d\times d$ matrices $\{\mat D_{k,n}\}_{k=k_1}^{k_2-1}$,
    \begin{equation}\label{eq:weighted_poisson_remainder_full}
        \begin{aligned}
            & \quad \max_{k_1 \leq j \leq k_2}
            \left\lVert\sum_{k=k_1}^{j-1}\mat D_{k,n}
            \bigl(\widehat F(\vc x_k,\xi_k)+\vc\zeta_k+\vc m_k\bigr)\right\rVert_{L_p}
            \leq \martingale_p
            \left(\sum_{k=k_1}^{k_2-1}\lVert\mat D_{k,n}\rVert^2\gamma_{k-1}\right)^{1/2}
            \\ &+L_\Phi\update_p(k_1)\sum_{k=k_1}^{k_2-1}\lVert\mat D_{k,n}\rVert\gamma_{k-1}
            +\telescope_p(k_1)\max_{k_1<j\leq k_2}
            \left(\lVert\mat D_{k_1,n}\rVert+\lVert\mat D_{j-1,n}\rVert
            +\sum_{k=k_1}^{j-2}\lVert\mat D_{k+1,n}-\mat D_{k,n}\rVert\right)
            .
        \end{aligned}
    \end{equation}
    If the observations are independent, then
    \begin{equation}\label{eq:weighted_poisson_remainder_iid_full}
        \max_{k_1\leq j\leq k_2}
        \left\lVert\sum_{k=k_1}^{j-1}\mat D_{k,n}
        \bigl(\widehat F(\vc x_k,\xi_k)+\vc\zeta_k+\vc m_k\bigr)\right\rVert_{L_p}
        \leq \frac{\martingale_p}{4}
        \left(\sum_{k=k_1}^{k_2-1}\lVert\mat D_{k,n}\rVert^2\gamma_k\right)^{1/2}.
    \end{equation}
\end{lemma}
For $p \geq 1$, define
\begin{equation*}
    \rosbound_p(\vc m)
    \coloneqq
    \BRvar\sqrt{p \vee 2}\,\seqnorm{\condsd(\vc m \mid \mathcal H)}_{L_{p \vee 2}}
    +\BRinc (p \vee 2)\,\seqnorm{\vc m}_{L_{p \vee 2}}.
\end{equation*}
For every $1 \leq r \leq q_y$, define $B_r\coloneqq\sup_{k \geq k_0(q_y)}\lVert\vc y_k\rVert_{L_r}$.

\begin{lemmafull}{lem:remainder_bounds}\label{lem:remainder_bounds_full}
    Suppose \Crefrange{ass:regularity}{ass:normalized} hold with $1/q_y+1/q_V\leq1/2$.
    The following hold for all $1 \leq p \leq \min\{q_m, q_V /2, q_y / 2\}$.
    \begin{lemmasub}
        \item\label{lem:remainder_last_full} Suppose \Cref{ass:step_size_last} with $\lambda\gamma_1>2$ holds.
        Then, there exist constants $c_{\mathrm{drift}}$, $c_{\mathrm{sn}}$, $c_{\mathrm{nl}}$, $c_{\mathrm P}$, $c_{\mathrm{inc}}$, and $c_{\mathrm{bd}}$ independent of $p$ and $n$ such that for all $n > n_0$,
        \begin{equation*}
            \begin{aligned}
                \left\lVert\sum_{k=n_0}^{n-1}\mat W_{k,n} \frac{\vc R_k}{\sqrt{\gamma_k}}\right\rVert_{L_p}
                \leq \quad
                &\frac{c_{\mathrm{drift}}B_p+c_{\mathrm{sn}}\rosbound_p(\vc m)}{n}\\
                &+\sqrt{\gamma_n}\left(c_{\mathrm{nl}}B_{2p}^2+c_{\mathrm P}\martingale_p
                +c_{\mathrm{inc}}\update_p(n_0)
                +c_{\mathrm{bd}}\telescope_p(n_0)\right).
            \end{aligned}
        \end{equation*}

        \item\label{lem:remainder_pr_full} Suppose \Cref{ass:step_size_pr} holds.
        Then, there exist constants $\bar c_{\mathrm{nl}}$, $\bar c_{\mathrm{inc}}$, $\bar c_{\mathrm P}$, and $\bar c_{\mathrm{bd}}$ independent of $p$ and $n$ such that, for all $n$ with $\bar n_0\geq n_0$,
        \begin{equation*}
            \begin{aligned}
                \left\lVert
                    \sum_{k=n_0}^{n-1} \sqrt{\gamma_{k+1}} \bar{\mat W}_{k+1, n} (\vc R_k^{\mathrm{nl}} + \vc R_k^{\mathrm P})
                \right\rVert_{L_p}
                &\leq
                \left(\bar c_{\mathrm{nl}} L_\nabla B_{2p}^2
                +
                \bar c_{\mathrm{inc}} \update_p (n_0)
                \right)\sqrt{n} \gamma_n
                \\ & \quad+\bar c_{\mathrm P}\martingale_p\sqrt{\gamma_n}
                +
                \frac{\bar c_{\mathrm{bd}}\telescope_p (n_0)}{\sqrt{n}} .
            \end{aligned}
        \end{equation*}
    \end{lemmasub}
\end{lemmafull}

 \subsection{Last Iterate}\label{sec:last_iterate_proof}
\begin{theoremfull}{thm:last_iterate}\label{thm:last_iterate_full}
    Suppose the hypotheses of \Cref{lem:remainder_last_full} hold and \Cref{ass:clt} holds.
    Then, for all $n>n_0$ and $p\in[1,p_\star]$,
    \begin{equation}\label{eq:last_iterate_full}
        \begin{aligned}
            \Wass_p\left(\Law(\vc y_n),\normal(0,\mat\Sigma)\right)
            \leq {}&\frac{\rho_p}{(n-n_0)^r}
            +B_p\left(\frac{n_0}{n}\right)^{\lambda\gamma_1/2}
            +\frac{c_{\mathrm{drift}}B_p+c_{\mathrm{sn}}\rosbound_p(\vc m)}{n}\\
            &+\sqrt{\gamma_n}\left(
                c_{\mathrm{nl}}B_{2p}^2+c_{\mathrm P}\martingale_p
                +c_{\mathrm{inc}}\update_p(n_0)+c_{\mathrm{bd}}\telescope_p(n_0)
            \right).
        \end{aligned}
    \end{equation}
\end{theoremfull}
\begin{proof}{Proof}
Recall $\vc y_n$ from \Cref{lem:linearization}.
By \Cref{lem:remainder_last_full} and the triangle inequality, we have
\begin{equation*}
\begin{aligned}
    \Wass_p\left(\Law(\vc y_n),\normal(0,\mat\Sigma)\right)
    &\leq \frac{\rho_p}{(n-n_0)^r}
    +\left\lVert \left(\prod_{k=n_0}^{n-1} (\identity - \gamma_k \drift) \right) \vc y_{n_0} \right\rVert_{L_p}
    +\frac{c_{\mathrm{drift}}B_p+c_{\mathrm{sn}}\rosbound_p(\vc m)}{n}\\
    &\quad+\sqrt{\gamma_n}\left(c_{\mathrm{nl}}B_{2p}^2+c_{\mathrm P}\martingale_p
    +c_{\mathrm{inc}}\update_p(n_0)+c_{\mathrm{bd}}\telescope_p(n_0)\right).
\end{aligned}
\end{equation*}
We substitute for the second term the bound
\begin{equation*}
    \left\lVert \left(\prod_{k=n_0}^{n-1} (\identity - \gamma_k \drift) \right) \vc y_{n_0} \right\rVert_{L_p}
    \leq
    \prod_{k=n_0}^{n-1}\left(1 - \frac{1}{2} \lambda \gamma_k \right) B_p
    \leq
    \left(\frac{n_0}{n}\right)^{\lambda \gamma_1 / 2} B_p .
\end{equation*}

\end{proof}
 \subsection{Polyak--Ruppert Average}\label{sec:pr_average_proof}
\begin{theoremfull}{thm:pr_average}\label{thm:pr_average_full}
    Suppose the hypotheses of \Cref{lem:remainder_pr_full} and \Cref{ass:clt} hold.
    Then, for all $n\geq 4n_0$ and $p\in[1,p_\star]$, there exists $\bar c_{\mathrm{init}}$ such that
    \begin{equation*}
    \begin{aligned}
        \Wass_p\left(\Law(\bar{\vc y}_n),\normal(0,\bar{\mat\Sigma})\right)
        &\leq \frac{\bar\rho_p}{(n-\bar n_0)^{\bar r}}
        +\bar c_{\mathrm{init}}B_p e^{-\lambda\gamma_1 n^{1-a}/16}
        \\ &
        +\left(\bar c_{\mathrm{nl}}L_\nabla B_{2p}^2+\bar c_{\mathrm{inc}}\update_p(n_0)\right)\sqrt n\,\gamma_n
        +\bar c_{\mathrm P}\martingale_p\sqrt{\gamma_n}
        +\frac{\bar c_{\mathrm{bd}}\telescope_p(n_0)}{\sqrt n}.
    \end{aligned}
    \end{equation*}
\end{theoremfull}
\begin{proof}{Proof}
Recall $\bar{\vc y}_n$ from \eqref{eq:yns}.
By \Cref{ass:clt_pr}, \Cref{lem:remainder_pr_full}, and the triangle inequality,
\begin{equation*}
\begin{aligned}
    \Wass_p\left(
    \Law(\bar{\vc y}_{n}),
    \normal(0, \bar{\mat\Sigma})
    \right)
    &
    \leq
    \frac{\bar \rho_p}{(n-\bar n_0)^{\bar r}}
    +    \left\lVert \bar{\mat W}_{n_0, n} (\vc x_{n_0} - \vc x^*)\right\rVert_{L_p}
    \\ &
    +
    \left(\bar c_{\mathrm{nl}} L_\nabla B_{2p}^2
    +
    \bar c_{\mathrm{inc}} \update_p (n_0) \right)\sqrt{n} \gamma_n
    +\bar c_{\mathrm P}\martingale_p\sqrt{\gamma_n}
    +
    \frac{\bar c_{\mathrm{bd}}\telescope_p (n_0)}{\sqrt{n}} .
\end{aligned}
\end{equation*}
To obtain a bound on the second term above, we use that the matrices $\{\bar{\mat W}_{k,n}\}$ satisfy the recursion
\begin{equation}\label{eq:barW_recursion}
    \bar{\mat W}_{k,n}
    =(\identity-\gamma_k\drift)\bar{\mat W}_{k+1,n}
    +
    \frac{\identity}{\sqrt{n-\bar n_0}} \mathbf{1}\{k > \bar n_0\} ,
    \qquad
    \bar{\mat W}_{n,n}=\frac{\identity}{\sqrt{n-\bar n_0}}
    ,
\end{equation}
By \Cref{ass:stability}, we obtain using induction:
\begin{equation}\label{eq:barW_bound}
    \bar{\mat W}_{n_0,n}
    =\bar{\mat W}_{\bar n_0,n}
    \prod_{i=n_0}^{\bar n_0-1}(\identity-\gamma_i\drift) ,
    \qquad
    \lVert\bar{\mat W}_{k,n}\rVert
    \leq
    \frac{2\sqrt{2} }{\lambda} \cdot \frac{1}{\sqrt n \gamma_n},
    \qquad \bar n_0\leq k\leq n
    .
\end{equation}
Therefore, under \Cref{ass:normalized}, we have
\[
\begin{aligned}
\left\|\bar{\mathbf W}_{n_0,n}(\mathbf x_{n_0}-\mathbf x^*)\right\|_{L_p}
\leq
\|\bar{\mathbf W}_{\bar n_0,n}\|
\prod_{i=n_0}^{\bar n_0-1}\|I-\gamma_i\mathbf J\|
\|\mathbf x_{n_0}-\mathbf x^*\|_{L_p}
&\leq
\frac{2\sqrt2}{\lambda\sqrt n\,\gamma_n}
\prod_{i=n_0}^{\bar n_0-1}
\left(1-\frac{\lambda\gamma_i}{2}\right)
\sqrt{\gamma_{n_0}}B_p\\
&\leq
\frac{2\sqrt2\sqrt{\gamma_{n_0}}B_p}
{\lambda\sqrt n\,\gamma_n}
\exp\left(-\frac{\lambda}{2}
\sum_{i=n_0}^{\bar n_0-1}\gamma_i\right).
\end{aligned}
\]
For any $n\geq 4n_0$, this can be bounded as
\begin{equation*}
    \left\|\bar{\mathbf W}_{n_0,n}(\mathbf x_{n_0}-\mathbf x^*)\right\|_{L_p}
    \leq c_{\mathrm{init}} B_p
    \exp\left(-\frac{\lambda\gamma_1}{16}n^{1-a}\right),
    \quad
    c_{\mathrm{init}}
    \coloneqq
    \frac{2\sqrt{2}}
    {\lambda\sqrt{\gamma_1}\,n_0^{a/2}}
    \left(
        \frac{16a}{e\lambda\gamma_1(1-a)}
    \right)^{a/(1-a)}.
\end{equation*}
\end{proof}
 \section{Martingale $L_p$ Bounds}\label{sec:probability_estimates}
\begin{lemma}\label{lem:pinelis21}
    Assume $\| \vc X \|_{L_q}, \| \vc Z \|_{L_q} < \infty$ for some $q \geq 2$.
    Suppose $\vc X$ is $\mathcal H$-measurable and $\mathbb{E}[\vc Z \mid \mathcal H] = \vc 0$.
    Then,
    \begin{equation}\label{eq:conditional_Lq_smoothness}
        \lVert\vc X+\vc Z\rVert_{L_q}^2
        \leq
        \lVert\vc X\rVert_{L_q}^2
        +(q-1)\lVert\vc Z\rVert_{L_q}^2.
    \end{equation}
    Moreover, let $\{\vc Z_k\}_{k=k_1}^{k_2}$ be adapted to $\{\mathcal H_k\}$.
    If $\vc X$ is $\mathcal H_{k_1-1}$-measurable, $\|\vc X\|_{L_q}, \| \vc Z_k\|_{L_q} <\infty$, and $\mathbb{E}[\vc Z_k \mid \mathcal H_{k-1}] = \vc 0$, then
    \begin{equation}\label{eq:pinelis_repeated}
        \left\lVert\vc X+\sum_{k=k_1}^{k_2}\vc Z_k\right\rVert_{L_q}^2
        \leq
        \lVert\vc X\rVert_{L_q}^2
        +(q-1)\sum_{k=k_1}^{k_2}\lVert\vc Z_k\rVert_{L_q}^2.
    \end{equation}
\end{lemma}
\begin{proof}{Proof}
    The first part of \Cref{lem:pinelis21} is from \citep[Proposition~2.1]{pinelis1994optimum}, which we repeat for completeness.

    We first prove for $q > 2$.
    Consider $g(t)=\lVert\vc X+t\vc Z\rVert_{L_q}^2$, which satisfies
    \begin{align*}
        g'(0) = 2\|\vc X\|_{L_q}^{2-q}
            \mathbb E\left[\|\vc X\|^{q-2}\langle\vc X,\vc Z\rangle\right]
            = 2 \| \vc X \|_{L_q}^{2-q}
            \mathbb E\left[\|\vc X\|^{q-2}\langle\vc X, \mathbb{E}[\vc Z \mid \mathcal H] \rangle\right]
            = \vc 0
    \end{align*}
    and $g''(t)\leq2(q-1)\lVert\vc Z\rVert_{L_q}^2$.
    By Taylor's formula $g(1) = g(0) + g'(0) + \int_0^1 g''(t) (1 - t) dt$, we obtain \eqref{eq:conditional_Lq_smoothness}.
    For $q = 2$, the same inequality follows by direct expansion.

    To obtain the second part of \Cref{lem:pinelis21}, we apply \eqref{eq:conditional_Lq_smoothness}:
    \[
        \left\lVert\vc X+\sum_{k=k_1}^{j}\vc Z_k\right\rVert_{L_q}^2
        \leq
        \left\lVert\vc X+\sum_{k=k_1}^{j-1}\vc Z_k\right\rVert_{L_q}^2
        +(q-1)\lVert\vc Z_j\rVert_{L_q}^2
        ,
        \qquad
        k_1\leq j\leq k_2
        .
    \]
    Summing over $j=k_1,\ldots,k_2$ proves the second inequality.
\end{proof}

\subsection{Burkholder--Rosenthal Inequality}\label{sec:BR}

Let $\{\vc m_k\}$ be a martingale difference sequence with respect to $\{\mathcal H_k\}$.
The following is from \citep[Theorem 4.1]{pinelis1994optimum}, with universal constants from \citep[Table 1]{durmus2023rosenthal}, which sharpens the inequalities of \citep{rosenthal1970subspaces,burkholder1973distribution}.
Recall that the Hilbert space $\mathbb R^d$ with inner product $\langle\vc u, \vc v\rangle_{\mat Q} = \vc u\tp\mat Q\vc v$ is $(2, 1)$-smooth.
\begin{proposition}[Burkholder--Rosenthal]\label{prop:BR}
    Let $p \geq 2$.
    For any $k_2 > k_1$, it holds that
    \begin{equation}\label{eq:BR}
        \left\lVert\max_{k_1 \leq j < k_2}
        \left\lVert\sum_{k = k_1}^{j}\vc m_k\right\rVert\right\rVert_{L_p}
        \leq
        \BRvar\sqrt p
        \left\lVert\sum_{k = k_1}^{k_2 - 1}
        \mathbb E[\lVert\vc m_k\rVert^2 \mid \mathcal H_{k - 1}]
        \right\rVert_{L_{p/2}}^{1/2}
        +
        \BRinc p
        \left\lVert\max_{k_1 \leq k < k_2}\lVert\vc m_k\rVert\right\rVert_{L_p},
    \end{equation}
    where $\BRvar = 60 e$ and $\BRinc = 60$ are universal constants.
\end{proposition}
A consequence is the following result for weighted sums.
\begin{lemma}\label{lem:BR_weighted}
    Let $\{\mat D_{k,n}\}$ be deterministic $d\times d$ matrices and define $\vc S_j = \sum_{k = k_1}^{j}\mat D_{k,n}\vc m_k$.
    For $p \geq 2$,
    \begin{equation}\label{eq:BR_weighted_sharp}
        \begin{aligned}
            \left\lVert\max_{k_1 \leq j < k_2}\lVert\vc S_j\rVert\right\rVert_{L_p}
            \leq {}& \BRvar\sqrt p\,\seqnorm{\condsd(\vc m \mid \mathcal H)}_{L_p}
            \left(\sum_{k = k_1}^{k_2 - 1}\lVert\mat D_{k,n}\rVert^2\right)^{1/2} \\
                   &+\BRinc p\seqnorm{\vc m}_{L_p}
            \left(\sum_{k = k_1}^{k_2 - 1}\lVert\mat D_{k,n}\rVert^p\right)^{1/p}.
        \end{aligned}
    \end{equation}
    For every $p \geq 1$, the following weaker bound holds:
    \begin{equation}\label{eq:BR_weighted}
        \left\lVert\max_{k_1 \leq j < k_2}\lVert\vc S_j\rVert\right\rVert_{L_p}
        \leq \rosbound_p(\vc m)
        \left(\sum_{k = k_1}^{k_2 - 1}\lVert\mat D_{k,n}\rVert^2\right)^{1/2}
        .
    \end{equation}
\end{lemma}
\begin{proof}{Proof}
    Let $p \geq 2$ and apply \Cref{prop:BR} to the martingale difference sequence $\{\mat D_{k,n}\vc m_k\}$.
    Since $\mathbb E[\lVert\mat D_{k,n}\vc m_k\rVert^2\mid\mathcal H_{k-1}]\leq\lVert\mat D_{k,n}\rVert^2\mathbb E[\lVert\vc m_k\rVert^2\mid\mathcal H_{k-1}]$ and $p/2\geq1$, Minkowski's inequality in $L_{p/2}$ bounds the first term by $\BRvar\sqrt p\,\seqnorm{\condsd(\vc m \mid \mathcal H)}_{L_p}(\sum_k\lVert\mat D_{k,n}\rVert^2)^{1/2}$.
    For the second, $\mathbb E\max_k\lVert\mat D_{k,n}\vc m_k\rVert^p \leq \sum_k\lVert\mat D_{k,n}\rVert^p\mathbb E\lVert\vc m_k\rVert^p$ gives
    \begin{equation*}
        \left\lVert\max_{k_1 \leq k < k_2}\lVert\mat D_{k,n}\vc m_k\rVert\right\rVert_{L_p}
        \leq \seqnorm{\vc m}_{L_p}
        \left(\sum_{k = k_1}^{k_2 - 1}\lVert\mat D_{k,n}\rVert^p\right)^{1/p}.
    \end{equation*}
    Combining, we obtain \eqref{eq:BR_weighted_sharp}.
    Since $p \geq 2$ gives $(\sum_k\lVert\mat D_{k,n}\rVert^p)^{1/p} \leq (\sum_k\lVert\mat D_{k,n}\rVert^2)^{1/2}$, the bound \eqref{eq:BR_weighted} follows from \eqref{eq:BR_weighted_sharp} and the definition of $\rosbound_p$.
    The bound for $p < 2$ follows by applying the bound at $p = 2$ and applying Jensen's inequality.
\end{proof}

 \section{Central Limit Theorems}\label{sec:clt_bounds}

We use $\Wass_{p, 2}$ to denote \eqref{eq:wasserstein_def} with the weighted norm $\lVert\cdot\rVert$ replaced by the Euclidean norm $\lVert\cdot\rVert_2$.
The central limit theorems in this section are proved for the Euclidean norm.
Likewise, $\seqnorm{\cdot}_{L_r, 2}$ is the sequence $L_r$ norm \eqref{eq:rosenthal_quantity} with respect to the Euclidean norm $\lVert\cdot\rVert_2$.
Since $\lVert\vc v\rVert \leq \lVert\mat Q\rVert_2^{1/2}\lVert\vc v\rVert_2$ and $\lVert\vc v\rVert_2 \leq \lVert\mat Q^{-1}\rVert_2^{1/2}\lVert\vc v\rVert$,
\begin{equation}\label{eq:norm_conversion}
    \Wass_p(\Law(\vc u), \Law(\vc v)) \leq \lVert\mat Q\rVert_2^{1/2}\Wass_{p, 2}(\Law(\vc u), \Law(\vc v))
\end{equation}

In this section, we use results from \citep{bonis2020steinsmethodnormalapproximation,zhang2026gaussianapproximationmultivariatemartingale,fang2026martingalecentrallimittheorems} to verify \Cref{ass:clt} under the assumptions in \Cref{sec:applications}.
In particular, we assume $\vc \zeta_k = 0$ for simplicity of exposition.
For general martingale differences, \citet{fang2026martingalecentrallimittheorems} established finite-sample Wasserstein-$p$ bounds, but their rate is not optimal.
For $p = 1$, we use results from \citep{zhang2026gaussianapproximationmultivariatemartingale} to obtain improved rates.
When $\{\vc m_k\}$ is an i.i.d. sequence, we use results from \citep{bonis2020steinsmethodnormalapproximation} to obtain bounds with improved dependence on $p \geq 2$.
The results below assume $n > n_0$ and $p \geq 2$ unless stated otherwise.

We use $\{\mat D_{k,n}\}_{k=n_0}^{n-1}$ to denote a sequence of deterministic matrices.
For a martingale difference sequence $\vc m$, define $\mat V_n=\operatorname{Cov}\left(\sum_{k=n_0}^{n-1}\mat D_{k,n}\vc m_k\right)$ so that $\mat V_n = \mat \Sigma_n$ when $\mat D_{k,n} = \mat W_{k,n}$ and $\mat V_n = \bar{\mat \Sigma}_n$ when $\mat D_{k,n} = \gamma_k\bar{\mat W}_{k+1,n}$.
A common feature of \Cref{lem:nm_clt,lem:id_clt,lem:lm_clt} is that they require the weights and corresponding covariance to satisfy
\begin{equation}\label{eq:clt_conditions}
    \max_{n_0 \leq k < n} \| \mat D_{k,n} \|_2 = \mathcal{O}\left(\frac{1}{\sqrt{n - n_0}} \right) , \quad \mat V_n \succ 0, \quad \| \mat V_n^{-1} \| = \mathcal O(1)
    .
\end{equation}

When $\{\xi_k\}$ is uniformly ergodic and a further regularity on $\Phi$ is imposed, we have the following central limit theorem which we prove using \citep[Theorem 1]{zhang2026gaussianapproximationmultivariatemartingale}.
\begin{lemma}\label{lem:nm_clt}
    Suppose $\vc \zeta_k = 0$, $\{\xi_k\}$ is uniformly ergodic, and $\Phi(\vc x^*, \cdot)$ is bounded.
    If $\{\mat D_{k,n}\}$ is such that \eqref{eq:clt_conditions} hold, then
    \begin{align*}
        \Wass_{1}\left(\Law\left(\sum_{k=n_0}^{n-1}\mat D_{k,n}\vc m_k\right),\normal(\vc0,\mat V_n)\right)
        =
        \mathcal O\left(\frac{1}{\sqrt{n-n_0}} \right)
    \end{align*}
\end{lemma}
\begin{proof}{Proof}
    See \Cref{proof:nm_clt}.
\end{proof}

When $\{\vc m_k\}$ is an i.i.d. sequence, we obtain the following using \citep[Theorems 5 and 6]{bonis2020steinsmethodnormalapproximation}.
\begin{lemma}\label{lem:id_clt}
    Let $p\geq 2$, and let $\{\vc m_k\}$ be i.i.d., centered, with covariance $\mat\Gamma\succ0$ and finite $(p+2)$-moments.
    If $\{\mat D_{k,n}\}$ is such that \eqref{eq:clt_conditions} hold, then
    \begin{equation*}
        \Wass_p\!\left(
            \Law\!\left(\sum_{k=n_0}^{n-1}\mat D_{k,n}\vc m_k\right),
            \normal(\vc0,\mat V_n)\right)
        \leq
        C \frac{\lVert\mat Q\rVert_2^{1/2}\lVert\mat V_n^{1/2}\rVert_2
        (1+p)^{3/2}\Psi_p^{\mathrm{iid}}}{\sqrt{n - n_0}}
        .
    \end{equation*}
\end{lemma}
\begin{proof}{Proof}
    See \Cref{proof:id_clt}.
\end{proof}

For general martingale sequences $\{\vc m_k\}$ induced by bounded $\Phi$, denote $\mat \Gamma_k = \mathbb{E}[\vc m_k \vc m_k\tp \mid \mathcal H_{k-1}]$.
When $\Phi(\vc x^*, \cdot)$ is bounded, we obtain the following using results from \citep[Theorem 2.3(ii)]{fang2026martingalecentrallimittheorems}.
\begin{lemma}\label{lem:lm_clt}
    Suppose \Cref{ass:noise_markov} holds,
    $\Phi(\vc x^*,\cdot)$ is bounded, and $\vc\zeta_k=\vc0$.
    If $\{\mat D_{k,n}\}_{k=n_0}^{n-1}$ satisfies \eqref{eq:clt_conditions}, then there exists $C > 0$ such that for all $p\geq2$ and sufficiently large $n$,
    \[
        \Wass_p\left(
            \Law\left(\sum_{k=n_0}^{n-1}\mat D_{k,n}\vc m_k\right),
            \normal(\vc0,\mat V_n)
        \right)
        \leq
        C\lVert\mat Q\rVert_2^{1/2}
        \frac{d^{1/6}\sqrt p\,\lVert\vc z\rVert_{L_p,2}}
        {(n-n_0)^{1/6}}.
    \]
\end{lemma}
\begin{proof}{Proof}
    See \Cref{proof:lm_clt}.
\end{proof}

 \section{Supplement to \texorpdfstring{\protect\Cref{sec:applications}}{Section \ref*{sec:applications}}}\label{sec:application_confidence_proofs}
Recall $\vc \zeta = 0$ is assumed to simplify exposition.
In proving moment bounds of the form \Cref{ass:normalized}, we use the following inequalities.

\paragraph{Moment Bounds.}
In establishing moment bounds for \Cref{sec:applications}, we use the following properties.
Suppose $\lVert\vc x_1-\vc x^*\rVert_{L_q}<\infty$.
If
\begin{equation*}
    \lVert F(\vc X,\xi_k)+\vc\zeta_k\rVert_{L_q}
    \leq C_q\left(1+\lVert\vc X-\vc x^*\rVert_{L_q}\right)
\end{equation*}
holds for all $\mathcal H_{k-1}$-measurable random vectors $\vc X\in L_q$, then
\begin{equation}\label{eq:finite_moment}
    \lVert\vc x_{k+1}-\vc x^*\rVert_{L_q}
    \leq (1+C_q\gamma_k)\lVert\vc x_k-\vc x^*\rVert_{L_q}+C_q\gamma_k.
\end{equation}
It follows by induction that $\vc x_k\in L_q$, and hence $\vc y_k\in L_q$, for every fixed $k$.

\paragraph{Central Limit Theorems.}

Assumptions~\ref{ass:nm}, \ref{ass:id}, and \ref{ass:lm} all imply that, for some $c_m,C_m>0$ and all sufficiently large $k$,
\begin{equation}\label{eq:Emm_bounds}
    c_m \identity \preceq \mathbb E[\vc m_k\vc m_k^\top]\preceq C_m\identity
\end{equation}
and for some $K_{P}$,
\begin{equation}\label{eq:covariance_convergence_primitive}
    \sup_{\xi\in\mathcal S}
    \frac{\left\|
    P\bigl(\Phi(\vc x^*,\cdot)\Phi(\vc x^*,\cdot)^\top\bigr)(\xi)
    -(P\Phi)(\vc x^*,\xi)(P\Phi)(\vc x^*,\xi)^\top
    \right\|}{V(\xi)}
    <K_P.
\end{equation}
Therefore, under Assumptions~\ref{ass:nm}, \ref{ass:id}, and \ref{ass:lm}, \eqref{eq:covariance_convergence_primitive} holds and $\mat \Sigma_n$ and $\bar{\mat \Sigma}_n$ converge to their stationary limits sufficiently fast as established in \Cref{lem:covariance_convergence}, and therefore,
\begin{equation*}
    \Wass_p\left(
    \normal(\vc 0,\mat\Sigma_n),
    \normal(\vc 0,\mat\Sigma)
    \right)
    = \mathcal O\left( \frac{1}{n- n_0} \right)
    \quad \text{and} \quad
    \Wass_p\left(\normal(\vc0,\bar{\mat\Sigma}_n),\normal(\vc0,\bar{\mat\Sigma})\right)
    = \mathcal O\left(\frac{1}{n \gamma_n } \right)
\end{equation*}
as established in \Cref{lem:covariance_convergence}.

Recall that \Cref{lem:nm_clt,lem:id_clt,lem:lm_clt} require \eqref{eq:clt_conditions}.
\Cref{lem:covariance_convergence} establishes the two conditions on $\mat V_n$ when $\mat D_{k,n} = \mat W_{k,n}$ which implies $\mat V_n = \mat \Sigma_n$, and when $\bar{\mat D}_{k,n} = \gamma_k \bar{\mat W}_{k+1, n}$ which implies $\mat V_n = \bar{\mat \Sigma}_n$.
We next show that the bound on $\|\mat D_{k,n}\|$ holds.
\begin{proposition}\label{prop:clt_conditions}
    Under Assumptions~\ref{ass:step_size_last} and \ref{ass:stability} with $\lambda\gamma_1>2$, there exists $C_W>0$ such that $\max_{n_0\leq k<n} \lVert\mat W_{k,n}\rVert_2 \leq \frac{C_W}{\sqrt{n-n_0}}$.
    Under Assumptions~\ref{ass:step_size_pr} and \ref{ass:stability}, there exists $\bar C_W>0$ such that $\max_{n_0 \leq k < n}\gamma_k  \| \bar{\mat W}_{k+1, n} \| \leq \frac{\bar C_W}{\sqrt{n}}$.
\end{proposition}
\begin{proof}{Proof}
    See \Cref{proof:clt_conditions}.
\end{proof}
Therefore, the assumptions of \Cref{lem:nm_clt,lem:id_clt,lem:lm_clt} hold for the weight matrices $\{\mat W_{k,n}\}$ and $\{\bar{\mat W}_{k,n}\}$.
The remainder of this section verifies \Cref{ass:normalized} and the remaining conditions of \Cref{lem:nm_clt,lem:id_clt,lem:lm_clt} under \Cref{ass:nm,ass:id,ass:lm}, respectively.
We omit the proofs of the corollaries, as they follow immediately using the argument in \Cref{sec:confidence_transfer}.

\subsection{Nonlinear Operators and Markovian Noise}\label{sec:nm_supplement}

We first define the ODE-at-infinity condition used in \Cref{sec:nonlinear_markov}.
For $c \geq 1$, define $\bar F_c(\vc x) = c^{-1}\bar F\bigl(\vc x^* + c(\vc x - \vc x^*)\bigr)$ and let $\varphi_c(t, \vc x)$ denote the solution to $\dot{\vc u}(t) = -\bar F_c(\vc u(t))$ initialized at $\vc u(0) = \vc x$.
In particular, $\bar F_1 = \bar F$ and $\varphi_1 = \varphi$.
\begin{settingassump}{ODE stability at infinity}{ODE$_\infty$}\label{ass:ode_infinity}
    The scaled vector field $\bar F_\infty$ exists, in the sense that $\bar F_c \to \bar F_\infty$ uniformly on compact subsets of $\mathbb R^d$ as $c \to \infty$.
    The ODE at infinity $\dot{\vc u}(t) = -\bar F_\infty(\vc u(t))$ is globally asymptotically stable with $\vc x^*$ as its unique equilibrium.
\end{settingassump}

\paragraph{Proof of \Cref{prop:nm_clt}.}
Under \Cref{ass:nm}, the Markov chain $\{\xi_k\}$ is uniformly ergodic and $\Phi(\vc x^*, \cdot)$ is bounded.
Together with the conditions \eqref{eq:clt_conditions} verified earlier, we obtain \Cref{prop:nm_clt} using \Cref{lem:nm_clt}.

 \subsection{Nonlinear Operators with Independent Noise}\label{sec:ni_supplement}
We verify \Cref{ass:normalized,ass:clt} under Assumptions~\ref{ass:step_size}, \ref{ass:regularity}, and \ref{ass:id}.

\paragraph{Proof of \Cref{prop:nonlinear_moment_stability}.}
By Assumptions~\ref{ass:regularity} and \ref{ass:id}, the finite moment bound in \eqref{eq:finite_moment} holds with $C_{q_{\mathrm{ID}}} = L_F + \sigma_{q_{\mathrm{ID}}}$.
Further, we have $\mu \leq L_F$ and the two inequalities
\begin{align*}
    \lVert\vc x_k - \vc x^* - \gamma_k\bar F(\vc x_k)\rVert^2
    &\leq \left(1 - 2\mu\gamma_k + L_F^2\gamma_k^2\right)\lVert\vc x_k - \vc x^*\rVert^2
    \\
    \lVert\widehat F(\vc x_k, \xi_k)\rVert_{L_{q_{\mathrm{ID}}}}^2
    &\leq \sigma_{q_{\mathrm{ID}}}^2\left(1 + \mathbb E\lVert\vc x_k - \vc x^*\rVert^{q_{\mathrm{ID}}}\right)^{2/q_{\mathrm{ID}}}
    \leq \sigma_{q_{\mathrm{ID}}}^2\left(1 + \lVert\vc x_k - \vc x^*\rVert_{L_{q_{\mathrm{ID}}}}^2\right)
    .
\end{align*}
Combining with \Cref{lem:pinelis21}, we obtain
\begin{align*}
    \lVert\vc x_{k + 1} - \vc x^*\rVert_{L_{q_{\mathrm{ID}}}}^2
        &\leq
    \lVert\vc x_k - \vc x^* - \gamma_k\bar F(\vc x_k)\rVert_{L_{q_{\mathrm{ID}}}}^2
    +(q_{\mathrm{ID}} - 1)\gamma_k^2\lVert\widehat F(\vc x_k, \xi_k)\rVert_{L_{q_{\mathrm{ID}}}}^2
    \\ &\leq
    \left(1 - 2\mu\gamma_k + L_F^2\gamma_k^2\right)\lVert\vc x_k - \vc x^*\rVert_{L_{q_{\mathrm{ID}}}}^2
    +(q_{\mathrm{ID}} - 1)\gamma_k^2 \sigma_{q_{\mathrm{ID}}}^2\left(1 + \lVert\vc x_k - \vc x^*\rVert_{L_{q_{\mathrm{ID}}}}^2\right)
    .
\end{align*}
Define $C_{q_{\mathrm{ID}}}^{\mathrm{ID}} = L_F^2 + (q_{\mathrm{ID}} - 1)\sigma_{q_{\mathrm{ID}}}^2$.
Substituting $\vc y_k = \gamma_k^{-1/2}(\vc x_k - \vc x^*)$, we obtain the recursion
\begin{equation*}
    \lVert\vc y_{k + 1}\rVert_{L_{q_{\mathrm{ID}}}}^2 \leq \frac{\gamma_k}{\gamma_{k + 1}}\left(1 - 2\mu\gamma_k + C_{q_{\mathrm{ID}}}^{\mathrm{ID}}\gamma_k^2\right)\lVert\vc y_k\rVert_{L_{q_{\mathrm{ID}}}}^2 + (q_{\mathrm{ID}} - 1)\sigma_{q_{\mathrm{ID}}}^2\frac{\gamma_k}{\gamma_{k + 1}}\gamma_k.
\end{equation*}
For all sufficiently large $k$ such that
\begin{equation}\label{eq:explicit_iid_moment_threshold}
    0 \leq 1 - (\mu - \bar\gamma)\gamma_k,
    \qquad
    \frac{\gamma_k}{\gamma_{k + 1}}
    \left(1 - 2\mu\gamma_k + C_{q_{\mathrm{ID}}}^{\mathrm{ID}}\gamma_k^2\right)
    \leq 1 - (\mu - \bar\gamma)\gamma_k ,
\end{equation}
we obtain, for all such values of $k$, the recursion
\begin{equation}\label{eq:quadratic_iid_moment_recursion}
    \lVert\vc y_{k + 1}\rVert_{L_{q_{\mathrm{ID}}}}^2 \leq \left(1 - (\mu - \bar\gamma)\gamma_k\right)\lVert\vc y_k\rVert_{L_{q_{\mathrm{ID}}}}^2 +2 (q_{\mathrm{ID}} - 1)\sigma_{q_{\mathrm{ID}}}^2\gamma_k.
\end{equation}
Solving the recursion, we obtain \Cref{prop:nonlinear_moment_stability}.

\paragraph{Proof of \Cref{prop:clt_iid}.}
Under \Cref{ass:id}, the assumptions of \Cref{lem:id_clt} hold for all $p \leq q_{\mathrm{ID}} - 2$, as verified in \eqref{eq:clt_conditions}.
By \Cref{lem:covariance_convergence}, we obtain \Cref{prop:clt_iid}.

 \subsection{Linear Operators with Markovian Noise}\label{sec:lm_supplement}

We first exploit the additional structure assumed in \Cref{ass:lm} to refine the bounds in \Cref{thm:last_iterate,thm:pr_average}.
From \eqref{eq:linear_poisson_solution}, we have by $\vc e = \vc x - \vc x^*$ the identity
\[
    \qquad
    (P\Phi)(\vc x,\xi)
    =(P\mathcal A)(\xi)(\vc x-\vc x^*)+(P\Phi)(\vc x^*,\xi).
\]
Since $\mathcal A$ and $\Phi(\vc x^*,\cdot)$ are bounded, there exist constants $K_0$ and $K_1$ such that, for all $x, x', \xi$,
\begin{equation*}
    \begin{aligned}
        \lVert(P\Phi)(\vc x, \xi)\rVert
        \leq K_0 + K_1\lVert\vc x - \vc x^*\rVert
        \quad \text{and} \quad
        \lVert(P\Phi)(\vc x, \xi) - (P\Phi)(\vc x', \xi)\rVert
        \leq K_1\lVert\vc x - \vc x'\rVert
    \end{aligned}
\end{equation*}
From \eqref{eq:lsa}, $\bar{F}(\vc x) = \bar{\mat A}(\vc x - \vc x^*)$.
This linearity implies $\vc R_k^{\mathrm{nl}}$ in \eqref{eq:Rk} is zero, and hence $c_{\mathrm{nl}}$ and $\bar c_{\mathrm{nl}}$ in \Cref{lem:remainder_bounds_full} are zero.
Moreover, $\poismds_k $ in \eqref{eq:poisson_md_def} becomes $(\mathcal A(\xi_k)-(P\mathcal A)(\xi_{k-1})) (\vc x_{k-1}-\vc x^*)$ and hence
\[
    \left\lVert
        \frac{\poismds_k}{\sqrt{\gamma_{k-1}}}
    \right\rVert_{L_p}
    \leq C\lVert\vc y_{k-1}\rVert_{L_p}.
\]

\paragraph{Proof of \Cref{prop:linear_markov_moment}.}
Under \Cref{ass:lm}, the operators are of the form \eqref{eq:lsa}, and the Poisson solution has the affine form in \eqref{eq:linear_poisson_solution}.
For $k \geq 1$, define the martingale difference
\begin{equation}\label{eq:linear_iterate_martingale}
    -\widehat{\vc m}_k
    = \Phi(\vc x_k, \xi_k) - (P\Phi)(\vc x_k, \xi_{k - 1}) + \vc\zeta_k.
\end{equation}
Then $\mathbb E[\widehat{\vc m}_k \mid \mathcal H_{k - 1}] = 0$.
For $k \geq 1$, define $\tilde{\vc y}_k$ such that
\begin{equation}\label{eq:linear_corrected_iterate}
    \sqrt{\gamma_k}\widetilde{\vc y}_k
    = \sqrt{\gamma_k} \vc y_k - \gamma_k(P\Phi)(\vc x_k, \xi_{k - 1}).
\end{equation}
Taking the $L_q$ norms with $q \geq 1$ and rearranging, we obtain for sufficiently large $k$ such that $K_1 \gamma_k \leq 1/2$,
\begin{equation}\label{eq:linear_corrector_comparison}
    \lVert\vc y_k\rVert_{L_q}
    \leq \frac{\lVert\widetilde{\vc y}_k\rVert_{L_q} + K_0\sqrt{\gamma_k}}{1 - K_1\gamma_k},
\end{equation}
It remains to obtain a bound on $\| \widetilde{\vc y}_k \|_{L_q}$.

From \eqref{eq:sa_boilerplate}, \eqref{eq:linear_poisson_solution}, \eqref{eq:linear_iterate_martingale}, and \eqref{eq:linear_corrected_iterate}, there exist constants $K_m$ and $K_x$ such that
\begin{equation*}
    \lVert\widehat{\vc m}_k\rVert
    \leq
    K_m\left(1 + \sqrt{\gamma_k}\lVert\widetilde{\vc y}_k\rVert\right)
    \quad \text{and} \quad
    \lVert\vc x_{k + 1} - \vc x_k\rVert
    \leq K_x\gamma_k\left(1 + \sqrt{\gamma_k}\lVert\widetilde{\vc y}_k\rVert\right)
    .
\end{equation*}
Under \Cref{ass:lm}, the Poisson equation \eqref{eq:linear_poisson_solution} and \eqref{eq:linear_iterate_martingale} give
\begin{align*}
    F(\vc x_k,\xi_k)+\vc\zeta_k
        &=
    \bar{\mat A}(\vc x_k-\vc x^*)
    -\widehat{\vc m}_k
    +(P\Phi)(\vc x_k,\xi_{k-1})
    -(P\Phi)(\vc x_k,\xi_k).
\end{align*}
Substituting into \eqref{eq:sa_boilerplate}, subtracting $\gamma_{k+1} (P \Phi)(\vc x_{k+1}, \xi_k)$, and using \eqref{eq:linear_corrected_iterate} give
\begin{equation*}
    \begin{aligned}
        \sqrt{\gamma_{k + 1}}\widetilde{\vc y}_{k + 1}
        ={}& (\identity - \gamma_k\bar{\mat A})\sqrt{\gamma_k}\widetilde{\vc y}_k
        +\gamma_k\widehat{\vc m}_k
        +(\gamma_k - \gamma_{k + 1})(P\Phi)(\vc x_k, \xi_k) \\
        &+\gamma_{k + 1}\left((P\Phi)(\vc x_k, \xi_k) - (P\Phi)(\vc x_{k + 1}, \xi_k)\right)
        -\gamma_k^2\bar{\mat A}(P\Phi)(\vc x_k, \xi_{k - 1}).
    \end{aligned}
\end{equation*}
The growth and Lipschitz bounds above, the bound on $\vc x_{k + 1}-\vc x_k$, the monotonicity of the step size, and
$\gamma_k - \gamma_{k + 1} \leq a\gamma_k^2/\gamma_1$ give a finite constant $K_R$, independent of $q$ and $k$, such that
\begin{align*}
    &\quad (\gamma_k - \gamma_{k + 1})
    \lVert(P\Phi)(\vc x_k, \xi_k)\rVert_{L_q}
    +\gamma_{k + 1}K_1\lVert\vc x_{k + 1}-\vc x_k\rVert_{L_q}
    +\gamma_k^2\lVert\bar{\mat A}\rVert
    \lVert(P\Phi)(\vc x_k, \xi_{k - 1})\rVert_{L_q} \\
    &\leq K_R\gamma_k^2
    \left(1 + \sqrt{\gamma_k}\lVert\widetilde{\vc y}_k\rVert_{L_q}\right).
\end{align*}
We next establish a bound on $\| \widetilde{\vc y}_k\|_{L_q}$.
\begin{lemma}\label{lem:linear_poisson_moment_drift}
    Let $q \geq 2$ and define $    C_q^{\mathrm{LM}}
    =1+(1+\gamma_1^2)
    \left(\lVert\bar{\mat A}\rVert^2+2(q-1)K_m^2+2K_R^2\right)$.
    Under Assumptions~\ref{ass:step_size} and \ref{ass:lm}, the following holds for every $k$ such that $K_1 \gamma_k \leq 1/2$:
    \begin{equation}\label{eq:linear_poisson_drift}
        \lVert\widetilde{\vc y}_{k + 1}\rVert_{L_q}^2
        \leq \frac{\gamma_k}{\gamma_{k + 1}}
        \left(1 - 2(\lambda + \bar\gamma)\gamma_k + C_q^{\mathrm{LM}}\gamma_k^2\right)
        \lVert\widetilde{\vc y}_k\rVert_{L_q}^2
        +C_q^{\mathrm{LM}}\frac{\gamma_k}{\gamma_{k + 1}}\gamma_k.
    \end{equation}
\end{lemma}

\begin{proof}{Proof}
    See \Cref{proof:linear_poisson_moment_drift}.

\end{proof}

We use \Cref{lem:linear_poisson_moment_drift} to prove \Cref{prop:linear_markov_moment}.
For any $q \geq 2$, there exists a $q$-dependent threshold such that for all sufficiently large $k$ exceeding that threshold,
\begin{equation}\label{eq:linear_moment_threshold}
    K_1\gamma_k \leq 1/2,
    \qquad
    0 \leq 1 - \lambda\gamma_k,
    \qquad
    \frac{\gamma_k}{\gamma_{k + 1}}
    \left(1 - 2(\lambda + \bar\gamma)\gamma_k + C_q^{\mathrm{LM}}\gamma_k^2\right)
    \leq 1 - \lambda\gamma_k
    .
\end{equation}
For such values of $k$, the sequence $\tilde{\vc y}$ in \eqref{eq:linear_poisson_drift} satisfies the recursion
\begin{equation*}
    \lVert\widetilde{\vc y}_{k + 1}\rVert_{L_q}^2
    \leq \left(1 - \lambda\gamma_k\right)
    \lVert\widetilde{\vc y}_k\rVert_{L_q}^2
    + 2 C_q^{\mathrm{LM}}\gamma_k.
\end{equation*}
Solving this recursion, we obtain that there exists a constant $N_q^{\mathrm{LM}}$ such that \eqref{eq:linear_moment_threshold} holds and
\begin{equation*}
    \sup_{k \geq N_q^{\mathrm{LM}}}\lVert\widetilde{\vc y}_{k}\rVert_{L_q}^2
    \leq
    \frac{4 C_q^{\mathrm{LM}}}{\lambda}
    .
\end{equation*}
The comparison \eqref{eq:linear_corrector_comparison} proves $\sup_{k \geq N_q^{\mathrm{LM}}} \| \vc y_k \|_{L_q} \leq B_q^{\mathrm{LM}}$ with
\begin{equation*}
    B_q^{\mathrm{LM}}
    = \frac{2 \sqrt{C_q^{\mathrm{LM}} / \lambda} + K_0\sqrt{\gamma_{N_q^{\mathrm{LM}}}}}{1 - K_1\gamma_{N_q^{\mathrm{LM}}}}.
\end{equation*}

\paragraph{Proof of \Cref{prop:lm_clt}.}
Under \Cref{ass:lm}, we have \eqref{eq:clt_conditions} and the conditions of \Cref{lem:lm_clt} hold.
Therefore, \Cref{prop:lm_clt} follows directly from \Cref{lem:covariance_convergence,lem:lm_clt}.

  \section{Proofs of \texorpdfstring{Lemmas~\ref{lem:linearization}, \ref{lem:remainder_bounds_full}, \ref{lem:covariance_convergence}, and~\ref{lem:martingale_approximation_full}}{Lemmas \ref*{lem:linearization}, \ref*{lem:remainder_bounds_full}, \ref*{lem:covariance_convergence}, and \ref*{lem:martingale_approximation_full}}}\label{app:lemma_proofs}
We first establish a regularity property of the solution to Poisson's equation under \Cref{ass:regularity,ass:markov}.
Recall the constants $C_V < \infty$ and $\varrho \in (0, 1)$ of the $V$-uniform ergodicity bound \eqref{eq:V_ergodic}, and $K_F$ from \Cref{ass:poisson_moments}.
\begin{lemma}\label{lem:poisson_regularity}
    Suppose \Cref{ass:regularity,ass:markov} hold.
    The solution $\Phi$ to the Poisson equation \eqref{eq:poisson} satisfies
    \begin{equation*}
        \lVert\Phi(\vc x^*, \xi)\rVert
        \leq \frac{C_V K_F}{1 - \varrho}V(\xi),
        \quad \text{and} \quad
        \lVert(P\Phi)(\vc x^*, \xi)\rVert
        \leq \frac{C_V K_F}{1 - \varrho}(PV)(\xi).
    \end{equation*}
    Moreover, $\tfrac{\Phi(\cdot, \xi)}{V(\xi)}$ and $\tfrac{(P \Phi)(\cdot, \xi)}{(PV)(\xi)}$ are $L_\Phi$--Lipschitz continuous in $\vc x$ for all $\xi$.
\end{lemma}

\begin{proof}{Proof}
    By \citep[Theorem 17.4.2]{meyn2012markov}, the centered solution is
    \begin{equation*}
        \Phi(\vc x, \xi) = \sum_{j \geq 0}
        \left(P^j\widehat F(\vc x, \cdot)\right)(\xi).
    \end{equation*}
    Since $\bar F(\vc x^*) = 0$, we have $\widehat F(\vc x^*, \cdot) = F(\vc x^*, \cdot)$.
    Therefore, the $V$-ergodicity \eqref{eq:V_ergodic} and \Cref{ass:poisson_moments} give
    \begin{equation*}
        \lVert(P^j\widehat F(\vc x^*, \cdot))(\xi)\rVert
        \leq C_VK_F\varrho^jV(\xi).
    \end{equation*}
    Summing over $j \geq 0$ proves the first bound.
    Applying $P$ to that bound proves the second.
    By \Cref{ass:regularity} and $V \geq 1$,
    \begin{equation*}
        \lVert\widehat F(\vc x, \xi) - \widehat F(\vc x', \xi)\rVert
        \leq 2 L_F \lVert\vc x - \vc x'\rVert
        \leq 2L_F\lVert\vc x - \vc x'\rVert V(\xi).
    \end{equation*}
    Apply \eqref{eq:V_ergodic} to $\varphi = \widehat F(\vc x, \cdot) - \widehat F(\vc x', \cdot)$, whose supremum ratio is at most $2L_F\lVert\vc x - \vc x'\rVert$, and use the fact that $\widehat F$ is centered under $\pi$.
    Then
    \begin{equation*}
        \big\lVert(P^j(\widehat F(\vc x, \cdot) - \widehat F(\vc x', \cdot)))(\xi)\big\rVert
        \leq 2C_VL_F\varrho^j\lVert\vc x - \vc x'\rVert V(\xi).
    \end{equation*}
    Summing the geometric series proves the Lipschitz bound for $\Phi$.
    Applying $P$ to this bound proves the corresponding bound for $P\Phi$.
\end{proof}

\subsection{Proof of \texorpdfstring{\protect\Cref{lem:linearization}}{Lemma \ref*{lem:linearization}}}
We first shift both sides of \eqref{eq:sa_boilerplate} by $\vc x^*$ and divide by $\gamma_{k + 1}^{1/2}$ to express the recursion in terms of the normalized errors $\vc y_k = \gamma_k^{-1/2}(\vc x_k - \vc x^*)$:
\begin{equation}\label{eq:yk_raw_update}
    \vc y_{k + 1} = \sqrt{\frac{\gamma_k}{\gamma_{k + 1}}} \vc y_k - \frac{\gamma_k}{\sqrt{\gamma_{k + 1}}} \big(F(\vc x_k, \xi_k) + \vc \zeta_k \big).
\end{equation}
Next, we decompose $F(\vc x_k, \xi_k) + \vc\zeta_k$ into its mean $\bar F(\vc x_k)$ and centered noise $\widehat F(\vc x_k, \xi_k) + \vc\zeta_k$.
Since $\bar F(\vc x^*) = 0$ and $\vc x_k - \vc x^* = \sqrt{\gamma_k}\vc y_k$, we have
\begin{equation*}
    \bar F(\vc x_k)
    = \sqrt{\gamma_k}\nabla\bar F(\vc x^*)\vc y_k
    +\left(
    \bar F(\vc x_k)
    -\nabla\bar F(\vc x^*)(\vc x_k - \vc x^*)
\right).
\end{equation*}
The term $\vc R_k^{\mathrm P}$ arises from substituting
$\widehat F(\vc x_k,\xi_k)+\vc\zeta_k
= -\vc m_k
+\bigl(\widehat F(\vc x_k,\xi_k)+\vc\zeta_k+\vc m_k\bigr)$ into \eqref{eq:yk_raw_update}.
Substituting these mean and noise decompositions back into \eqref{eq:yk_raw_update} yields
\begin{align*}
    \vc y_{k + 1}
     &= \sqrt{\frac{\gamma_k}{\gamma_{k + 1}}}\left(\identity - \gamma_k \nabla\bar F(\vc x^*)\right) \vc y_k + \frac{\gamma_k}{\sqrt{\gamma_{k + 1}}} \vc m_k + \vc R_k^{\mathrm{nl}} + \vc R_k^{\mathrm{P}}
    \\
     &= \left(\identity - \gamma_k \drift_k\right) \vc y_k + \frac{\gamma_k}{\sqrt{\gamma_{k + 1}}} \vc m_k + \vc R_k^{\mathrm{nl}} + \vc R_k^{\mathrm{P}}
    .
\end{align*}
Using the identity $-\gamma_k\drift_k\vc y_k = -\gamma_k\drift\vc y_k + \vc R_k^{\mathrm{drift}}$ and splitting the noise coefficient as $\frac{\gamma_k}{\sqrt{\gamma_{k + 1}}}\vc m_k = \sqrt{\gamma_k}\vc m_k + \vc R_k^{\mathrm{step}}$ gives the stated normalized recursion.

The solution $\vc y_n$ is obtained by solving the recursion for $\vc y$ in the first part of \Cref{lem:linearization}.
To obtain the solution $\bar{\vc y}_n$, observe
\begin{align*}
    -\gamma_k
    \left(F(\vc x_k,\xi_k)-\nabla\bar F(\vc x^*)(\vc x_k-\vc x^*)+\vc\zeta_k\right)
    =
    \gamma_k\vc m_k
    +\sqrt{\gamma_{k+1}}(\vc R_k^{\mathrm{nl}}+\vc R_k^{\mathrm P}).
\end{align*}
The solution to the recursion \eqref{eq:x_linearization} is
\begin{align*}
    \vc x_k-\vc x^*
    &=\prod_{i=n_0}^{k-1}(\identity-\gamma_i\drift)(\vc x_{n_0}-\vc x^*)
    +\sum_{j=n_0}^{k-1}\prod_{i=j+1}^{k-1}(\identity-\gamma_i\drift)
    \left[\gamma_j\vc m_j+\sqrt{\gamma_{j+1}}(\vc R_j^{\mathrm{nl}}+\vc R_j^{\mathrm P})\right]
    .
\end{align*}
Summing over $k=\bar n_0+1,\ldots,n$ and dividing by $\sqrt{n-\bar n_0}$, we obtain
\begin{align*}
    \bar{\vc y}_n
    =
    \bar{\mat W}_{n_0,n}(\vc x_{n_0}-\vc x^*)
    +\frac{1}{\sqrt{n-\bar n_0}}
    \sum_{k=\bar n_0}^{n-1}\sum_{j=n_0}^{k}
    \prod_{i=j+1}^{k}(\identity-\gamma_i\drift)
    \left[\gamma_j\vc m_j+\sqrt{\gamma_{j+1}}(\vc R_j^{\mathrm{nl}}+\vc R_j^{\mathrm P})\right].
\end{align*}
Interchanging the sums for the last term, we obtain
\begin{equation*}
\begin{aligned}
    & \quad \frac{1}{\sqrt{n-\bar n_0}}
    \sum_{k=\bar n_0}^{n-1}\sum_{j=n_0}^{k}
    \prod_{i=j+1}^{k}(\identity-\gamma_i\drift)
    \left[\gamma_j\vc m_j
    +\sqrt{\gamma_{j+1}}(\vc R_j^{\mathrm{nl}}+\vc R_j^{\mathrm P})\right]
    \\ &=
    \sum_{j=n_0}^{n-1}
    \bar{\mat W}_{j+1, n}
    \left(\gamma_j\vc m_j
    +\sqrt{\gamma_{j+1}}(\vc R_j^{\mathrm{nl}}+\vc R_j^{\mathrm P})\right)
    .
\end{aligned}
\end{equation*}

\subsection{Proof of \texorpdfstring{\protect\Cref{lem:remainder_bounds_full}}{Lemma \ref*{lem:remainder_bounds_full}}}\label{sec:remainder}
Here we obtain bounds on the $L_p$ norm of the two quantities
\begin{equation*}
    \sum_{k=n_0}^{n-1} \frac{\mat W_{k,n}}{\sqrt{\gamma_k}} \vc R_k
    \quad \text{and} \quad
    \sum_{k=n_0}^{n-1} \sqrt{\gamma_{k+1}} \bar{\mat W}_{k+1, n} (\vc R_k^{\mathrm{nl}} + \vc R_k^{\mathrm P})
    ,
\end{equation*}
where $\vc R_k$ was defined in \Cref{lem:linearization}.
The bounds involving $\vc R_k^{\mathrm{drift}}$, $\vc R_k^{\mathrm{nl}}$, $\vc R_k^{\mathrm{step}}$, and $\vc R_k^{\mathrm P}$ are obtained in the following subsections.

\subsubsection{Time Varying Drift.}\label{sec:drift_mismatch}
Recall $\vc R_k^{\mathrm{drift}} = - \gamma_k (\drift_k - \drift) \vc y_k$.
We show that
\begin{equation*}
    \left\lVert
        \sum_{k=n_0}^{n - 1} \mat W_{k, n}\frac{\vc R_k^{\mathrm{drift}}}{\sqrt{\gamma_k}}
    \right\rVert_{L_p}
    = \mathcal{O}\left(\frac{B_p}{n} \right) .
\end{equation*}

Suppose \Cref{ass:normalized} holds and $p \leq q_y$.
Define $\driftmis{k} = \gamma_k \| \drift_k - \drift \|$, where by the Minkowski inequality, we obtain
\begin{equation}\label{eq:drift_mismatch_bounds1}
\begin{aligned}
    \left\lVert \sum_{k=n_0}^{n-1}\sqrt{\gamma_k}\mat W_{k,n}(\drift_k-\drift)\vc y_k\right\rVert_{L_p}
    &\leq
     B_p\sum_{k=n_0}^{n-1}
    \left(\frac{(k+1)^2}{n^2}+\frac{C}{n}\right)\driftmis{k}
    \\
    \frac{1}{\sqrt{n-\bar n_0}}\left\lVert\sum_{k=\bar n_0}^{n-1}\sqrt{\gamma_k}(\drift_k-\drift)\vc y_k\right\rVert_{L_p}
    &\leq\frac{B_p}{\sqrt{n-\bar n_0}}\sum_{k=\bar n_0}^{n-1}\frac{\driftmis{k}}{\sqrt{\gamma_k}},
\end{aligned}
\end{equation}
where the first line used \Cref{prop:dtct}.
Next, we obtain a bound on $\driftmis{k}$.
\begin{lemma}\label{lem:drift_mismatch}
    Suppose \Cref{ass:step_size,ass:stability} hold and $k\geq1$ satisfies $2\gamma_k\lVert\drift\rVert^2\leq\lambda$.
    Then
    \begin{equation*}
        \begin{cases}
            \displaystyle
            \driftmis{k}
            \leq \frac{a}{2k},
            & a<1,\\[8pt]
            \displaystyle
            \driftmis{k}
            \leq \left(\frac{\gamma_1\lVert\drift\rVert}{2}+\frac38\right)\frac1{k^2},
            & a=1.
        \end{cases}
    \end{equation*}
\end{lemma}
\begin{proof}{Proof}
    Write $\theta_k=\sqrt{\gamma_k/\gamma_{k+1}}-1=(1+1/k)^{a/2}-1\geq0$.
    Substituting $\nabla\bar F(\vc x^*)=\drift+\bar\gamma\identity$ into \eqref{eq:finite_step_drift} gives
    \begin{equation}\label{eq:drift_mismatch_identity}
        \drift_k-\drift
        =\left(\bar\gamma-\frac{\theta_k}{\gamma_k}\right)\identity
        +\theta_k\left(\drift+\bar\gamma\identity\right).
    \end{equation}
    Suppose first that $a<1$, so that $\bar\gamma=0$.
    Then \eqref{eq:drift_mismatch_identity} gives
    \begin{equation*}
        \lVert\drift_k-\drift\rVert
        =\frac{\theta_k}{\gamma_k}
        \lVert\identity-\gamma_k\drift\rVert
        \leq \frac{\theta_k}{\gamma_k}
        .
    \end{equation*}
    By $\theta_k\leq a/(2k)$, we obtain $\gamma_k \| \drift_k - \drift \| \leq a/2k$.

    For $a=1$, we use $\bar\gamma=1/(2\gamma_1)$ and $\gamma_k=\gamma_1/k$.
    Taylor's theorem $1 + x/2 - x^2/8 \leq \sqrt{1 + x} \leq 1 + x/2$ gives
    \begin{equation*}
        0\leq\frac{1}{2k}-\theta_k\leq\frac{1}{8k^2},
        \qquad
        0\leq\theta_k\leq\frac{1}{2k}.
    \end{equation*}
    Substituting into \eqref{eq:drift_mismatch_identity} and using the triangle inequality, we obtain
    \begin{equation*}
        \begin{aligned}
            \gamma_k\lVert\drift_k-\drift\rVert
            \leq \left|\frac{1}{2k}-\theta_k\right|
            +\gamma_k\theta_k\left(\lVert\drift\rVert+\frac{1}{2\gamma_1}\right)
            \leq \frac{1}{k^2}\left(\frac{\gamma_1\lVert\drift\rVert}{2}+\frac38\right).
        \end{aligned}
    \end{equation*}
\end{proof}
Substituting the bound for $\driftmis{k}$ to \eqref{eq:drift_mismatch_bounds1}, we obtain for some $c_{\mathrm{drift}}>0$.
\begin{equation*}
    \left\lVert \sum_{k=n_0}^{n-1}\sqrt{\gamma_k}\mat W_{k,n}(\drift_k-\drift)\vc y_k\right\rVert_{L_p}
    \leq
    CB_p\sum_{k=n_0}^{n-1}\frac1{k^2}
    \left(\frac{(k+1)^2}{n^2}+\frac{C}{n}\right)
    \leq \frac{c_{\mathrm{drift}}B_p}{n}
\end{equation*}
 \subsubsection{Linearized Mean Field.}\label{sec:linearized_remainder}
Recall $\vc R_k^{\mathrm{nl}}$ from \eqref{eq:Rk}.
We obtain a bound on
\begin{equation*}
    \left\lVert\sum_{k=n_0}^{n-1}\frac{\mat W_{k,n}}{\sqrt{\gamma_k}}\vc R_k^{\mathrm{nl}}\right\rVert_{L_p}
    \quad \text{and} \quad
    \left\lVert
        \sum_{k=n_0}^{n-1}\sqrt{\gamma_{k+1}}\bar{\mat W}_{k+1,n}\vc R_k^{\mathrm{nl}}
    \right\rVert_{L_p}
\end{equation*}

Suppose \Cref{ass:step_size,ass:regularity,ass:normalized} hold, $\bar F$ is nonlinear, $1\leq p\leq q_y$, and $2p\leq q_y$.
Substituting \eqref{eq:barF_linearization} into \eqref{eq:Rk} and applying Minkowski's inequality, we obtain using \Cref{prop:dtct} that for some $c_{\mathrm{nl}}$,
\begin{align*}
    \left\lVert\sum_{k=n_0}^{n-1}\frac{\mat W_{k,n}}{\sqrt{\gamma_k}}\vc R_k^{\mathrm{nl}}\right\rVert_{L_p}
    \leq \frac{L_\nabla B_{2p}^2}{2}\sum_{k=n_0}^{n-1}
    \left(\frac{(k+1)^2}{n^2}+\frac{C}{n}\right)\frac{\gamma_k^2}{\sqrt{\gamma_{k+1}}}
    \leq c_{\mathrm{nl}}B_{2p}^2\sqrt{\gamma_n} .
\end{align*}

For the PR average, we use \eqref{eq:barF_linearization} to obtain
\begin{equation*}
    \left\lVert
        \sum_{k=n_0}^{n-1}\sqrt{\gamma_{k+1}}\bar{\mat W}_{k+1,n}\vc R_k^{\mathrm{nl}}
    \right\rVert_{L_p}
    \leq
    \frac{L_\nabla B_{2p}^2}{2}
    \sum_{k=n_0}^{n-1}\gamma_k^2 \lVert \bar{\mat W}_{k+1, n}\rVert
    .
\end{equation*}
By \Cref{prop:clt_conditions}, we have
\[
    \sum_{k=n_0}^{n-1}\gamma_k^2
    \|\bar{\mat W}_{k+1,n}\|
    \leq
    \frac{\bar C_W}{\sqrt n}
    \sum_{k=n_0}^{n-1}\gamma_k
    \leq
    \frac{\bar C_W}{1-a}\sqrt n\,\gamma_n
    =\mathcal O(\sqrt n\,\gamma_n),
\]
which implies for some $\bar c_{\mathrm{nl}} > 0$ the bound
\begin{equation}\label{eq:pr_nl}
    \left\lVert
    \sum_{k=n_0}^{n-1}\sqrt{\gamma_{k+1}}\bar{\mat W}_{k+1,n}\vc R_k^{\mathrm{nl}}
    \right\rVert_{L_p}
    \leq  \bar c_{\mathrm{nl}} L_\nabla  B_{2p}^2 \sqrt{n} \gamma_n ,
\end{equation}

When $\bar F$ is linear, its Jacobian is constant, so we may take $L_\nabla=0$ in \Cref{ass:regularity}.
Then $\vc R_k^{\mathrm{nl}}=0$, and we set $c_{\mathrm{nl}} = 0$ and $\bar c_{\mathrm{nl}}=0$.
Because the bound no longer involves $B_{2p}$ for linear stochastic approximation, we do not require that $2p \leq q_y$ when $\bar{F}$ is linear.
 \subsubsection{Step-Size Normalization.}\label{sec:stepsize_normalization}
Recall $\vc R_k^{\mathrm{step}}$ from \eqref{eq:Rk}.
We show that
\begin{equation*}
    \left\lVert\sum_{k=n_0}^{n-1}\frac{\mat W_{k,n}}{\sqrt{\gamma_k}}\vc R_k^{\mathrm{step}}\right\rVert_{L_p} = \mathcal{O}\left( \frac{\rosbound_p}{n} \right).
\end{equation*}
Suppose \Cref{ass:step_size,ass:exogenous} hold and $1\leq p\leq q_m$.
The summands $\mat W_{k,n}\vc R_k^{\mathrm{step}}/\sqrt{\gamma_k}$ form a martingale-difference sequence.
Applying \Cref{lem:BR_weighted} with $\mat D_{k,n}=\mat W_{k,n}(\sqrt{\gamma_k/\gamma_{k+1}}-1)$ gives
\begin{equation}\label{eq:ssn}
    \left\lVert\sum_{k=n_0}^{n-1}\frac{\mat W_{k,n}}{\sqrt{\gamma_k}}\vc R_k^{\mathrm{step}}\right\rVert_{L_p}
    \leq \rosbound_p(\vc m)
    \left[\sum_{k=n_0}^{n-1}\lVert\mat W_{k,n}\rVert^2
    \left(\sqrt{\frac{\gamma_k}{\gamma_{k+1}}}-1\right)^2\right]^{1/2}.
\end{equation}
For $\gamma_k=\gamma_1k^{-a}$ with $a\in(0,1]$, Bernoulli's inequality,
$(1+x)^s\leq1+sx$ for $x\geq0$ and $s\in[0,1]$, gives, for every $k\geq j\geq1$ and $r\in\{1/2,1\}$,
\begin{equation}\label{eq:step_comparison}
    1
    \leq \left(\frac{\gamma_j}{\gamma_k}\right)^r
    = \left(1+\frac{k-j}{j}\right)^{ar}
    \leq 1+\frac{ar(k-j)}{j}.
\end{equation}
For the last iterate, \Cref{prop:dtct} and \eqref{eq:step_comparison} give, when $\lambda\gamma_1>2$,
\begin{equation*}
    \sum_{k=n_0}^{n-1}\|\mat W_{k,n}\|^2
    \left(\sqrt{\frac{\gamma_k}{\gamma_{k+1}}}-1\right)^2
    \leq C\sum_{k=n_0}^{n-1}\frac1{k^3}
    \left(\frac{(k+1)^2}{n^2}+\frac{C}{n}\right)^2
    =\mathcal O\left(\frac{1}{n^2}\right).
\end{equation*}
Substituting for \eqref{eq:ssn}, we obtain for some $c_{\mathrm{sn}}$ the bound
\begin{equation*}
    \left\lVert\sum_{k=n_0}^{n-1}\frac{\mat W_{k,n}}{\sqrt{\gamma_k}}\vc R_k^{\mathrm{step}}\right\rVert_{L_p}
    \leq c_{\mathrm{sn}}\frac{\rosbound_p(\vc m)}{n}.
\end{equation*}

 \subsubsection{Martingale Approximation.}
Recall $\vc R_k^{\mathrm P}$ from \eqref{eq:Rk}.
We obtain a bound on
\begin{equation*}
    \left\lVert\sum_{k=n_0}^{n-1}\frac{\mat W_{k,n}}{\sqrt{\gamma_k}}\vc R_k^{\mathrm P}\right\rVert_{L_p}
    \quad \text{and} \quad
    \left\lVert\sum_{k=n_0}^{n-1}\sqrt{\gamma_{k+1}}\bar{\mat W}_{k+1,n}\vc R_k^{\mathrm P}\right\rVert_{L_p}.
\end{equation*}
We apply \Cref{lem:martingale_approximation_full} with appropriately defined matrices $\{\mat D_{k,n}\}$.
For the last iterate with $\lambda\gamma_1>2$, we have for $\mat D_{k,n}=\sqrt{\gamma_k/\gamma_{k+1}}\mat W_{k,n}$, $k_1 = n_0$, and $k_2 = n$:
\begin{equation*}
    \sum_{k=n_0}^{n-1}\frac{\mat W_{k,n}}{\sqrt{\gamma_k}}\vc R_k^{\mathrm P}
    =-\sum_{k=n_0}^{n-1}\mat D_{k,n}
    \bigl(\widehat F(\vc x_k,\xi_k)+\vc\zeta_k+\vc m_k\bigr).
\end{equation*}
For the PR average, we have for $\mat D_{k,n} = \gamma_k \bar{\mat W}_{k+1,n}$, $k_1=n_0$, and $k_2=n$:
\begin{equation*}
    \sum_{k=n_0}^{n-1}\sqrt{\gamma_{k+1}}\bar{\mat W}_{k+1,n}\vc R_k^{\mathrm P}
    =-\sum_{k=n_0}^{n-1}\mat D_{k,n}\bigl(\widehat F(\vc x_k,\xi_k)+\vc\zeta_k+\vc m_k\bigr).
\end{equation*}
Applying \Cref{lem:martingale_approximation_full}, we obtain
\begin{equation}\label{eq:martingale_approximation_application}
    \begin{aligned}
        &\quad \left\lVert\sum_{k=k_1}^{k_2-1}\mat D_{k,n}\bigl(\widehat F(\vc x_k,\xi_k)+\vc\zeta_k+\vc m_k\bigr)\right\rVert_{L_p}
        \leq
        \martingale_p\left(\sum_{k=k_1}^{k_2-1}\gamma_{k-1}\lVert\mat D_{k,n}\rVert^2\right)^{1/2}
        \\ &+L_\Phi\update_p(n_0)
        \sum_{k=k_1}^{k_2-1}\gamma_{k-1}\lVert\mat D_{k,n}\rVert
        +\telescope_p(n_0)\max_{n_0<j\leq k_2}
        \left(\lVert\mat D_{n_0,n}\rVert+\lVert\mat D_{j-1,n}\rVert
        +\sum_{k=k_1}^{j-2}\lVert\mat D_{k+1,n}-\mat D_{k,n}\rVert\right).
    \end{aligned}
\end{equation}
The remainder of this section is dedicated to evaluating the bound for the different choices of $\mat D_{k,n}$.

\textbf{Last iterate.}
By \Cref{prop:dtct}, we have
\begin{equation*}
    \|\mat D_{k,n}\|
    = \mathcal O\left( \frac{1}{\sqrt{k}}
    \left(\frac{(k+1)^2}{n^2}+\frac{C}{n}\right)\right)
    .
\end{equation*}
Therefore, the first two terms in the upper bound \eqref{eq:martingale_approximation_application} are bounded as
\begin{equation}\label{eq:last_martingale_1}
    \left(\sum_{k=n_0}^{n-1}\lVert\mat D_{k,n}\rVert^2\gamma_{k-1}\right)^{1/2}
    \leq c_{\mathrm P}\sqrt{\gamma_n}
    \quad \text{and} \quad
    L_\Phi \sum_{k=n_0}^{n-1}\lVert\mat D_{k,n}\rVert\gamma_{k-1}
    \leq
    c_{\mathrm{inc}}\sqrt{\gamma_n}
\end{equation}
for some $c_{\mathrm P}$ and $c_{\mathrm{inc}}$.
The last term in \eqref{eq:martingale_approximation_application} is evaluated next.
For $n_0\leq k\leq n-2$, $\{\mat D_{k,n}\}$ satisfies
\begin{equation*}
    \begin{aligned}
        \mat D_{k+1,n}-\mat D_{k,n}
        &=\left[\sqrt{\frac{\gamma_{k+1}}{\gamma_{k+2}}}\identity
        -\frac{\gamma_k}{\gamma_{k+1}}(\identity-\gamma_{k+1}\drift)\right]\mat W_{k+1,n},\\
        \Rightarrow
        \lVert\mat D_{k+1,n}-\mat D_{k,n}\rVert
        &\leq\frac{1+\gamma_1\lVert\drift\rVert}{k}\lVert\mat W_{k+1,n}\rVert,
    \end{aligned}
\end{equation*}
where we used $0\leq (k+1)/k-\sqrt{(k+2)/(k+1)}\leq 1/k$.
By \Cref{prop:dtct}, we obtain
\[
    \sum_{k=n_0}^{n-2}\|\mat D_{k+1,n}-\mat D_{k,n}\|
    \leq C\sum_{k=n_0}^{n-2}\frac1{k^{3/2}}
    \left(\frac{(k+2)^2}{n^2}+\frac{C}{n}\right)
    =\mathcal O\left(\frac{1}{\sqrt n}\right).
\]
Combining with the bound on $\|\mat D_{k,n}\|$, we obtain
\[
    \max_{n_0<j\leq n}
    \left(\|\mat D_{n_0,n}\|+\|\mat D_{j-1,n}\|
    +\sum_{k=n_0}^{j-2}\|\mat D_{k+1,n}-\mat D_{k,n}\|\right)
    \leq c_{\mathrm{bd}}\sqrt{\gamma_n}
\]
for some $c_{\mathrm{bd}}>0$.
Hence,
\begin{equation*}
    \left\lVert\sum_{k=n_0}^{n-1}\frac{\mat W_{k,n}}{\sqrt{\gamma_k}}\vc R_k^{\mathrm P}\right\rVert_{L_p}
    \leq \sqrt{\gamma_n}\left(c_{\mathrm P}\martingale_p
    +c_{\mathrm{inc}}\update_p(n_0)+c_{\mathrm{bd}}\telescope_p(n_0)\right)
    .
\end{equation*}

\textbf{PR Average.}
Recall $\mat D_{k,n} = \gamma_k \bar{\mat W}_{k+1,n}$.
The first term in \eqref{eq:martingale_approximation_application} is evaluated using \Cref{prop:clt_conditions} to obtain for some $\bar c_{\mathrm{P}} > 0$:
\begin{align*}
    \sum_{k=n_0}^{n-1}\gamma_{k-1}\|\mathbf D_{k,n}\|^2
    \le
    \max_{k \in [n_0, n)}\|\mathbf D_{k,n}\|
    \sum_{k=n_0}^{n-1}\gamma_{k-1}\|\mathbf D_{k,n}\|
    \leq
    \frac{\bar C_W}{\sqrt n}\,
    2^a\sum_{k=n_0}^{n-1}\gamma_k^2
    \|\bar{\mathbf W}_{k+1,n}\|
\end{align*}
which implies for some $\bar c_{\mathrm P} > 0$ the bound
\begin{equation}\label{eq:pr_martingale_1}
    \left(\sum_{k=n_0}^{n-1}\|\mat D_{k,n}\|^2\gamma_{k-1}\right)^{1/2}
    \leq \bar c_{\mathrm P} \sqrt{\gamma_n}
    .
\end{equation}
The second term in \eqref{eq:martingale_approximation_application} is bounded using \eqref{eq:pr_nl}, where for some $\bar c_{\mathrm{inc}}$,
\begin{equation}\label{eq:pr_martingale_2}
    L_\Phi\sum_{k=n_0}^{n-1}\gamma_{k-1}\|\mat D_{k,n}\|
    \leq 2^a L_\Phi\sum_{k=n_0}^{n-1}\gamma_k^2\|\bar{\mat W}_{k+1,n}\|
    \leq \bar c_{\mathrm{inc}}\sqrt n\,\gamma_n
    .
\end{equation}
By \Cref{prop:clt_conditions}, we have $\max_{n_0<j\leq n}
\left(
    \|\mat D_{n_0,n}\|+\|\mat D_{j-1,n}\|
\right)
\leq \frac{2\bar C_W}{\sqrt n} $ and so
\begin{equation*}
    \max_{n_0<j\leq n}
    \left(\lVert\mat D_{n_0,n}\rVert+\lVert\mat D_{j-1,n}\rVert
    +\sum_{k=k_1}^{j-2}\lVert\mat D_{k+1,n}-\mat D_{k,n}\rVert\right)
    \leq
    \mathcal{O}\left(\frac{1}{\sqrt{n}} \right) + \max_{n_0 < j \leq n}\sum_{k=k_1}^{j-2}\lVert\mat D_{k+1,n}-\mat D_{k,n}\rVert
    .
\end{equation*}
It remains to obtain a bound on
\begin{equation*}
    \max_j \sum_k \left\| \mat D_{k+1, n} - \mat D_{k,n}\right\|
    =
        \max_{n_0<j\leq n}
    \sum_{k=n_0}^{j-2}
    \left\|
        \gamma_{k+1}\bar{\mat W}_{k+2,n}
        -\gamma_k\bar{\mat W}_{k+1,n}
    \right\|.
\end{equation*}
For $\bar n_0\leq k\leq n-2$, the recursion \eqref{eq:barW_recursion} gives
$    \bar{\mat W}_{k+1,n}
    =
    (\identity-\gamma_{k+1}\drift)\bar{\mat W}_{k+2,n}
    +\frac{\identity}{\sqrt{n-\bar n_0}}
    \mathbf{1}\{k\geq\bar n_0\}$, and so
\begin{align*}
    \gamma_{k+1}\bar{\mat W}_{k+2,n}
    -\gamma_k\bar{\mat W}_{k+1,n}
    =
    (\gamma_{k+1}-\gamma_k)\bar{\mat W}_{k+2,n}
    +
    \gamma_k\drift
    \left(
        \gamma_{k+1}\bar{\mat W}_{k+2,n}
        -\frac{\drift^{-1}}{\sqrt{n-\bar n_0}}
        \mathbf{1}\{k\geq\bar n_0\}
    \right)
\end{align*}
for all $n_0 \leq k \leq n - 2$.
We first establish a bound for $k \geq \bar n_0$.
Using $|\gamma_{k+1}-\gamma_k| \leq a\gamma_{k+1} / k$, we have by \Cref{prop:clt_conditions,prop:dtct} that
\begin{equation}\label{eq:pr_martingale_3_1}
    \left\|
        \gamma_{k+1}\bar{\mat W}_{k+2,n}
        -\gamma_k\bar{\mat W}_{k+1,n}
    \right\|
    \leq
    \mathcal O\left(\frac{1}{n^{3/2}}
    + \frac{\gamma_k }{\sqrt n}
    \left[
        \frac{1}{n\gamma_n}
        +e^{-\lambda(n-k-1)\gamma_n/2}
    \right]
    \right)
    , \quad \forall k \in [\bar n_0, n - 1)
    .
\end{equation}

Next we establish a similar bound for $k \in [n_0, \bar n_0)$.
Using the recursion \eqref{eq:barW_recursion} for $k < \bar n_0$, we obtain
\[
    \gamma_{k+1}\bar{\mat W}_{k+2,n}
    -\gamma_k\bar{\mat W}_{k+1,n}
    =
    \left(
        (\gamma_{k+1}-\gamma_k)\identity
        +\gamma_k\gamma_{k+1}\drift
    \right)\bar{\mat W}_{k+2,n}.
\]
For $\{\bar{\mat W}_{k,n}\}$, we use \eqref{eq:barW_bound} and the backwards recursion \eqref{eq:barW_recursion} to obtain that for all sufficiently \emph{small} $k$ satisfying $k \in [n_0, \bar n_0)$, there exists some $\bar C_W$ such that
\begin{equation}\label{eq:pr_early_weights}
    \lVert \bar{\mat W}_{k+1, n}\rVert
    \leq\frac{\bar C_W}{\sqrt n\,\gamma_n}
    \exp\left(-\frac\lambda2\sum_{i=k+1}^{\bar n_0-1}\gamma_i\right),\qquad n_0\leq k<\bar n_0
    .
\end{equation}
Using the identity $\gamma_{k+1}\bar{\mat W}_{k+2,n}
    -\gamma_k\bar{\mat W}_{k+1,n}
    =
    \left(
        (\gamma_{k+1}-\gamma_k)\identity
        +\gamma_k\gamma_{k+1}\drift
    \right)\bar{\mat W}_{k+2,n}$, we obtain
\begin{align*}
    \left\|
        \gamma_{k+1}\bar{\mat W}_{k+2,n}
        -\gamma_k\bar{\mat W}_{k+1,n}
    \right\|
    &\leq
    \gamma_{k+1}
    \left(\frac{a}{k}+\gamma_k\|\drift\|\right)
    \|\bar{\mat W}_{k+2,n}\|\\
    &\leq
    \frac{C\gamma_{k+1}}{\sqrt n\,\gamma_n}
    \left(\frac1k+\gamma_k\right)
    \exp\left(
        -\frac{\lambda}{2}
        \sum_{i=k+2}^{\bar n_0-1}\gamma_i
    \right).
\end{align*}
The sum over $[n_0, \bar n_0)$ is then bounded in two parts:
\begin{align*}
    \sum_{k=n_0}^{\lfloor n/ 4\rfloor}
    \left\|
        \gamma_{k+1}\bar{\mat W}_{k+2,n}
        -\gamma_k\bar{\mat W}_{k+1,n}
    \right\|
    &\leq
    C n^{a+1/2}e^{-\lambda n\gamma_n/16}
    \\
    \sum_{k=\lfloor n/ 4\rfloor + 1}^{\bar n_0 - 1}
    \left\|
        \gamma_{k+1}\bar{\mat W}_{k+2,n}
        -\gamma_k\bar{\mat W}_{k+1,n}
    \right\|
    &
    \leq
    \frac{C\gamma_n}{\sqrt n}
    \left(
        2+\sum_{\ell=1}^{\infty}
        e^{-\lambda\ell\gamma_n/2}
    \right)
    =\mathcal O\left(\frac{1}{\sqrt n} \right).
\end{align*}
Combining with \eqref{eq:pr_martingale_3_1}, we obtain for some $\bar c_{\mathrm{bd}}$ the bound
\begin{equation}\label{eq:pr_martingale_3}
    \max_{n_0 < j \leq n}
    \left(
        \|\mat D_{n_0,n}\|+\|\mat D_{j-1,n}\|
        +\sum_{k=n_0}^{j-2} \| \gamma_{k+1} \bar{\mat W}_{k+2, n} - \gamma_k \bar{\mat W}_{k+1, n} \|
    \right)
    \leq \frac{\bar c_{\mathrm{bd}}}{\sqrt n} .
\end{equation}
Combining \eqref{eq:pr_martingale_1}, \eqref{eq:pr_martingale_2}, and \eqref{eq:pr_martingale_3}, we obtain
\begin{equation*}
\begin{aligned}
    \left\lVert\sum_{k=n_0}^{n-1}\sqrt{\gamma_{k+1}}\bar{\mat W}_{k+1,n}\vc R_k^{\mathrm P}\right\rVert_{L_p}
    \leq \bar c_{\mathrm P} \martingale_p\sqrt{\gamma_n}+
    \bar c_{\mathrm{inc}}\update_p(n_0)\sqrt n\,\gamma_n
    +\bar c_{\mathrm{bd}}\frac{\telescope_p(n_0)}{\sqrt n}
\end{aligned}
\end{equation*}

\subsection{Proof of \texorpdfstring{\protect\Cref{lem:covariance_convergence}}{Lemma \ref*{lem:covariance_convergence}}}\label{proof:covariance_convergence}
For simplicity of exposition, we prove the result assuming $\vc \zeta = 0$.
Define the covariance matrix
\[
    \mat\Sigma_n^\pi\coloneqq\sum_{k=n_0}^{n-1}\mat W_{k,n}\mat\Gamma\mat W_{k,n}\tp ,
\]
which corresponds to the covariance of a stationary copy of $\vc m$, generated by $\{\xi\}$ independent of $\vc m$.
By \Cref{ass:noise_markov}, we have $\mathbb{E}[V(\xi_0)] < \infty$.
Using $\varphi$ in \eqref{eq:varphi_covariance}, we have by \Cref{ass:noise_markov} the inequality \eqref{eq:V_ergodic}.
Using that $\mathbb E[\vc m_k\vc m_k^\top]=\mathbb E[\varphi(\xi_{k-1})]
    =\mathbb E[(P^{k-1}\varphi)(\xi_0)]$, we obtain
\[
    \|\mathbb E[\vc m_k\vc m_k^\top]-\mat\Gamma\|
    \le \mathbb E\!\left[
    \|(P^{k-1}\varphi)(\xi_0)-\pi(\varphi)\|\right]
    \underset{\eqref{eq:V_ergodic}}{\le}
    \left(C_V \mathbb E[V(\xi_0)]
    \sup_{\xi\in\mathcal S}\frac{\|\varphi(\xi)\|}{V(\xi)}\right)
    \varrho^{k-1}
\]
Therefore, for some constant $C > 0$,
\begin{align*}
    \|\mat\Sigma_n-\mat\Sigma_n^\pi\|_2
    &\le
    \sum_{k=n_0}^{n-1}\|\mat W_{k,n}\|_2^2
    \|\mathbb E[\vc m_k\vc m_k^\top]-\mat\Gamma\|_2
    \le \frac{C}{n-n_0}\sum_{k=n_0}^{n-1}\varrho^{k-1}
    =\mathcal O \left(\frac{1}{n-n_0}\right)
    .
\end{align*}
We next prove that $\| \mat \Sigma_n^\pi - \mat \Sigma \| = \mathcal{O}(1/(n-n_0))$.
Recall that $\mat \Sigma$ is the solution to the Lyapunov equation \(\drift\mat\Sigma+\mat\Sigma\drift^\top=\mat\Gamma\).
Using the recursion
\[
\mat\Sigma_{n+1}^\pi-\mat\Sigma
=(\identity-\gamma_n\drift)
 (\mat\Sigma_n^\pi-\mat\Sigma)
 (\identity-\gamma_n\drift)^\top
+\gamma_n^2\drift\mat\Sigma\drift^\top,
\quad \mat\Sigma_{n_0}^\pi=0
,
\]
we obtain for $a = 1$ and $\lambda \gamma_1 > 2$:
\begin{equation}\label{eq:Sigma_pi}
    \|\mat\Sigma_n^\pi-\mat\Sigma\|_2
    \le C\left(\frac{n_0}{n}\right)^{\lambda\gamma_1}
    +C\sum_{k=n_0}^{n-1}\frac{1}{k^2}
        \left(\frac{k+1}{n}\right)^{\lambda\gamma_1}
    =\mathcal O\left(\frac{1}{n - n_0}\right)
    .
\end{equation}
Combining with $\| \mat \Sigma_n - \mat \Sigma_n^\pi\|_2 = \mathcal{O}(1/(n-n_0))$, we obtain
\(
    \|\mat\Sigma_n-\mat\Sigma\|_2
    =\mathcal O\left(\frac{1}{n-n_0}\right).
\)

For the PR average, the triangle inequality gives
\begin{equation*}
\begin{aligned}
    \|\bar{\mat\Sigma}_n-\bar{\mat\Sigma}\|_2
    &\leq
    \sum_{k=n_0}^{n-1}\gamma_k^2\|\bar{\mat W}_{k+1,n}\|_2^2
    \|\mathbb E[\vc m_k\vc m_k^\top]-\mat\Gamma\|_2
    +
    \left\|
    \sum_{k=n_0}^{n-1}\gamma_k^2\bar{\mat W}_{k+1,n}\mat\Gamma\bar{\mat W}_{k+1,n}^\top
    -\bar{\mat\Sigma}
    \right\|_2
\end{aligned}
\end{equation*}
The first term is bounded by $\mathcal O(1/n)$ using $\| \mathbb{E} \vc m_k \vc m_k^\top - \mat \Gamma \| = \mathcal{O}(\varrho^{k-1})$ and $\gamma_k \| \bar{\mat W}_{k+1, n} \| = \mathcal{O}(1/\sqrt{n})$.
The second term is bounded using the identity $\bar{\mat \Sigma} = \drift^{-1} \mat \Gamma \drift^{-\top}$ to obtain
\begin{align*}
&\quad \sum_{k=n_0}^{n-1}\gamma_k^2\bar{\mat W}_{k+1,n}\mat\Gamma\bar{\mat W}_{k+1,n}^\top-\bar{\mat\Sigma}
\\ &=
\sum_{k=n_0}^{\bar n_0-1}\gamma_k^2\bar{\mat W}_{k+1,n}\mat\Gamma\bar{\mat W}_{k+1,n}^\top+\sum_{k=\bar n_0}^{n-1}\left(\gamma_k^2\bar{\mat W}_{k+1,n}\mat\Gamma\bar{\mat W}_{k+1,n}^\top-\frac{\drift^{-1}\mat\Gamma\drift^{-\top}}{n-\bar n_0}\right)
.
\end{align*}
By the triangle inequality, we have
\begin{align*}
    &\quad \left\lVert \sum_{k=n_0}^{n-1}\gamma_k^2\bar{\mat W}_{k+1,n}\mat\Gamma\bar{\mat W}_{k+1,n}^\top-\bar{\mat\Sigma}
    \right\rVert_2
    \leq
    \|\mat\Gamma\|_2\sum_{k=n_0}^{\bar n_0-1}\gamma_k^2\|\bar{\mat W}_{k+1,n}\|_2^2
    \\ &+\sum_{k=\bar n_0}^{n-1}\left\|\left(\gamma_k\bar{\mat W}_{k+1,n}-\frac{\drift^{-1}}{\sqrt{n-\bar n_0}}\right)\mat\Gamma\gamma_k\bar{\mat W}_{k+1,n}^\top+\frac{\drift^{-1}}{\sqrt{n-\bar n_0}}\mat\Gamma\left(\gamma_k\bar{\mat W}_{k+1,n}-\frac{\drift^{-1}}{\sqrt{n-\bar n_0}}\right)^\top\right\|_2
    .
\end{align*}
Using $\|\bar{\mat W}_{k+1,n}\|=\mathcal O\!\left(\frac{1}{\sqrt n\,\gamma_n}\right)$ from \eqref{eq:barW_bound}
and
\(
    \sum_{k=n_0}^{\bar n_0-1}
    \gamma_k^2\|\bar{\mat W}_{k+1,n}\|
    =\mathcal O(1/\sqrt{n})
\) from \eqref{eq:pr_early_weights}, we obtain for some $C > 0$ the rate
\begin{equation}\label{eq:pr_covar_1}
    \|\mat\Gamma\|_2\sum_{k=n_0}^{\bar n_0-1}\gamma_k^2\|\bar{\mat W}_{k+1,n}\|_2^2
    \leq\frac{C}{\sqrt n\,\gamma_n}\sum_{k=n_0}^{\bar n_0-1}\gamma_k^2\|\bar{\mat W}_{k+1,n}\|
    = \mathcal O\left( \frac{1}{n \gamma_n}\right) .
\end{equation}
Next, we use
\begin{align*}
    &\quad \sum_{k=\bar n_0}^{n-1}\left\|\left(\gamma_k\bar{\mat W}_{k+1,n}-\frac{\drift^{-1}}{\sqrt{n-\bar n_0}}\right)\mat\Gamma\gamma_k\bar{\mat W}_{k+1,n}^\top+\frac{\drift^{-1}}{\sqrt{n-\bar n_0}}\mat\Gamma\left(\gamma_k\bar{\mat W}_{k+1,n}-\frac{\drift^{-1}}{\sqrt{n-\bar n_0}}\right)^\top\right\|_2
    \\
    &\leq\|\mat\Gamma\|_2\sum_{k=\bar n_0}^{n-1}\left(\gamma_k\|\bar{\mat W}_{k+1,n}\|_2+\frac{\|\drift^{-1}\|_2}{\sqrt{n-\bar n_0}}\right)\left\|\gamma_k\bar{\mat W}_{k+1,n}-\frac{\drift^{-1}}{\sqrt{n-\bar n_0}}\right\|_2
\end{align*}
Using $\|\gamma_k \bar{\mat W}_{k+1, n} \| = \mathcal{O}(1/\sqrt n)$ from \Cref{prop:clt_conditions}, we obtain from \Cref{prop:dtct} the bound
\begin{align*}
    \sum_{k=\bar n_0}^{n-1}\left(\gamma_k\|\bar{\mat W}_{k+1,n}\|_2+\frac{\|\drift^{-1}\|_2}{\sqrt{n-\bar n_0}}\right)\left\|\gamma_k\bar{\mat W}_{k+1,n}-\frac{\drift^{-1}}{\sqrt{n-\bar n_0}}\right\|
    &\leq
    \frac{C}{n}
    \sum_{k=\bar n_0}^{n-1}
    \left(
        \frac1{n\gamma_n}
        +e^{-\lambda(n-k)\gamma_n/2}
    \right)
    \\ &=
    \mathcal O\left(\frac1{n\gamma_n}\right).
\end{align*}

We next prove the Wasserstein bounds in \Cref{lem:covariance_convergence}.
The convergence of the covariance matrices implies $\mat \Sigma_n, \mat \Sigma_n^{-1}, \bar{\mat \Sigma}_n$, and $\bar{\mat \Sigma}_n^{-1}$ are all bounded for all sufficiently large $n$ when $\mat \Sigma$ and $\mat \Gamma$ are positive definite.
For positive definite $\mat \Sigma, \mat\Gamma$, the Lipschitz continuity of the matrix square root implies
\begin{equation*}
    \| \mat \Sigma_n^{1/2} - \mat \Sigma^{1/2} \|_2 = \mathcal O(n^{-1})
    \quad \text{and} \quad
    \| \bar{\mat \Sigma}_n^{1/2} - \bar{\mat \Sigma}^{1/2} \|_2 = \mathcal O\left(\frac{1}{n \gamma_n }\right) .
\end{equation*}
Using that for any positive definite covariances $\mat C_1, \mat C_2$, $\Wass_p\left(
\normal(\vc 0,\mat C_1),
\normal(\vc 0,\mat C_2)
\right)
\leq
\mathbb E^{1/p}\left\|\bigl(\mat C_1^{1/2}-\mat C_2^{1/2}\bigr)\vc z\right\|^p$, we obtain for any $p \geq 1$,
\begin{equation*}
\begin{aligned}
    \Wass_p\left(
    \normal(\vc 0,\mat\Sigma_n),
    \normal(\vc 0,\mat\Sigma)
    \right)
    &\leq
    C\lVert\mat Q\rVert_2^{1/2}
    \frac{\lVert\vc z\rVert_{L_p,2}}{n-n_0}
    \\
    \Wass_p\left(\normal(\vc0,\bar{\mat\Sigma}_n),\normal(\vc0,\bar{\mat\Sigma})\right)
    &\leq C \| \mat Q \|_2^{1/2} \frac{\| \vc z\|_{L_p, 2}}{n \gamma_n }
    .
\end{aligned}
\end{equation*}
 
\subsection{Proof of \texorpdfstring{\protect\Cref{lem:martingale_approximation_full}}{Lemma \ref*{lem:martingale_approximation_full}}}\label{proof:martingale_approximation}
The centered noise $\widehat F(\vc x_k, \xi_k) + \vc \zeta_k$ is decomposed using \eqref{eq:poisson} and \eqref{eq:Mk} to obtain for all $k \geq 2$:
\begin{equation}\label{eq:martingale_approximation}
    \widehat F(\vc x_k, \xi_k) + \vc\zeta_k
    =-\vc m_k + \poismds_k + \poisinc_k
    +(P\Phi)(\vc x_{k - 1}, \xi_{k - 1})
    -(P\Phi)(\vc x_k, \xi_k) ,
\end{equation}
where
\begin{align}
    \poismds_k
    &\coloneqq
    \Phi(\vc x_{k - 1}, \xi_k) - \Phi(\vc x^*, \xi_k) -\bigl[(P\Phi)(\vc x_{k - 1}, \xi_{k - 1})
    -(P\Phi)(\vc x^*, \xi_{k - 1})\bigr]
    , \label{eq:poisson_md_def}
    \\
    \poisinc_k
    &\coloneqq \Phi(\vc x_k, \xi_k) - \Phi(\vc x_{k - 1}, \xi_k).
    \label{eq:poisson_update_def}
\end{align}

Let $j$ be any integer satisfying $k_1 \leq j \leq k_2$.
By construction, $\poismds$ is a martingale difference sequence with respect to $\mathcal H$, while $\poisinc$ records the change in $\Phi(\cdot,\xi_k)$ along the SA sequence.
By \eqref{eq:martingale_approximation},
\begin{equation*}
    \begin{aligned}
        &\quad \sum_{k=k_1}^{j-1}\mat D_{k,n}
        \bigl(\widehat F(\vc x_k,\xi_k)+\vc\zeta_k+\vc m_k\bigr)
        \\ &
        =\sum_{k=k_1}^{j-1}\mat D_{k,n}
        \left(\poismds_k+\poisinc_k
        +(P\Phi)(\vc x_{k-1},\xi_{k-1})-(P\Phi)(\vc x_k,\xi_k)\right).
    \end{aligned}
\end{equation*}
We obtain bounds on the three components involving $\poismds$, $\poisinc$, and the last two separately.

The Lipschitz continuity of $\Phi$ and $P\Phi$ in \Cref{lem:poisson_regularity} gives
\begin{equation*}
    \lVert\poismds_k\rVert
    \leq 2L_\Phi\sqrt{\gamma_{k - 1}}
    \lVert\vc y_{k - 1}\rVert
    \left(V(\xi_k) \vee (PV)(\xi_{k - 1})\right).
\end{equation*}
H\"older's inequality and Jensen's inequality for $PV$ therefore give
\begin{equation*}
    \seqnorm{\condsd\left(\frac{\poismds_k}{\sqrt{\gamma_{k - 1}}} \mid \mathcal H_{k-1}\right)}_{L_{p \vee 2}}
    \leq \seqnorm{\frac{\poismds_k}{\sqrt{\gamma_{k - 1}}}}_{L_{p \vee 2}}
    \leq 4L_\Phi B_{q_y}\seqnorm{V(\xi)}_{L_{q_V}},
\end{equation*}
where the norm $\seqnorm{\cdot}$ involving $\poismds_k / \sqrt{\gamma_{k-1}}$ above are taken over $k_1 \leq k < k_2$.
Applying \Cref{lem:BR_weighted} with matrix weights $\mat D_{k,n}\sqrt{\gamma_{k-1}}$, we obtain
\begin{equation}\label{eq:ma_1}
    \left\lVert\sum_{k = k_1}^{j - 1}\mat D_{k,n}\poismds_k\right\rVert_{L_p}
    \leq \martingale_p
    \left(\sum_{k = k_1}^{k_2 - 1}\lVert\mat D_{k,n}\rVert^2\gamma_{k - 1}\right)^{1/2}.
\end{equation}

Recall $\poisinc_k = \Phi(\vc x_k,\xi_k)-\Phi(\vc x_{k-1},\xi_k)$.
By Lipschitz continuity (\Cref{lem:poisson_regularity}), $\| \poisinc_k \| \leq L_\Phi \| \vc x_k - \vc x_{k-1} \| V(\xi_k)$.
Substituting \eqref{eq:sa_boilerplate} for $\vc x_k - \vc x_{k - 1}$, we obtain
\begin{equation*}
    \lVert\poisinc_k\rVert
    \leq L_\Phi\gamma_{k - 1}
    \left\lVert F(\vc x_{k - 1}, \xi_{k - 1}) + \vc\zeta_{k - 1}\right\rVert
    V(\xi_k).
\end{equation*}
Using the identity $F(\vc x_{k - 1},\xi_{k - 1}) = F(\vc x_{k - 1},\xi_{k - 1})-F(\vc x^*,\xi_{k - 1})+F(\vc x^*,\xi_{k - 1})$ and the definition $\vc x_{k - 1}-\vc x^*=\sqrt{\gamma_{k - 1}}\vc y_{k - 1}$, we obtain using \Cref{ass:regularity}, H\"older's inequality, and $\gamma_{k - 1}\leq 2\gamma_{k_1}$ the inequality
\begin{equation*}
    \left\lVert
    \bigl(F(\vc x_{k - 1},\xi_{k - 1})-F(\vc x^*,\xi_{k - 1})\bigr)V(\xi_k)
    \right\rVert_{L_p}
    \leq \sqrt{2\gamma_{k_1}}L_FB_{q_y}\seqnorm{V(\xi)}_{L_{q_V}}.
\end{equation*}
Moreover, \Cref{ass:poisson_moments}, H\"older's inequality, and the independence of $\vc\zeta_{k - 1}$ and $\xi_k$ give
\begin{equation*}
    \left\lVert F(\vc x^*,\xi_{k - 1})V(\xi_k)\right\rVert_{L_p}
    \leq K_F\seqnorm{V(\xi)}_{L_{q_V}}^2
    ,
    \qquad
    \left\lVert\vc\zeta_{k - 1}V(\xi_k)\right\rVert_{L_p}
    \leq \seqnorm{\vc\zeta}_{L_p}\seqnorm{V(\xi)}_{L_p}.
\end{equation*}
Therefore, $\sup_{k_1 \leq k < k_2}
    \left\lVert
    \left(F(\vc x_{k - 1}, \xi_{k - 1}) + \vc\zeta_{k - 1}\right)V(\xi_k)
    \right\rVert_{L_p}
    \leq \update_p(k_1)$, and Minkowski's inequality gives
\begin{equation}\label{eq:ma_2}
    \left\lVert\sum_{k=k_1}^{j-1}\mat D_{k,n}\poisinc_k\right\rVert_{L_p}
    \leq L_\Phi\update_p(k_1)
    \sum_{k = k_1}^{k_2 - 1}\lVert\mat D_{k,n}\rVert\gamma_{k - 1}.
\end{equation}

Lastly, we write $(P\Phi)(\vc x_k,\xi_k) = (P\Phi)(\vc x^*,\xi_k) +\bigl((P\Phi)(\vc x_k,\xi_k)-(P\Phi)(\vc x^*,\xi_k)\bigr)$.
The first $(P \Phi)$ term is bounded using Jensen's inequality
\[
    \left\|(P\Phi)(\vc x^*,\xi_k)\right\|_{L_p}
    \le
    \left\|\Phi(\vc x^*,\xi_{k+1})\right\|_{L_p}
    \le
    \seqnorm{\Phi(\vc x^*, \xi)}_{L_p}.
\]
By Lipschitz continuity, we obtain
\begin{equation*}
    \left\lVert(P\Phi)(\vc x_k,\xi_k)-(P\Phi)(\vc x^*,\xi_k)\right\rVert
    \leq L_\Phi\lVert\vc x_k-\vc x^*\rVert(PV)(\xi_k).
\end{equation*}
Using $\vc x_k-\vc x^*=\sqrt{\gamma_k}\vc y_k$, H\"older's inequality,
Jensen's inequality for $PV$, and $\gamma_{k_1-1}\leq 2\gamma_{k_1}$, we
obtain
\begin{equation*}
    \sup_{k_1-1\leq k<k_2}
    \left\lVert
    (P\Phi)(\vc x_k,\xi_k)-(P\Phi)(\vc x^*,\xi_k)
    \right\rVert_{L_p}
    \leq
    \sqrt{2\gamma_{k_1}}L_\Phi B_{q_y}
    \seqnorm{V(\xi)}_{L_{q_V}}.
\end{equation*}
By the triangle inequality, we obtain $\sup_{k_1-1\leq k<k_2}
    \left\lVert(P\Phi)(\vc x_k,\xi_k)\right\rVert_{L_p}
    \leq \telescope_p(k_1)$.
For $j>k_1$, summation by parts gives
\begin{equation*}
    \begin{aligned}
        &\quad\sum_{k=k_1}^{j-1}\mat D_{k,n}
        \left((P\Phi)(\vc x_{k-1},\xi_{k-1})-(P\Phi)(\vc x_k,\xi_k)\right)\\
        &=\mat D_{k_1,n}(P\Phi)(\vc x_{k_1-1},\xi_{k_1-1})
        -\mat D_{j-1,n}(P\Phi)(\vc x_{j-1},\xi_{j-1})
        +\sum_{k=k_1}^{j-2}(\mat D_{k+1,n}-\mat D_{k,n})(P\Phi)(\vc x_k,\xi_k).
    \end{aligned}
\end{equation*}
Minkowski's inequality therefore gives
\begin{equation}\label{eq:ma_3}
\begin{aligned}
    & \quad \left\lVert
    \sum_{k=k_1}^{j-1}\mat D_{k,n}
    \left((P\Phi)(\vc x_{k-1},\xi_{k-1})-(P\Phi)(\vc x_k,\xi_k)\right)
    \right\rVert_{L_p}
    \\ &\leq \telescope_p(k_1)
    \left(\lVert\mat D_{k_1,n}\rVert+\lVert\mat D_{j-1,n}\rVert
    +\sum_{k=k_1}^{j-2}\lVert\mat D_{k+1,n}-\mat D_{k,n}\rVert\right).
\end{aligned}
\end{equation}
Combining \eqref{eq:ma_1}--\eqref{eq:ma_3}, and taking the maximum over $k_1 \leq j \leq k_2$ yields the bound in \Cref{lem:martingale_approximation_full}.

If the observations are independent, then $(P\Phi)(\vc x,\cdot)=0$, and \eqref{eq:martingale_approximation} simplifies to
\begin{equation*}
    \widehat F(\vc x_k,\xi_k)+\vc\zeta_k+\vc m_k
    =\poismds_k+\poisinc_k
    =\Phi(\vc x_k,\xi_k)-\Phi(\vc x^*,\xi_k),
\end{equation*}
which is a martingale difference sequence.
By \Cref{lem:poisson_regularity}, $\vc x_k-\vc x^*=\sqrt{\gamma_k}\vc y_k$, and the moment bounds of \Cref{ass:normalized,ass:noise_markov} give
\begin{equation*}
    \seqnorm{\condsd\left(\frac{\poismds_k+\poisinc_k}{\sqrt{\gamma_k}}\mid\mathcal H_{k-1}\right)}_{L_{p\vee2}}
    \leq\seqnorm{\frac{\poismds_k+\poisinc_k}{\sqrt{\gamma_k}}}_{L_{p\vee2}}
    \leq L_\Phi B_{q_y}\seqnorm{V(\xi)}_{L_{q_V}},
\end{equation*}
where the first inequality follows from Jensen's inequality and the sequence norms $\seqnorm{\cdot}$ are over the interval $[k_1, k_2)$.
Using
\[
  \begin{aligned}
  \rosbound_p\left(\left\{
  \frac{\poismds_k+\poisinc_k}{\sqrt{\gamma_k}}
  \right\}_{k=k_1}^{k_2-1}\right)
  \le
  \bigl(\BRvar\sqrt{p\vee2}
  +\BRinc(p\vee2)\bigr)
  L_\Phi B_{q_y}\seqnorm{V(\xi)}_{L_{q_V}}
  =\frac{\martingale_p}{4}
  \end{aligned}
\]
and applying \Cref{lem:BR_weighted} with weights $\mat D_{k,n}\sqrt{\gamma_k}$ proves \eqref{eq:weighted_poisson_remainder_iid_full}.
 \section{Deferred Proofs}\label{app:deferred_proofs}

\subsection{Proof of \texorpdfstring{\protect\Cref{prop:dtct}}{Proposition \ref*{prop:dtct}}}\label{proof:dtct}
We first prove the bound for the sequence  $\{\mat W_{k,n}\}$.
By \eqref{eq:contraction} and Taylor's remainder theorem, we obtain
\begin{align*}
    \left\|\identity-\gamma_j\drift-e^{-\gamma_j\drift}\right\|
    =\left\|\int_0^{\gamma_j}(\gamma_j-t)\drift^2e^{-t\drift}\,dt\right\|
    \leq \frac{\|\drift\|^2}{2}\gamma_j^2.
\end{align*}
Using the identity
\begin{align*}
    \frac{\mat W_{k,n}}{\sqrt{\gamma_k}}-\exp\left(-\drift\sum_{i=k+1}^{n-1}\gamma_i\right)
    &=\prod_{i=k+1}^{n-1}(\identity-\gamma_i\drift)-\exp\left(-\drift\sum_{i=k+1}^{n-1}\gamma_i\right)\\
    &=\prod_{i=k+1}^{n-1}(\identity-\gamma_i\drift)
    +\sum_{j=k+1}^{n-2}\left[\prod_{i=k+1}^{j}(\identity-\gamma_i\drift)\right]\exp\left(-\drift\sum_{i=j+1}^{n-1}\gamma_i\right)\\
    &-\sum_{j=k+1}^{n-2}\left[\prod_{i=k+1}^{j}(\identity-\gamma_i\drift)\right]\exp\left(-\drift\sum_{i=j+1}^{n-1}\gamma_i\right)
    -\exp\left(-\drift\sum_{i=k+1}^{n-1}\gamma_i\right),
\end{align*}
The sum of the first two terms and the sum of the third and fourth terms are, respectively,
\[
    \sum_{j=k+1}^{n-1}\prod_{i=k+1}^{j}(\identity-\gamma_i\drift) \exp\left(-\drift\sum_{i=j+1}^{n-1}\gamma_i\right)
    \quad
    \text{and} \quad
    -\sum_{j=k+1}^{n-1}\prod_{i=k+1}^{j-1}(\identity-\gamma_i\drift)\exp\left(-\drift\sum_{i=j}^{n-1}\gamma_i\right).
\]
Combining, we obtain
\begin{align*}
    \frac{\mat W_{k,n}}{\sqrt{\gamma_k}}
    -\exp\left(-\drift\sum_{j=k+1}^{n-1}\gamma_j\right)
    =
    \sum_{j=k+1}^{n-1}
    \left[\prod_{i=k+1}^{j-1}(\identity-\gamma_i\drift)\right]
    \left(\identity-\gamma_j\drift-e^{-\gamma_j\drift}\right)
    \exp\left(-\drift\sum_{i=j+1}^{n-1}\gamma_i\right).
\end{align*}
Under \Cref{ass:stability}, we have \eqref{eq:contraction} and thus,
\begin{align*}
    \left\|
        \frac{\mat W_{k,n}}{\sqrt{\gamma_k}}
        -\exp\left(-\drift\sum_{j=k+1}^{n-1}\gamma_j\right)
    \right\|
    &\leq
    \frac{C \| \drift \|^2}{2} \sum_{j=k+1}^{n-1}  \gamma_j^2 \exp\left(-\frac{\lambda}{2}
    \sum_{\substack{i=k+1}}^{n-1}\gamma_i\right)
    \\ & =
    \mathcal O\left(\frac{1}{k}\left(\frac{k+1}{n}\right)^{\lambda\gamma_1/2}\right)
    .
\end{align*}
For $\lambda \gamma_1 > 2$, this bound is $\mathcal O(1/n)$.

We next prove the bound for the weights $\{\bar{\mat W}_{k,n}\}$.
Recall the recursion \eqref{eq:barW_recursion}, which is solved starting from $\bar{\mat W}_{n, n}$ to obtain
\[
    \bar{\mat W}_{k+2,n}
    =
    \frac{1}{\sqrt{n-\bar n_0}}
    \sum_{j=k+2}^{n}
    \prod_{i=k+2}^{j-1}(\identity-\gamma_i\drift),
    \quad
    k \in [\bar n_0 - 1, n -2].
\]
Using the identity
\begin{equation*}
    \identity
    =
    \sum_{j=k+2}^{n}\gamma_j\drift
    \prod_{i=k+2}^{j-1}(\identity-\gamma_i\drift)
    +
    \prod_{i=k+2}^{n}(\identity-\gamma_i\drift),
\end{equation*}
we obtain
\begin{align*}
    \gamma_{k+1}\bar{\mat W}_{k+2,n}
    -\frac{\drift^{-1}}{\sqrt{n-\bar n_0}}
    =
    \frac{1}{\sqrt{n-\bar n_0}}
    \left[
        \sum_{j=k+2}^{n}(\gamma_{k+1}-\gamma_j)
        \prod_{i=k+2}^{j-1}(\identity-\gamma_i\drift)
        -
        \drift^{-1}
        \prod_{i=k+2}^{n}(\identity-\gamma_i\drift)
    \right]
    .
\end{align*}
By the triangle inequality, we obtain
\begin{align*}
    &\quad \left\|
        \sum_{j=k+2}^{n}(\gamma_{k+1}-\gamma_j)
        \prod_{i=k+2}^{j-1}(\identity-\gamma_i\drift)
        -
        \drift^{-1}\prod_{i=k+2}^{n}(\identity-\gamma_i\drift)
    \right\|\\
    &\leq
    C
    \left[
        \frac{\gamma_n}{n}
        \sum_{j=k+2}^{n}(j-k-1)
        e^{-\lambda(j-k-2)\gamma_n/2}
        +e^{-\lambda(n-k-1)\gamma_n/2}
    \right].
\end{align*}
The sum is simply
\(
    \sum_{j=k+2}^{n}(j-k-1)e^{-\lambda(j-k-2)\gamma_n/2}
    =\mathcal O(\gamma_n^{-2})
\), and we obtain
\begin{equation}
    \left\lVert \gamma_{k+1}\bar{\mat W}_{k+2,n}
    -\frac{\drift^{-1}}{\sqrt{n-\bar n_0}}\right\rVert
    =
    \mathcal O\left(\frac{1}{\sqrt n}
    \left[
        \frac{1}{n\gamma_n}
        +e^{-\lambda(n-k-1)\gamma_n/2}
    \right]\right) .
\end{equation} \subsection{Proof of \texorpdfstring{\protect\Cref{lem:nm_clt}}{Lemma \ref*{lem:nm_clt}}}\label{sec:uniform_ergodic_mc}\label{proof:nm_clt}

The following CLT from \citep[Theorem 1]{zhang2026gaussianapproximationmultivariatemartingale} holds for stationary and uniformly ergodic Markov chains.
Recall $\mat V_n = \mathrm{Cov}(\sum_{k=n_0}^{n-1} \mat D_{k,n} \vc m_k)$.
Denote by $\vc m^\pi$ the stationary copy of $\vc m$ and $\mat V_n^\pi$ the corresponding covariance $\mathrm{Cov}(\sum_{k=n_0}^{n-1} \mat D_{k,n} \vc m_k^\pi)$.
\begin{lemma}\label{lem:zhang_markov_clt}
    Suppose $\vc\zeta_k=\vc0$, $\{\xi_k\}$ is uniformly ergodic, the martingale difference sequence $\vc m$ is bounded, and $\mat V_n^\pi \succ0$.
    Then, there exists a constant $C$ such that
    \begin{equation*}
        \Wass_{1}\left(\Law\left(\sum_{k=n_0}^{n-1}\mat D_{k,n}\vc m_k^\pi\right),\normal(\vc0,\mat V_n^\pi)\right)
        \leq C\lVert\mat V_n^\pi\rVert_2^{1/2}
        \left(\sum_{k=n_0}^{n-1}a_{k,n}^4\right)^{1/2}
        \left[1+d^{1/4}\left(\sum_{k=n_0}^{n-1}a_{k,n}^2\right)^{1/4}\right]
        ,
    \end{equation*}
    where $a_{k,n} \coloneqq \left\lVert (\mat V_n^\pi )^{-1/2}\mat D_{k,n}\vc m_k^\pi\right\rVert_{L_\infty,2}$.
\end{lemma}
Denoting by $\vc m^\pi$ a copy of $\vc m$ initialized from the stationary distribution $\pi$, and $\vc m$ one with any initial distribution with finite third moments, \citet{zhang2026gaussianapproximationmultivariatemartingale} show that, if $\max_{n_0\leq k<n}\lVert\mat D_{k,n}\rVert_2=\mathcal O(1/\sqrt{n-n_0})$, then
\begin{equation}\label{eq:stationary_coupling}
    \Wass_{1} \left(\Law\left(\sum_{k=n_0}^{n-1}\mat D_{k,n}\vc m_k\right), \Law\left(\sum_{k=n_0}^{n-1}\mat D_{k,n}\vc m_k^\pi\right) \right)  = \mathcal O\left(\frac{1}{\sqrt{n-n_0}} \right) .
\end{equation}

To complete the proof of \Cref{lem:nm_clt}, we follow the proof of \Cref{lem:covariance_convergence} to obtain for $\max_{n_0 \leq k < n} \| \mat D_{k,n} \|_2 = \mathcal{O}(1/\sqrt{n-n_0})$,
\begin{equation}\label{eq:Vk_convergence}
    \|\mat V_n-\mat V_n^{\pi}\|_2
    \leq
    C\max_{n_0\leq k<n}\|\mat D_{k,n}\|_2^2
    \sum_{k=n_0}^{n-1}\varrho^{k-1}
    =\mathcal O\left( \frac{1}{n-n_0}\right).
\end{equation}
By the assumption $\| \mat V_n^{-1}\|_2 = \mathcal{O}(1)$, we obtain \(\|(\mat V_n^{\pi})^{-1}\|_2=\mathcal O(1)\).
Since $\vc m$ is bounded under the assumptions of \Cref{lem:nm_clt}, $a_{k,n} =\mathcal O(1/\sqrt{n -n_0})$.
This implies $\left(\sum_{k=n_0}^{n-1}a_{k,n}^4\right)^{1/2}=\mathcal O\left(\frac{1}{\sqrt{n-n_0}}\right)$ and $\left(\sum_{k=n_0}^{n-1}a_{k,n}^2\right)^{1/4}=\mathcal O(1)$.
Using \Cref{lem:zhang_markov_clt}, \eqref{eq:stationary_coupling}, \eqref{eq:Vk_convergence}, and the triangle inequality, we obtain \Cref{lem:nm_clt}.

 \subsection{Proof of \texorpdfstring{\protect\Cref{lem:id_clt}}{Lemma \ref*{lem:id_clt}}}\label{eq:clt_iid}\label{proof:id_clt}
We use the following result from \citep[Theorems~5 and 6]{bonis2020steinsmethodnormalapproximation}.
While their result is stated for $\Wass_2$, their proof yields the following result for $\Wass_p$.
\begin{lemma}\label{lem:bonis_independent}
    Let $p\geq 2$.
    For a centered independent sequence $\{\vc X_k\}_{k=n_0}^{n-1}$, denote by $\nu$ the law of $(n-n_0)^{-1/2} \sum_{k=n_0}^{n-1} \vc X_k$.
    For $\ell\in[0,2]$, define
    \[
        M(\ell)^2
        \coloneqq
        \frac{1}{n-n_0}\sum_{k=n_0}^{n-1}
        \left(
            \left\lVert\mathbb E[\vc X_k\vc X_k\tp]\right\rVert_{HS}^2
            \mathbb E[\lVert\vc X_k\rVert_2^\ell]^2
            +
            \left\lVert
                \mathbb E[\vc X_k\vc X_k\tp\lVert\vc X_k\rVert_2^\ell]
            \right\rVert_{HS}^2
        \right).
    \]
    If
    \(
        \sum_{k=n_0}^{n-1}\mathbb E[\vc X_k\vc X_k\tp]
        =(n-n_0)\identity
    \) and \(
        \sum_{k=n_0}^{n-1}\mathbb E[\lVert\vc X_k\rVert_2^{p+2}]<\infty,
    \)
    then there exists a universal constant $C$ such that
    \begin{align*}
        \Wass_{p,2}\left(\nu,\normal(0,\identity)\right)
        \leq C(1+p)^{3/2}\Bigg[{}
        &\frac{
            \left(\sum_{k=n_0}^{n-1}\mathbb E[\lVert\vc X_k\rVert_2^{p+2}]\right)^{1/p}
        }{(n-n_0)^{1/2+1/p}}
        +
        \frac{
            \left(\sum_{k=n_0}^{n-1}\mathbb E[\lVert\vc X_k\rVert_2^4]\right)^{1/2}
        }{n-n_0}\\
        &+
        \frac{
            \left(\sum_{k=n_0}^{n-1}
                \left\lVert\mathbb E[\vc X_k\vc X_k\tp]\right\rVert_{HS}^2
            \right)^{1/2}
        }{n-n_0}
        +
        \left(\frac{M(0)M(2)}{n-n_0}\right)^{1/2}
        \Bigg].
    \end{align*}
\end{lemma}
By $M(0)\leq\sqrt2\max_k\|\vc X_k\|_{L_4,2}^{2}$ and $M(2)\leq\sqrt2\max_k\|\vc X_k\|_{L_4,2}^{4}$, we obtain that for $p \geq 2$,
\begin{equation}\label{eq:bonis_main}
    \Wass_{p,2}\bigl(\nu,\normal(0,\identity)\bigr)
    \leq
    \frac{C(1+p)^{3/2}}{\sqrt{n-n_0}}
    \left[
    \left(\frac1{n-n_0}\sum_{k=n_0}^{n-1}\mathbb E\|\vc X_k\|_2^{p+2}\right)^{1/p}
    +\max_k\|\vc X_k\|_{L_4,2}^{3}
    \right],
\end{equation}

Define a normalized sequence
\(
    \vc X_k
    =\sqrt{n-n_0}\mat V_n^{-1/2}\mat D_{k,n}\vc m_k
\)
so that
\begin{align*}
    \sum_{k=n_0}^{n-1}\mathbb E\|\vc X_k\|_2^{p+2}
    &\leq
    (n-n_0)^{(p+2)/2}
    \mathbb E\|\mat\Gamma^{-1/2}\vc m_1\|_2^{p+2}
    \sum_{k=n_0}^{n-1}\|\mat V_n^{-1/2}\mat D_{k,n} \mat \Gamma^{1/2}\|_2^{p+2}
    \\ &
    \leq (n-n_0)^{(p+2)/2+1}
        \mathbb E\|\mat\Gamma^{-1/2}\vc m_1\|_2^{p+2}
        \max_k\|\mat V_n^{-1/2}\mat D_{k,n} \mat \Gamma^{1/2}\|_2^{p+2}
\end{align*}
and
\begin{equation*}
\max_{n_0 \leq k < n}\|\vc X_k\|_{L_4,2}^{3} \leq
(n-n_0)^{3/2}\max_{n_0 \leq k < n}\|\mat V_n^{-1/2}\mat D_{k,n} \mat \Gamma^{1/2}\|_2^{3}
\|\mat\Gamma^{-1/2}\vc m_1\|_{L_4,2}^{3}
\end{equation*}
are satisfied.
In \eqref{eq:Psi_iid}, $\Psi_p^{\mathrm{iid}}$ is defined such that $\mathbb{E}^{1/p}\| \mat \Gamma^{-1/2} \vc m_1 \|_2^{p + 2}, \| \mat \Gamma^{-1/2} \vc m_1 \|_{L_4, 2}^3 \leq \Psi_p^{\mathrm{iid}}$.
Under \eqref{eq:clt_conditions}, we have that $\max_k \| \mat V_n^{-1/2}\mat D_{k,n} \mat \Gamma^{1/2}\|_2
\|\mat\Gamma^{-1/2}\| = \mathcal{O}(1/\sqrt{n - n_0})$, and so the bound of \Cref{lem:id_clt} holds.

 \subsection{Proof of \texorpdfstring{\protect\Cref{lem:lm_clt}}{Lemma \ref*{lem:lm_clt}}}\label{sec:martingale_wp_clt}\label{proof:lm_clt}
We use the following from \citep[Theorem~2.3(ii)]{fang2026martingalecentrallimittheorems}.
\begin{lemma}\label{lem:fks_martingale_clt}
    If $\vc m$ has finite $3p$ moments, then there exists a universal constant $C$ such that
    \begin{align*}
        &\quad  \Wass_{p,2}^3 \left(\Law\left(\sum_{k=n_0}^{n-1}\mat D_{k,n}\vc m_k\right),\normal(\vc0,\mat V_n)\right)
        \leq C \lVert\vc z\rVert_{L_p,2}^{2}
        p\sum_{k=n_0}^{n-1}
        \left\lVert
            \mat D_{k,n}\vc m_k
        \right\rVert_{L_{3p}, 2}^3
        \\ &
        +
        C \lVert\vc z\rVert_{L_p,2}^{2}\biggl(\Bigl\lVert\sum_{k=n_0}^{n-1}\mat D_{k,n}\vc m_k\Bigr\rVert_{L_{3p},2}
        +\lVert\mat V_n^{1/2}\vc z\rVert_{L_p,2}\biggr)
        \Bigl\lVert\sum_{k=n_0}^{n-1}\mat D_{k,n}
        \bigl(\mat \Gamma_k
        -\mathbb E[\mat \Gamma_k]\bigr)\mat D_{k,n}\tp
        \Bigr\rVert_{L_{3p/2},2}
        .
    \end{align*}
\end{lemma}
If $\Phi(\vc x^*, \cdot)$ is bounded, then $\vc m$ is bounded.
When \eqref{eq:clt_conditions} is satisfied, then $\sum_{k=n_0}^{n-1}\|\mat D_{k,n}\vc m_k\|_{L_{3p},2}^{3} = \mathcal{O}(1/\sqrt{n-n_0})$.
From the second part of \Cref{lem:pinelis21}, we obtain $\left\|\sum_{k=n_0}^{n-1}\mat D_{k,n}\vc m_k\right\|_{L_{3p},2}\leq C\sqrt p$.
Under the assumptions of \Cref{lem:lm_clt}, the bound in \Cref{lem:fks_martingale_clt} simplifies to
\[
    C\|\vc z\|_{L_p,2}^{2}
    \left[
    \frac{p}{\sqrt{n-n_0}}
    +\left(\sqrt p+\|\mat V_n^{1/2}\vc z\|_{L_p,2}\right)
    \left\|
    \sum_{k=n_0}^{n-1}\mat D_{k,n}
    (\mat\Gamma_k-\mathbb E\mat\Gamma_k)
    \mat D_{k,n}^{\top}
    \right\|_{L_{3p/2},2}
    \right].
\]
Since $\mathbb E[\vc m_j\vc m_k^\top]=0$ for $j \neq k$, it holds that $\|\mat V_n\|_2 =\left\| \sum_{k=n_0}^{n-1}\mat D_{k,n} \mathbb E[\vc m_k\vc m_k^\top]\mat D_{k,n}^\top \right\|_2$ is bounded by a constant, and we have
\[
\sqrt p+\|\mat V_n^{1/2}\vc z\|_{L_p,2}
\leq C\|\vc z\|_{L_p,2}.
\]
Therefore, when \Cref{lem:lm_clt} is satisfied, the bound in \Cref{lem:fks_martingale_clt} can be written as
\[
    C
    \frac{p\|\vc z\|_{L_p,2}^{2}}{\sqrt{n-n_0}}
    +
    C\|\vc z\|_{L_p,2}^{3}
    \left\|
    \sum_{k=n_0}^{n-1}\mat D_{k,n}
    (\mat\Gamma_k-\mathbb E\mat\Gamma_k)
    \mat D_{k,n}^{\top}
    \right\|_{L_{3p/2},2}
    .
\]
We complete the proof by substituting the following bound.
\begin{lemma}\label{lem:martingale_clt}
    Suppose \Cref{ass:noise_markov} holds, and that the corresponding $\Phi(\vc x^*, \cdot)$ is bounded, $\sup_k \| \vc m_k \|_{L_\infty, 2} < \infty$, and $\vc \zeta = 0$.
    Then, for $p \geq 2$,
    \begin{equation*}
    \begin{aligned}
        \left\lVert\sum_{k=n_0}^{n-1}\mat D_{k,n}
        \bigl(\mat \Gamma_k
        -\mathbb E[\mat \Gamma_k ]\bigr)\mat D_{k,n}\tp
        \right\rVert_{L_{3p/2},2}
        \leq C\sqrt d\,p^{3/2}
        \left(\sum_{k=n_0}^{n-1}\lVert\mat D_{k,n}\rVert_2^4\right)^{1/2}.
    \end{aligned}
    \end{equation*}
\end{lemma}
\begin{proof}{Proof}
    See \Cref{proof:martingale_clt}.
\end{proof}
 \subsection{Proof of \texorpdfstring{\protect\Cref{prop:clt_conditions}}{Proposition \ref*{prop:clt_conditions}}}\label{proof:clt_conditions}

By \Cref{prop:dtct}, we have $\| \mat W_{k,n} \| \leq \sqrt{\gamma_k} e^{-\lambda \sum_{j=k+1}^{n-1} \gamma_j} + \mathcal O(\sqrt{\gamma_k} / n)$.
For $\lambda \gamma_1 \geq 2$, we obtain $\| \mat W_{k,n} \| = \mathcal O(1/\sqrt n)$ uniformly over $k \in [n_0, n)$.

For the PR average, \Cref{prop:dtct} implies the $1/\sqrt n$ bound for $\| \gamma_k \bar{\mat W}_{k+1,n} \|$ when $k \geq \bar n_0$.
For $k < \bar n_0$, we use \eqref{eq:pr_early_weights} to obtain
\begin{align*}
    \max_{n_0 \leq k < n} \| \gamma_k \bar{\mat W}_{k+1, n} \|
    &\leq
    \max_{n_0 \leq k < \bar n_0}\bar C_W \frac{\gamma_k}{\sqrt n\,\gamma_n}
    \exp\left(-\frac\lambda2\sum_{i=k+1}^{\bar n_0-1}\gamma_i\right)
    +
    \mathcal O\left(\frac{1}{\sqrt n} \right)
    ,
\end{align*}
which is $    \mathcal O\left(\frac{1}{\sqrt n}\right)$.
 \subsection{Proof of \texorpdfstring{\protect\Cref{lem:martingale_clt}}{Lemma \ref*{lem:martingale_clt}}}\label{proof:martingale_clt}

Under \Cref{ass:markov}, we have by $\sup_j\mathbb EV(\xi_j)<\infty$ and \eqref{eq:V_ergodic} that $\left\lVert \mathbb{E}[\mat \Gamma_k \mid \mathcal H_j]  -\mat\Gamma\right\rVert_{L_1,2}=\mathcal{O}\left(\varrho^{k-1-j} \right)$.
Since $\Phi(\vc x^*,\cdot)$ is bounded, $\sup_k\lVert\vc m_k\rVert_{L_\infty,2}<\infty$, and \eqref{eq:Mk} give
\[
    \mat \Gamma_k
    =P\bigl(\Phi(\vc x^*,\cdot)\Phi(\vc x^*,\cdot)\tp\bigr)(\xi_{k-1})
    -(P\Phi)(\vc x^*,\xi_{k-1})(P\Phi)(\vc x^*,\xi_{k-1})\tp.
\]
Using H\"{o}lder's inequality and that $\mathbb{E}[\mat \Gamma_k \mid \mathcal H_j]$ is bounded, we obtain
\[
\begin{aligned}
    \left\|\mathbb E[\mat \Gamma_k \mid\mathcal H_j]
    -\mat\Gamma\right\|_{L_{3p/2},2}^{3p/2}
    &\le
    \left\|\mathbb E[\mat \Gamma_k\mid\mathcal H_j]
    -\mat\Gamma\right\|_{L_\infty,2}^{3p/2-1}
    \left\|\mathbb E[\mat \Gamma_k\mid\mathcal H_j]
    -\mat\Gamma\right\|_{L_1,2}\\
    &=
    \mathcal{O}\left(\varrho^{k-1-j} \right).
\end{aligned}
\]
Taking the $2/3p$ power on both sides, we obtain
\begin{equation}\label{eq:lm_conditional_covariance}
    \left\lVert\mathbb E[\mat \Gamma_k \mid\mathcal H_j]-\mat\Gamma\right\rVert_{L_{3p/2},2}
    \leq C\varrho^{2(k-1-j)/(3p)},
    \qquad 0 \leq j < k .
\end{equation}

Observe the identity
\begin{align*}
    \mat \Gamma_k -\mathbb E \mat \Gamma_k
    =\mathbb E[\mat \Gamma_k \mid\mathcal H_0]-\mathbb E[\mat \Gamma_k]
    +\sum_{j=1}^{k-1}\bigl(\mathbb E[\mat \Gamma_k \mid\mathcal H_j]
    -\mathbb E[\mat \Gamma_k \mid\mathcal H_{j-1}]\bigr)
    .
\end{align*}
By the Minkowski inequality,
\begin{equation}\label{eq:martingale_cov_bound}
\begin{aligned}
    &\left\lVert\sum_{k=n_0}^{n-1}\mat D_{k,n}
    \bigl(\mat \Gamma_k -\mathbb E \mat \Gamma_k\bigr)\mat D_{k,n}\tp
    \right\rVert_{L_{3p/2},2}
    \leq
    \left\lVert\sum_{k=n_0}^{n-1}\mat D_{k,n}
    \bigl(\mathbb E[\mat \Gamma_k \mid\mathcal H_0]-\mathbb E \mat \Gamma_k\bigr) \mat D_{k,n}\tp
    \right\rVert_{L_{3p/2}, 2}
    \\ &+
    \left\lVert\sum_{k=n_0}^{n-1} \sum_{j=1}^{k-1}\mat D_{k,n}
    \bigl(\mathbb E[\mat \Gamma_k \mid\mathcal H_j]-\mathbb E[\mat \Gamma_k \mid \mathcal H_{j-1}]\bigr) \mat D_{k,n}\tp
    \right\rVert_{L_{3p/2}, 2}
    .
\end{aligned}
\end{equation}
Using \eqref{eq:lm_conditional_covariance}, we obtain
\begin{align*}
    \left\lVert\sum_{k=n_0}^{n-1}\mat D_{k,n}
    \bigl(\mathbb E[\mat \Gamma_k \mid\mathcal H_0]-\mathbb E \mat \Gamma_k\bigr) \mat D_{k,n}\tp
    \right\rVert_{L_{3p/2}, 2}
    &\leq
    2 C_1 \sum_{k=n_0}^{n-1}\|\mat D_{k,n}\|_2^2\varrho^{2(k-1)/(3p)}
    \\ & \leq
    2C_1 \max_{n_0\le k<n}\|\mat D_{k,n}\|_2^2  \frac{\varrho^{2(n_0 - 1)/3p}}{1 - \varrho^{2/3p}}
    \\ & \leq
    2 C_1 \max_{n_0 \leq k < n} \| \mat D_{k,n} \|_2^2 \left(1 - \frac{3}{2 \log \varrho}\right) p
    ,
\end{align*}
where the last inequality used $\varrho^{2 (n_0 - 1)/3p} \leq 1$ and $1/(1 - e^{-x}) \leq 1 + 1/x$.

The second term in \eqref{eq:martingale_cov_bound} is bounded using
\begin{align*}
    &\quad \left\lVert\sum_{k=n_0}^{n-1} \sum_{j=1}^{k-1}\mat D_{k,n}
    \bigl(\mathbb E[\mat \Gamma_k \mid\mathcal H_j]-\mathbb E[\mat \Gamma_k \mid \mathcal H_{j-1}]\bigr) \mat D_{k,n}\tp
    \right\rVert_{L_{3p/2}, 2}
    \\ &=
    \left\lVert\sum_{j=1}^{n-2}\sum_{k=\max\{n_0,j+1\}}^{n-1}
    \mat D_{k,n}
    \bigl(\mathbb E[\mat\Gamma_k\mid\mathcal H_j]
    -\mathbb E[\mat\Gamma_k\mid\mathcal H_{j-1}]\bigr)
    \mat D_{k,n}^\top \right\rVert_{L_{3p/2}, 2}
    .
\end{align*}
Define
\(
    \mat Z_j=\sum_{k=\max\{n_0,j+1\}}^{n-1}
    \mat D_{k,n}\bigl(\mathbb E[\mat\Gamma_k\mid\mathcal H_j]-\mathbb E[\mat\Gamma_k\mid\mathcal H_{j-1}]\bigr)\mat D_{k,n}^{\top}
\).
Using the second part of \Cref{lem:pinelis21} with $q = 3p/2$, which holds for matrices using the Hilbert-Schmidt norm, we obtain
\[
    \left\|\sum_{j=1}^{n-2} \mat Z_j\right\|_{L_q,HS}^{2}
    \le(q-1)\sum_{j=1}^{n-2}\left\lVert\mat Z_j\right\rVert_{L_q,HS}^{2}
    .
\]
By the Minkowski inequality, $\lVert\cdot\rVert_{HS}\leq\sqrt d\,\lVert\cdot\rVert_2$, and $\| \mathbb{E}[\mat \Gamma_k \mid \mathcal H_j] - \mat \Gamma \|_{L_{3p/2},2}^{3p/2} = \mathcal{O}(\varrho^{k - 1 -j })$ we obtain
\[
\left\lVert\mat Z_j\right\rVert_{L_q,HS}
\leq
2C\sqrt d
\sum_{k=\max\{n_0,j+1\}}^{n-1}
\left\lVert\mat D_{k,n}\right\rVert_2^2
\varrho^{(k-1-j)/q}
\]
which we substitute to obtain
\[
\left\|\sum_{j=1}^{n-2} \mat Z_j \right\|_{L_q, HS}
\le
C\sqrt d\,\sqrt{\frac{q-1}{(1-\varrho^{1/q})^2}
\sum_{k=n_0}^{n-1}\lVert\mat D_{k,n}\rVert_2^4}.
\]
For $q = 3p/2$, we obtain $\sqrt{q - 1} / (1 - \varrho^{1/q}) = \mathcal{O}(p^{3/2})$, where we again used that $1/(1 - \varrho^{2/3p}) \leq 3p/(2(1-\varrho))$.
Combining the two bounds for \eqref{eq:martingale_cov_bound}, we obtain
\begin{align*}
    \left\lVert\sum_{k=n_0}^{n-1}\mat D_{k,n}
    \bigl(\mat \Gamma_k -\mathbb E \mat \Gamma_k\bigr)\mat D_{k,n}\tp
    \right\rVert_{L_{3p/2},2}
    \leq
    \frac{C}{1-\varrho} \left( p \max_{n_0 \leq k < n} \| \mat D_{k,n} \|_2^2
    +
    \sqrt d\,p^{3/2} \sqrt{\sum_{k=n_0}^{n-1} \| \mat D_{k,n} \|_2^4} \right) .
\end{align*}
Using that $\max_{n_0 \leq k < n} \| \mat D_{k,n} \|_2^4 \leq \sum_{k=n_0}^{n-1} \| \mat D_{k,n} \|_2^4$, we obtain \Cref{lem:martingale_clt}.
 \subsection{Proof of \texorpdfstring{\protect\Cref{lem:linear_poisson_moment_drift}}{Lemma \ref*{lem:linear_poisson_moment_drift}}}\label{proof:linear_poisson_moment_drift}

Since $\widetilde{\vc y}_k$ is $\mathcal H_{k-1}$-measurable and $\mathbb E[\widehat{\vc m}_k\mid\mathcal H_{k-1}]=0$, we have by \Cref{lem:pinelis21} the bound
\[
    \begin{aligned}
    &\left\|
    (\identity-\gamma_k\bar{\mat A})\sqrt{\gamma_k}\widetilde{\vc y}_k
    +\gamma_k\widehat{\vc m}_k
    \right\|_{L_q}^2
    &\leq
    \left\|
    (\identity-\gamma_k\bar{\mat A})\sqrt{\gamma_k}\widetilde{\vc y}_k
    \right\|_{L_q}^2
    +(q-1)\gamma_k^2\|\widehat{\vc m}_k\|_{L_q}^2.
    \end{aligned}
\]
Since $\bar{\mat A}=\drift+\bar\gamma\identity$, the Lyapunov equation \eqref{eq:lyapunov} gives $\bar{\mat A}\tp\mat Q+\mat Q\bar{\mat A} = \identity + 2 \bar \gamma \mat Q $.
Since $\lambda = 1/(2 \| \mat Q \|_2)$, we have $\identity\succeq \mat Q/\lVert\mat Q\rVert_2=2\lambda\mat Q$ from which we obtain
\begin{equation*}
    \bar{\mat A}^\top\mat Q+\mat Q\bar{\mat A} \succeq2(\lambda+\bar\gamma)\mat Q
    .
\end{equation*}
The resulting contraction $\lVert(\identity-\gamma_k\bar{\mat A})\vc v\rVert^2
\leq
\left(1-2(\lambda+\bar\gamma)\gamma_k
+\lVert\bar{\mat A}\rVert^2\gamma_k^2\right)\lVert\vc v\rVert^2$ and the bound on $\| \widehat{\vc m}_k \|$ gives the inequality
\begin{equation*}
    \begin{aligned}
        &\left\lVert
            (\identity-\gamma_k\bar{\mat A})\sqrt{\gamma_k}\widetilde{\vc y}_k
            +\gamma_k\widehat{\vc m}_k
        \right\rVert_{L_q}^2 \\
        &\qquad\leq
        \gamma_k\left(1-2(\lambda+\bar\gamma)\gamma_k
        +\left(\lVert\bar{\mat A}\rVert^2+2(q-1)K_m^2\right)\gamma_k^2\right)
        \lVert\widetilde{\vc y}_k\rVert_{L_q}^2
        +2(q-1)K_m^2\gamma_k^2.
    \end{aligned}
\end{equation*}
Minkowski's inequality and
$(u+v)^2\leq(1+\gamma_k^2)u^2+(1+\gamma_k^{-2})v^2$ give
\begin{align*}
    \gamma_{k + 1}\lVert\widetilde{\vc y}_{k + 1}\rVert_{L_q}^2
    &\leq
    \gamma_k\left(1-2(\lambda+\bar\gamma)\gamma_k
    +C_q^{\mathrm{LM}}\gamma_k^2\right)
    \lVert\widetilde{\vc y}_k\rVert_{L_q}^2
    +C_q^{\mathrm{LM}}\gamma_k^2.
\end{align*}
Dividing by $\gamma_{k + 1}$ proves \Cref{lem:linear_poisson_moment_drift}.
   
\bibliographystyle{plainnat}
\bibliography{bibliography}

\end{document}